\documentclass[11pt,a4paper]{article}
\usepackage[margin=2.5cm,bottom=3.cm]{geometry}	

\usepackage[american]{babel}
\usepackage{natbib} % has a nice set of citation styles and commands
\usepackage{mathtools} % amsmath with fixes and additions
\usepackage{booktabs} % commands to create good-looking tables
\usepackage{tikz} % nice language for creating drawings and diagrams

\title{Stochastic Online Instrumental Variable Regression:\\ Regrets for Endogeneity and Bandit Feedback}

\author{ Riccardo Della Vecchia
\\
Universit\'e de Lille, Inria, CNRS,\\
  F-59000 Lille, France\\
  \texttt{ric.della-vecchia@gmail.com}
  \and
 Debabrota Basu\\
Universit\'e de Lille, Inria, CNRS,\\
  F-59000 Lille, France\\
  \texttt{debabrota.basu@inria.fr} 
  }
 \date{}

\usepackage[utf8]{inputenc} % allow utf-8 input
\usepackage[T1]{fontenc}    % use 8-bit T1 fonts
\usepackage{url}            % simple URL typesetting
\usepackage{booktabs,longtable,array}       % professional-quality tables
\usepackage{amsfonts}       % blackboard math symbols
\usepackage{nicefrac}       % compact symbols for 1/2, etc.
\usepackage{microtype}      % microtypography
\usepackage{xcolor}         % colors
\usepackage{amsthm}
\usepackage{amsmath,amssymb}
\usepackage{amsfonts}
\usepackage{bbm}
\usepackage{mathtools}
\usepackage{url}
\usepackage{xcolor}
\usepackage[most]{tcolorbox}
\usepackage{natbib}
\usepackage{sidecap}
\usepackage{todonotes}
\usepackage{marginnote}

\usepackage{appendix}
\usepackage{mathtools}
\usepackage{etoolbox}
\usepackage{subcaption}
\usepackage{booktabs}
\usepackage{cases}

\usepackage{algorithm}
\usepackage{algpseudocode}

\usepackage{enumitem}
\setlist{nosep}

\newlength{\minipagewidth}
\newtheorem{example}{Example}[section]
\newcommand{\cA}{\mathcal{A}}
\newcommand{\cB}{\mathcal{B}}

\newcommand{\cX}{\mathcal{X}}

\newcommand{\cZ}{\mathcal{Z}}
\newcommand{\bA}{\mathbf{A}}
\newcommand{\bB}{\mathbf{B}}

\newcommand{\bE}{\mathbf{E}}

\newcommand{\bG}{\mathbf{G}}
\newcommand{\bH}{\mathbf{H}}

\newcommand{\bP}{\mathbf{P}}

\newcommand{\bX}{\mathbf{X}}

\newcommand{\bZ}{\mathbf{Z}}
\newcommand{\bx}{\boldsymbol x}
\newcommand{\bv}{\boldsymbol v}
\newcommand{\bz}{\boldsymbol z}
\newcommand{\by}{\boldsymbol y}

\newcommand{\be}{\boldsymbol \eta}

\newcommand{\E}{\mathbb{E}}

\newcommand{\I}{\mathbf{I}}

\newcommand{\R}{\mathbb{R}}

\DeclareMathOperator*{\argmin}{argmin}
\DeclareMathOperator*{\argmax}{argmax}

\newcommand{\wbs}{\widehat{\boldsymbol{\Sigma}}}

\newcommand{\bsig}{\boldsymbol{\Sigma}}

\newcommand{\bt}{\boldsymbol \Theta}

\newcommand{\bb}{\boldsymbol \beta}
\newcommand{\beeta}{\boldsymbol \eta}

\newcommand{\indep}{\perp \!\!\! \perp}

\newcommand{\bep}{\boldsymbol{\epsilon}}

\usepackage{nicefrac}

\DeclareFontFamily{U}{matha}{\hyphenchar\font45}
\DeclareFontShape{U}{matha}{m}{n}{
<-6> matha5 <6-7> matha6 <7-8> matha7
<8-9> matha8 <9-10> matha9
<10-12> matha10 <12-> matha12
}{}
\DeclareSymbolFont{matha}{U}{matha}{m}{n}
\DeclareFontFamily{U}{mathx}{\hyphenchar\font45}
\DeclareFontShape{U}{mathx}{m}{n}{
<-6> mathx5 <6-7> mathx6 <7-8> mathx7
<8-9> mathx8 <9-10> mathx9
<10-12> mathx10 <12-> mathx12
}{}
\DeclareSymbolFont{mathx}{U}{mathx}{m}{n}
\DeclareMathDelimiter{\vvvert} {0}{matha}{"7E}{mathx}{"17}%
\DeclarePairedDelimiter\norm\lVert\rVert
\DeclarePairedDelimiterX{\normiii}[1]
{\vvvert}
{\vvvert}
{\ifblank{#1}{\:\cdot\:}{#1}}

\DeclarePairedDelimiter{\Prfences}{(}{)}
\newcommand{\llmin}{\operatorname{\lambda_{min}}\Prfences}
\newcommand{\llmax}{\operatorname{\lambda_{max}}\Prfences}

\DeclarePairedDelimiterX{\expectarg}[1]{[}{]}{%
  \ifnum\currentgrouptype=16 \else\begingroup\fi
  \activatebar#1
  \ifnum\currentgrouptype=16 \else\endgroup\fi
}

\newcommand{\innermid}{\nonscript\;\delimsize\vert\nonscript\;}
\newcommand{\activatebar}{%
  \begingroup\lccode`\~=`\|
  \lowercase{\endgroup\let~}\innermid 
  \mathcode`|=\string"8000
}

\usepackage[nohints]{minitoc}

\usepackage{mathrsfs}
\DeclareRobustCommand{\bigO}{%
  \text{\usefont{OMS}{cmsy}{m}{n}O}%
}
\allowdisplaybreaks

\newcommand{\tsls}{{\ensuremath{\mathsf{\text{2}SLS}}}}
\newcommand{\otsls}{{\ensuremath{\mathsf{O\text{2}SLS}}}}
\newcommand{\ofuliv}{{\ensuremath{\mathsf{OFUL\text{-}IV}}}}
\newcommand{\oful}{{\ensuremath{\mathsf{OFUL}}}}
\newcommand{\ridge}{{\ensuremath{\mathsf{Ridge}}}}

\newcommand{\erm}{{\ensuremath{\mathsf{ERM}}}}
\newcommand{\mse}{{\ensuremath{\mathsf{MSE}}}}
\newcommand{\mmsee}{{\ensuremath{\mathsf{MMSEE}}}}
\newcommand{\mmse}{{\ensuremath{\mathsf{MMSE}}}}
\newcommand{\ols}{{\ensuremath{\mathsf{OLS}}}}
\newcommand{\vaw}{{\ensuremath{\mathsf{VAWR}}}}

\makeatletter
\newtheorem*{rep@theorem}{\rep@title}
\newcommand{\newreptheorem}[2]{%
	\newenvironment{rep#1}[1]{%
		\def\rep@title{\textbf{#2} \ref{##1}}%
		\begin{rep@theorem}}%
		{\end{rep@theorem}}}
\makeatother
\newreptheorem{theorem}{Theorem}
\newreptheorem{lemma}{Lemma}
\newreptheorem{proposition}{Proposition}
\newreptheorem{assumption}{Assumption}
\newreptheorem{corollary}{Corollary}
\newreptheorem{definition}{Definition}
\newreptheorem{example}{Example}
\usepackage{cleveref}
\usepackage{amsmath}
\usepackage{amssymb}
\usepackage{mathtools}
\usepackage{amsthm}
\theoremstyle{plain}
\newtheorem{theorem}{Theorem}[section]
\newtheorem{proposition}[theorem]{Proposition}
\newtheorem{lemma}[theorem]{Lemma}
\newtheorem{corollary}[theorem]{Corollary}
\theoremstyle{definition}
\newtheorem{definition}[theorem]{Definition}
\newtheorem{assumption}[theorem]{Assumption}
\theoremstyle{remark}
\newtheorem{remark}[theorem]{Remark}
\allowdisplaybreaks

\usepackage{todonotes}

\newcommand{\red}[1]{{\color{black}#1}}

\begin{document}
\maketitle

\begin{abstract}
Endogeneity, i.e. the dependence of noise and covariates, is a common phenomenon in real data due to omitted variables, strategic behaviours, measurement errors etc.
In contrast, the existing analyses of stochastic online linear regression with unbounded noise and linear bandits depend heavily on exogeneity, i.e. the independence of noise and covariates.
Motivated by this gap, we study the \textit{over- and just-identified Instrumental Variable (IV) regression, specifically Two-Stage Least Squares, for stochastic online learning}, and propose to use an online variant of Two-Stage Least Squares, namely \otsls{}. 
%IV regression and approach to it are widely deployed in economics and causal inference to identify the underlying model from an endogenous dataset. 
% Thus, we propose to use an online variant of Two-Stage Least Squares approach, namely \otsls{}, to tackle endogeneity in stochastic online learning. 
We show that \otsls{} achieves $\bigO(d_{\bx}d_{\bz}\log^2 T)$ identification and $\widetilde{\bigO}(\gamma \sqrt{d_{\bz} T})$ oracle regret after $T$ interactions, where $d_{\bx}$ and $d_{\bz}$ are the dimensions of covariates and IVs, and $\gamma$ is the bias due to endogeneity.
For $\gamma=0$, i.e. under exogeneity, \otsls{} exhibits $\bigO(d_{\bx}^2 \log^2 T)$ oracle regret, which is of the same order as that of the stochastic online ridge.
Then, we leverage \otsls{} as an oracle to design \ofuliv, a stochastic linear bandit algorithm to tackle endogeneity. \ofuliv yields $\widetilde{\bigO}(\sqrt{d_{\bx}d_{\bz}T})$ regret that matches the regret lower bound under exogeneity.
For different datasets with endogeneity, we experimentally show efficiencies of \otsls{} and \ofuliv{}.
\end{abstract}

\section{Introduction}\label{s:intro}
\begin{example}[Learning price-sales dynamics]\label{example1}
    A market analyst wants to learn how price of a food item at a given day $t$ affects the sales of the item, given a daily stream of data of food prices and sales. Further goal is to help any restaurant or food production company to understand how setting price of a food item is going to increase the sales. The analyst decides to run \emph{online linear regression} to learn the relation $\mathrm{Price} \rightarrow \mathrm{Sales}$ (Fig.~\ref{fig:ex_1a}). This is equivalent to learning $\beta$ in the second-stage $\mathrm{Sales}_t = \beta \times \mathrm{Price}_t + \eta_t$ (Figure~\ref{fig:ex_1b}). The analyst considers $\eta_t$ to be an i.i.d. sub-Gaussian noise due to unintentional external factors. They are independent of the covariates in this regression, i.e. $\mathrm{MaterialCost}$ and $\mathrm{Price}$. After running regression, she learns that increasing the price decreases sale of the item, i.e. $\beta=-1$.
\end{example}
%endogeneity -> iv regression with 2sls
% \vspace*{-0.5cm}
Like this example, online regression is a founding component of online learning~\citep{kivinen2004online}, sequential testing~\citep{kazerouni2021best}, contextual bandits~\citep{foster2020beyond}, and reinforcement learning~\citep{ouhamma2022bilinear}. Especially, online linear regression is widely used and analysed to design efficient algorithms with theoretical guarantees~\citep{greene2003econometric,abbasi2011online,hazan2012linear}. 
In linear regression, the \textit{outcome} (or output variable) $Y \in \R$, and the \textit{input features} (or covariates, or treatments) $\bX \in \R^d$ are related by a structural equation: $Y = \bb^T \bX + \eta$,
where $\bb$ is the \textit{true parameter} and $\eta$ is the observational noise with variance $\sigma^2$.
\textit{The goal is to estimate $\bb$ from an observational dataset.}
Two common assumptions in the analysis of linear regression are (i) bounded observations and covariates~\citep{vovk1997competitive,bartlett2015minimax,gaillard2019uniform}, and (ii) \textit{exogeneity}, i.e. independence of the noise $\eta$ and the input features $\bX$ ($\E[\eta|\bX]= 0$)~\citep{abbasi2011online,ouhamma2021stochastic}. Under exogeneity, researchers have studied scenarios where the observational noise is unbounded and has only bounded variance $\sigma^2$. But this setting asks for a different technical analysis than the bounded adversarial setting popular in online regression literature. For example, \cite{ouhamma2021stochastic} analyse online forward and online ridge regressions in the unbounded stochastic setting.
% asks for a different technical analysis than the bounded adversarial setting popular in online regression literature but . 

\setlength{\textfloatsep}{2pt}% Remove \textfloatsep
\begin{repexample}{example1}[Continued]
    If a big festival happens in a city, a restaurant owner can increase price of the food knowing still the sales will increase. This event contradicts the analyst's previously learnt parameter. Thus, as she incorporates the price-sales data of this city, the estimated $\beta$ changes, for example, to $-0.5$. Because her algorithm assumes exogeneity of noise and covariates, while in reality the omitted variable $\mathrm{Event}_t$ affects both price and sales (Fig.~\ref{fig:ex_1a}). Hence, the covariate $\mathrm{Price}_t$, gets correlated with the noise, $\eta_t = \rho_s \mathrm{Event}_t + \eta_{S,t}$ (Fig.~\ref{fig:ex_1b}). $\mathrm{Price}_t$ is called an endogeneous variable.
\end{repexample}

Similar to this example, in real-life, \textit{endogeneity}, i.e. dependence between noise and covariates ($\E[\eta|X]\neq 0$)~\citep{greene2003econometric,angrist1996identification} is often observed due to omitted explanatory variables, strategic behaviours during data generation, measurement errors, the dependence of the output and the covariates on unobserved confounding variables etc.~\citep{wald1940fitting,mogstad2021causal,zhu2022causal}.
Motivated by such scenarios, \textit{we analyse online linear regression aiming to estimate $\bb$ accurately from endogenous observational data, where noise is stochastic and unbounded.}
\setlength{\textfloatsep}{2pt}% Remove \textfloatsep
\begin{figure*}[t!]
     \centering
     \begin{subfigure}[b]{0.3\textwidth}
         \centering
         \includegraphics[width=1\columnwidth]{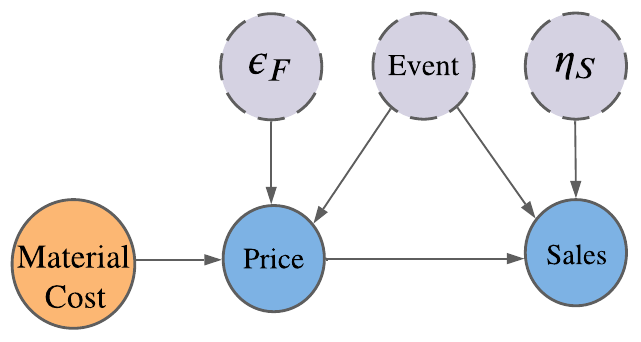}
         \caption{Graphical Model of Example~\ref{example1}}\label{fig:ex_1a}
     \end{subfigure}
     \hfill
     \begin{subfigure}[b]{0.35\textwidth}
         \centering
         \includegraphics[width=\textwidth]{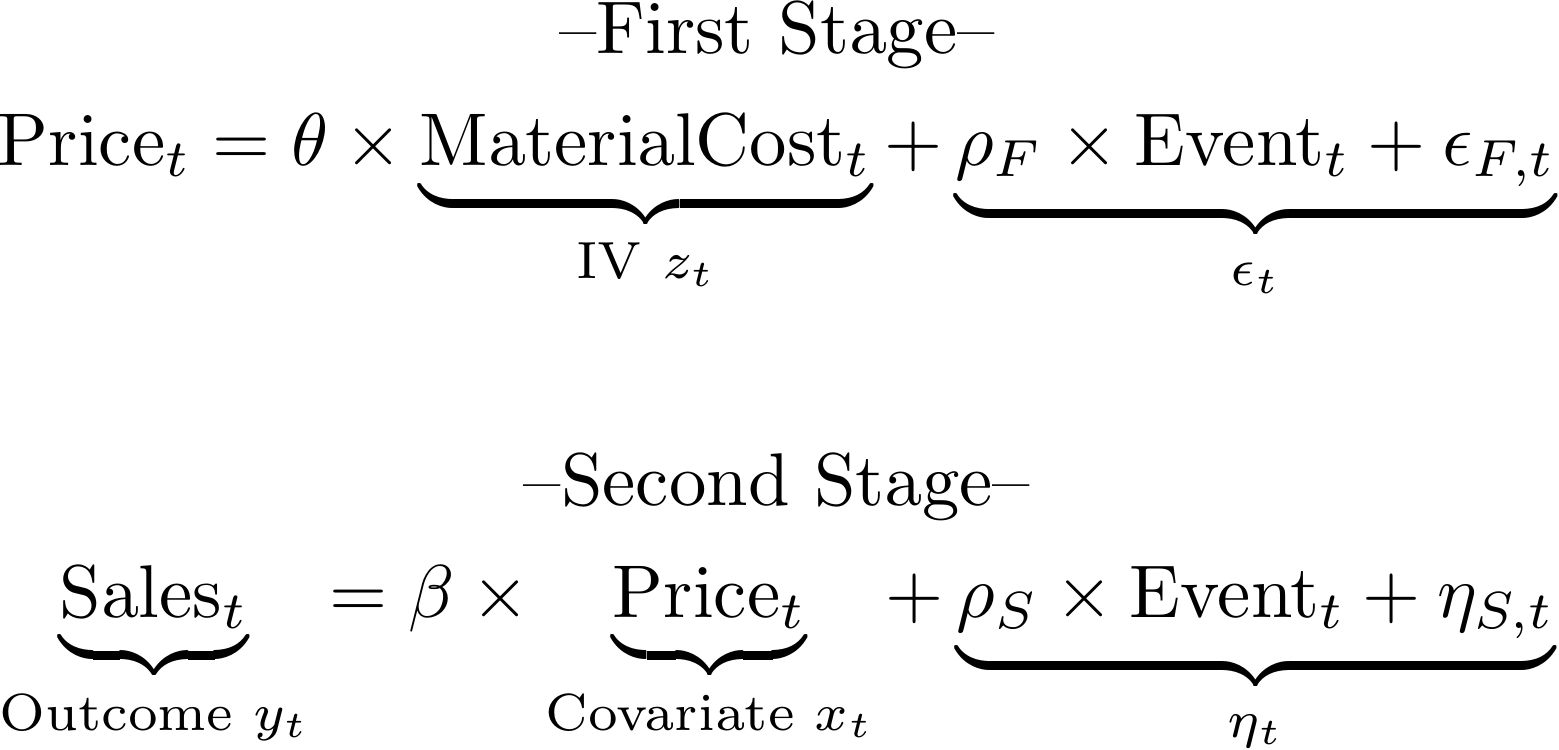}
         \caption{Two-stages of IV Regression}\label{fig:ex_1b}
     \end{subfigure}
     \hfill
     \begin{subfigure}[b]{0.18\textwidth}
         \centering
         \includegraphics[width=0.8\columnwidth]{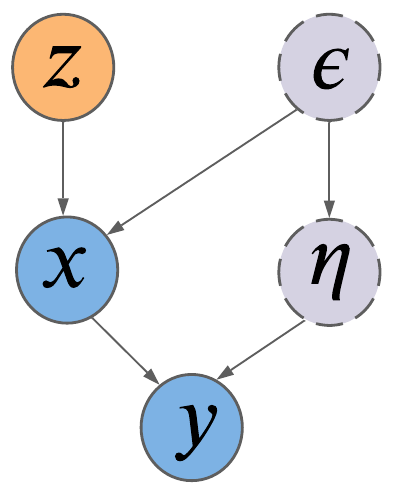}
         \caption{DAG of \tsls{} at $t$}\label{fig:pgm}
     \end{subfigure}
        \caption{Relations between IVs (green), Covariates (blue), and Outcome (blue) in Example~\ref{example1} and in general for Online Two-stage Regression. Unobserved variables are in dotted circles. Observed quantities are in solid circles.}\label{fig:example}
\end{figure*}

\textbf{Instrumental Variable (IV) regression.} Historically, IVs are introduced to identify and quantify the causal effects of endogenous covariates~\citep{newey2003instrumental}. IVs are widely used in economics~\citep{wright1928tariff,mogstad2021causal}, causal inference~\citep{rubin1974estimating,hernan2020causal,harris2022strategic}, bio-statistics and epidemiology~\citep{burgess2017review}.
One popular framework is to choose IVs which heavily influence the covariates but are independent of the exogenous noise.
Following that, one aims to learn the relations between IVs and covariates in one stage, then uses the predicted values of covariates to learn the relation between covariates and the outcome variables in the next stage.
This approach to conducting two stages of linear regression using IVs is called \textit{Two Stage Least Squares Regression} (\tsls{})~\citep{angrist1995two,angrist1996identification}. \tsls{} has become the standard tool in economics, social sciences, and statistics to study the effect of treatments on outcomes involving endogeneity. 
Recently, in machine learning, researchers have extended traditional \tsls{} to nonlinear structures, non-compliant instruments, and corrupted observations using deep learning~\citep{liu2020deep,xu2020learning,xu2021deep}, graphical models~\citep{stirn2018thompson}, and kernel regression~\citep{zhu2022causal}.
\begin{repexample}{example1}[Continued]
To rectify the error due to endogeneity, the analyst further studies the main factors behind pricing and observes that $\mathrm{Price}_t$ depends on material cost of the required components $\mathrm{MaterialCost}_t$. She also observes that $\mathrm{MaterialCost}_t$ is independent of the omitted variable $\mathrm{Event}_t$.
To rectify the error due to endogeneity, the analyst observes that $\mathrm{MaterialCost}_t$ is independent of the omitted variable $\mathrm{Event}_t$. Thus, she chooses $\mathrm{MaterialCost}_t$ to be an \textbf{Instrumental Variable (IV)}. She runs an online linear regression to learn the relation between $\mathrm{Price}_t$ and $\mathrm{MaterialCost}_t$. This is equivalent to learning $\theta$ in the first-stage $\mathrm{Price}_t = \theta \times \mathrm{MaterialCost}_t + \epsilon_t$, where exogeneity is obeyed. She learns that an increase in material cost increases the price, say $\theta = 1$. 
Now, she uses the prices predicted at time $t$, by the first-stage regressor, as the second-stage covariates, and uses them further the relation between price and sales in the second-stage. This leads her to an estimate of the second-stage parameter $\beta$, say $-1.01$, which is accurate even under endogeneity.
\end{repexample}

Thus, we are interested in studying the \tsls{} approach for online learning. But the existing analyses of \tsls{} are asymptotic, i.e. what can be learned if we have access to an infinite number of samples in an offline setting~\citep{singh2020machine,liu2020deep,nareklishvili2022deep}. In applications, this analysis is vacuous as one has access to only finite samples. Additionally, in practice, it is natural to acquire the data sequentially as treatments are chosen on-the-go and then to learn the structural equation from the sequential data~\citep{venkatraman2016online}. This motivates us to analyse the online extension of \tsls{}, named \otsls{}. 

Additionally, as stated in Example~\ref{example1}, the goal of the analyst might be to aid a restaurant owner or food supplier to strategically decide on a price of the food item from a price range and also corresponding raw material cost among multiple suppliers to improve the sales as much as possible. At this point, the estimation of the first- and second-stage parameters aid in sequential decision making, leading to the maximisation of accumulated outcomes over time. This is exactly a linear bandit problem. 
In this case, the analyst only observes the outcome corresponding to the choice of price and material cost. This is referred to as bandit feedback in online learning literature and studied under the linear bandit formulation~\citep{abbasi2011improved}. This motivates us to extend \otsls{} to\textit{ linear bandits, where both bandit feedback and endogeneity occur}. In this paper, we investigate two questions:

1. \textit{What is the upper bound on the loss in performance for deploying parameters estimated by \otsls{} instead of the true parameters $\bb$? How does estimating the true parameters $\bb$ influence different performance metrics under endogeneity?}

2. \textit{Can we design efficient algorithms for linear bandits with endogeneity by using \otsls{}?}

\noindent\textbf{Our contributions.} Our investigation has led to

1. \textit{A non-asymptotic analysis of \otsls{}:} First, we identify three notions of regret: \textit{identification regret}, \textit{oracle regret}, and \textit{population regret}. Though all of them are of same order under exogeneity, we show that the relations are more nuanced under endogeneity and unbounded noise \red{(see Appendix~\ref{sec:reg_big})}. We focus specifically on the identification regret, i.e. the sum of differences between the estimated parameters $\lbrace \bb_t\rbrace_{t=1}^T$, and the true parameter $\bb$, and \textit{oracle regret}, i.e. the sum of differences between the losses incurred by the estimated parameters $\lbrace \bb_t\rbrace_{t=1}^T$, and the true parameter $\bb$. %Lower is the regret better is the performance of the algorithm. 
In Section~\ref{sec:otsls}, we theoretically show that \otsls{} achieve $\bigO(d_{\bx}d_{\bz} \log^2 T)$ identification regret and $\bigO(d_{\bx}d_{\bz} \log^2 T + \gamma \sqrt{d_{\bz} T \log T})$ oracle regret after receiving $T$ samples from the observational data. Identification regret of \otsls{} is $d_{\bz}$ multiplicative factor higher than regret of online linear regression under exogeneity, and oracle regret is $\bigO(\gamma \sqrt{d_{\bz} T \log T})$ additive factor higher. These are the costs that \otsls{} pay for tackling endogeneity in two stages. In our knowledge, \textit{we are the first to propose a non-asymptotic regret analysis of \otsls{} with stochastic and unbounded noise}. \red{Due to two-stages and endogeneity, we can't rely on standard techniques and need to prove concentration bounds for dependent r.v. (Appendix~\ref{sec:concentration_correlated}) leading to a novel analysis for Thm. \ref{thm:reg_otsls_oracle}.}
 We also experimentally demonstrate efficiency of \otsls{} on synthetic and real-data with endoegeneity (Section~\ref{sec:exp_main}, Appendix~\ref{App:exp}).

2. \textit{\ofuliv{} for linear bandits with endogeneity:} In Section~\ref{sec:ofuliv}, we study the linear bandit problem with endogeneity. \textit{We design an extension of OFUL algorithm used for linear bandit with exogeneity, namely \ofuliv{}}, to tackle this problem. \ofuliv{} uses \otsls{} to estimate the parameters, and corresponding confidence bounds on $\bb$ to balance exploration--exploitation. 
\red{Lemma 3.3 derives a new confidence ellipsoid around O2SLS estimator $\bb_t$ with a new design matrix $\hat{\boldsymbol{H}}_t$. Following existing derivations for ridge would ask algorithm to know tight upper bounds on the hidden paramenter unlike us.}
We show that \ofuliv{} achieve $\bigO(\sqrt{d_{\bx}d_{\bz} T}\log T)$ regret after $T$ interactions. We further experimentally validate that \ofuliv{} incurs lower regret under endeogeneity than \oful{}~\citep{abbasi2011improved} (Section~\ref{sec:exp_main}, Appendix~\ref{App:exp}).

\textbf{Related work: Online regression without endogeneity.} Our analysis of \otsls{} extends the tools and techniques of online linear regression without endogeneity.  Analysis of online linear regression began with~\citep{foster1991prediction,littlestone1991line}. \cite{vovk1997competitive,vovk2001competitive} show that forward and ridge regressions achieve $\bigO(d_{\bx} Y_{\max}^2 \log T)$ for outcomes with bound $Y_{\max}$. \cite{bartlett2015minimax} generalise the analysis further by considering the features known in hindsight. \cite{gaillard2019uniform} improve the analysis further to propose an optimal algorithm and a lower bound. \textit{These works perform an adversarial analysis with bounded outcomes, covariates, and observational noise, while we focus on the stochastic setting.} \cite{ouhamma2021stochastic} study the stochastic setting with bounded input features and unbounded noise. But they need to assume independence of noise and input features. \textit{In this paper, we analyse online \tsls{} under endogeneity and unbounded (stochastic) noise.} We do not assume to know the bound on the outcome and derive high probability bounds for any bounded sequence of features. Previously, \cite{venkatraman2016online} studied \otsls{} for system identification but provided only asymptotic analysis.

\textbf{Related Work: Linear bandits without endogeneity.} Linear bandits generalise the setting of online linear regression under bandit feedback~\citep{abbasi2011improved,abbasi2012online,foster2020beyond}. To be specific, in bandit feedback, the algorithm observes only the outcomes for the input features that it has chosen to draw during an interaction. Popular algorithm design techniques, such as optimism-in-the-face-of-uncertainty and Thompson sampling, are extended to propose OFUL~\citep{abbasi2012online} and LinTS~\citep{Abeille2017LinearTS}, respectively. OFUL and LinTS algorithms demonstrate $\bigO(d\sqrt{T}\log T)$ and $\bigO(d^{1.5}\sqrt{T}\log T)$ regret under exogeneity.
\textit{Here, we use \otsls{} as a regression oracle to develop \ofuliv{} for linear bandits with endogeneity. We prove that \ofuliv{} achieves $\bigO(d\sqrt{T}\log T)$ regret.}

%b) 2sls with offline and asymptotic analysis
\textbf{Related Work: Instrument-armed bandits.} \cite{kallus2018instrument} is the first to study endogeneity and instrumental variables in a stochastic bandit setting. \cite{stirn2018thompson} propose a Thompson sampling-type algorithm for stochastic bandits, where endogeneity arises due to non-compliant actions. But both \cite{kallus2018instrument} and \cite{stirn2018thompson} study only the finite-armed bandit setting where arms are independent of each other. In this paper, \textit{we study the stochastic linear bandits with endogeneity requiring different techniques of analysis and algorithm design.}
\cite{krishnamurthy2018semiparametric} studies a linear contextual bandit setup close to ours, but they assume arm-independent and bounded noise, and thus, yielding significantly different analysis.

\textbf{Notations:} Matrices and vectors are denoted by bold capital and bold small letters (e.g. $\bA$ and $\mathbf{a}$). $\norm{\cdot}_p$ is $l_p$ norm of a vector. $\sigma_{\min}(\cdot)$ is minimum singular value and $\normiii{\cdot}_2$ is operator norm of a matrix. %, respectively.
For any vector $y \in \mathbb{R}^n$ and a positive definite matrix $\bA \in \mathbb{R}^{n\times n}$, let us define the norm $\| \by \|_\bA \triangleq \sqrt{\by^T \bA \by}=\sqrt{\langle\by, \bA \by\rangle} $.

\section{Background: Instrumental variables and offline two-stage least squares}% (\tsls)}% 
We are given an observational dataset $\{\bx_i, y_i\}_{i=1}^n$ consisting of $n$ pairs of input features and outcomes. Here, $y_i \in \R$ and $\bx_i \in \R^{d_{\bx}}$. Inputs and outcomes are stochastically generated using a linear model
\begin{equation}
	y_{i}=\bb^\top \bx_{i} +\eta_{i}\tag{Second stage},\label{eq:second_stage}
\end{equation}
where $\bb \in \R^{d_{\bx}}$ is the \textit{unknown true parameter vector} of the linear model, and $\eta_{i} \sim \mathcal{N}(0, \sigma_{\eta}^2)$ is the \textit{unobserved error} representing all causes of $y_{i}$ other than $\bx_{i}$. It is assumed that the error terms $\eta_i$ are independent and identically distributed with bounded variance $\sigma_{\eta}^2$.
The parameter vector $\bb$  quantifies the causal effect on $y_{i}$ due to a unit change in a component of $\bx_{i}$, while retaining other causes of $y_{i}$ constant. 
The goal of linear regression is to estimate $\bb$ by \textit{minimising the square loss over dataset}~\citep{Brier1950VERIFICATIONOF}, i.e. $\widehat{\bb} \triangleq \argmin_{\bb'} \sum_{i=1}^n (y_i - \bb'^\top \bx_i)^2$.

The obtained solution is called the Ordinary Least Square (OLS) estimate of $\bb$~\citep{wasserman2004all}, and is a corner stone of online regression~\citep{gaillard2019uniform} and linear bandit algorithms~\citep{foster2020beyond}.
Specifically, if we define the input feature matrix to be $\bX_{n} \triangleq [\bx_{1}, \bx_{2}, \ldots, \bx_{n}]^{\top} \in \R^{n\times d_{\bx}}$, the outcome vector to be $\by_n \triangleq [y_1, \ldots, y_n]^{\top}$, and the noise vector is $\beeta_n \triangleq [\eta_1, \ldots, \eta_n]^{\top}$, OLS estimator is
\begin{align*}
    {\widehat {\bb }}_{\mathrm {OLS} }
&\triangleq
    (\bX_{n}^{\top }\bX_{n})^{-1}\bX_{n}^{\top }\by_n
% =
%     (\bX_{n}^{\top }\bX_{n})^{-1}\bX_{n}^{\top }(\bX_{n}\bb +\beeta_n)
=
    \bb +(\bX_{n}^{\top }\bX_{n})^{-1}\bX_{n}^{\top }\beeta_n
\end{align*}
\textit{If $\bx_i$ and $\eta_i$ are independent}, i.e. under exogeneity, the OLS estimator is unbiased: $\mathbb E[\widehat \bb_{\mathrm {OLS}} ] = \bb$. Furthermore, it is also asymptotically consistent since $\bX_{n}^{\top }\beeta_n\overset{p}{\rightarrow} 0$ implies ${\widehat {\bb }}_{\mathrm {OLS} } \overset{p}{\rightarrow} \bb$~\citep{greene2003econometric}.% as $n\to \infty$. %so the estimator is unbiased and consistent. %When the hypothesis of independence between 

In practice, the input features $\bx$ and the noise $\eta$ are often correlated~\citep[Chapter 8]{greene2003econometric}. As in Figure~\ref{fig:pgm}, this dependence, called \textit{endogeneity}, is modelled with \textit{a confounding unobserved random variable} $\bep$.
To compute an unbiased estimate of $\bb$ under endogeneity, the Instrumental Variables (IVs) $\bz$ are introduced~\citep{angrist1996identification,newey2003instrumental}. IVs are chosen such that they are highly correlated with endogenous components of $\bx$ (relevance condition) but are independent of noise $\eta$ (exogeneity of $\bz$). %This formulation leads to instrumental variable regression to tackle endogeneity.
This leads to Two-stage Least Squares (\tsls{}) approach to IV regression~\citep{angrist1995two,angrist1996identification}. We further assume that IVs, $\bZ_n \triangleq [\bz_{1}, \ldots, \bz_{n}]^{\top} \in \R^{n\times d_{\bz}}$, cause linear effects on the endogenous covariates
\begin{equation}
    \bX_{n}= \bZ_{n}\bt  + \bE_n.\tag{First stage}\label{eq:first_stage}
\end{equation}
$\bt\in\mathbb R^{d_{\bz}\times d_{\bx}}$ is the unknown first-stage parameter matrix and $\bE_n \triangleq [\bep_{1}, \ldots, \bep_{n}]^{\top} \in\mathbb R^{n\times d_{\bx}}$ is the unobserved noise matrix leading to confounding in the second stage.
First-stage is a ``classic'' \emph{multiple regression} with covariates $\bz$ independent of the noise $\bep\sim \mathcal{N}(0, \sigma_{\bep}^2 \I_{d_{\bx}})$~\citep[Ch. 13]{wasserman2004all}. Thus, OLS in first-stage yields an estimate $\widehat{\bt}_n \triangleq (\bZ_{n}^{\top }\bZ_{n})^{-1}\bZ_{n}^{\top }\bX_n$, leading to the \tsls{} estimator:
\begin{align}\label{eq:tsls}
\widehat{\bb}_{\mathrm{2SLS}}
&\triangleq
    \left(\widehat{\bt}^{\top}_n \bZ_{n}^\top  \bZ_{n} \widehat{\bt}_n\right)^{-1} \widehat{\bt}^{\top}_n\bZ_{n}^\top \by_{n}\notag\\
&=
    \left(\widehat{\bX}_{n}^\top  \widehat{\bX}_{n} \right)^{-1} \widehat{\bX}_{n}^\top \by_{n}. \tag{2SLS}
\end{align}
%This specification approaches the true parameter as the sample gets large, 
Here, $\widehat{\bX}_{n}$ is the predicted covariate from first-stage. Given $\mathbb E[\bz_i  \eta_i]=0$ in the true model, we observe that
$\widehat{\bb}_{\mathrm{2SLS}}
% =
%     \underbrace{\left(\bZ_{n}^\top  \bX_{n}\right)^{-1} \bZ_{n}^\top \by_{n}}_{\text{dim: }(k \times T)\times (T \times k)\times (k \times T)\times (T\times 1) }
=
    \left(\widehat{\bX}_{n}^\top  \widehat{\bX}_{n} \right)^{-1} \widehat{\bX}_{n}^\top  \bX_{n} \bb+\left(\widehat{\bX}_{n}^\top  \widehat{\bX}_{n} \right)^{-1} \widehat{\bX}_{n}^\top  \beeta_n \overset{p}{\rightarrow} \bb,
$ 
as $n \to \infty$.
Since $\bx$ and $\beeta$ are correlated, \tsls{} estimator is not unbiased in finite-time. When $d_{\bx}=d_{\bz}$, we call $\bz$ \textit{just-identified IVs}. When $d_{\bx} < d_{\bz}$, we call $\bz$ \textit{over-identified IVs}. In this paper, we analyse both of the conditions.

\begin{assumption}\label{assumption:2sls}
The assumptions for conducting \tsls{} are~\citep{greene2003econometric,hernan2020causal}:
\begin{enumerate}[topsep=0pt,parsep=0pt,itemsep=0pt,leftmargin=*]
    %\item \textbf{Well behaved data.} For every $n \in \mathbb N$, the matrices $\bZ_{n}^\top\bZ_{n}$ and $\bZ_{n}^\top \bX_n$ are full rank, and thus invertible.
    \item \textbf{Endogeneity of $\bx$.} The second stage input features $\bx$ and noise $\eta$ are not independent: $\bx\not\indep\eta$.
    \item \textbf{Exogeneity of $\bz$.} IV random variables are independent of the noise in the second stage: $\bz\indep\eta$.
    \item \textbf{Relevance Condition of IVs}. For every $t \in \mathbb N$, the variables $\bz$ and $\bx$ are correlated: $\bz \not\indep \bx$.%, i.e.
\end{enumerate}
\end{assumption}
\citep[Sec.8.2]{greene2003econometric} and \citep[Ch. 16]{hernan2020causal} state these three as the \textit{strictly necessary conditions for the existence of IVs} in offline IV-regression. Thus, we assume only these necessary conditions to hold for our analysis of online \tsls.

\section{\otsls: Online two-stage least squares regression}\label{sec:otsls}
In this section, we describe the problem setting and schematic of \textit{Online Two-Stage Least Squares Regression}, in brief \otsls. Following that, we first define two notions of regret: \textit{identification} and \textit{oracle}, measuring accuracy of estimating the true parameter and the predicted outcomes, respectively. We provide a theoretical analysis of \otsls{} and upper bound the two types of regret (Section~\ref{sec:otsls_theory}). 
%In Sec.~\ref{App:exp}, we experimentally show that \otsls{} computes an accurate estimate of the true parameter. %, and thus, incur lower regret than Online Ridge Linear Regression, in brief \ridge.

\textbf{\otsls.} In the online setting of IV regression, the data $(\bx_1, \bz_1, y_1), \ldots, (\bx_t, \bz_t, y_t), \ldots$ arrives in a stream. 
Following \tsls{} model (Fig.~\ref{fig:pgm}), data is generated endogeneously
\begin{align}\label{model1}
        \bx_t
=   
    \bt^\top\bz_t  + \bep_t &&
    y_t 
= 
    \boldsymbol{\beta}^\top \bx_t + \eta_t,
\end{align}
such that $\bx_t \not\indep \eta_t$ and $\bz_t \indep \eta_t$ for all $t \in \mathbb N$.
At each step $t$, the online IV regression algorithm is served with a new input feature $\bx_t$ and an IV $\bz_t$. At time $t+1$, the algorithm aims to yield an estimate of the parameter $\bb_t$ and predict an outcome $\widehat{y}_{t+1} \triangleq \bb_t^\top \bx_{t+1} \in \R$ using the data $\lbrace (\bx_s, \bz_s, y_s)\rbrace_{s=1}^t$ observed so far. 
Following the prediction, Nature reveals the true outcome $y_{t+1}$. %Quality of the prediction is evaluated using a square loss ${\ell_t\left(\bb_{t}\right)}\triangleq \left(\widehat y_{t}-y_{t}\right)^{2}$~\citep{foster1991prediction}. %The online protocol is the following.
\iffalse
\begin{mybox}{
At each round $t = 1, 2, \ldots, T$
\begin{enumerate}[topsep=0pt,parsep=0pt,itemsep=0pt]
    \item $\bz_{t}$ is sampled i.i.d. from an unknown distribution
    \item $\bx_t$ is sampled according to \Cref{model1} given $\bz_{t}$
    \item we compute an estimate $\bb_{\bullet}$ and make a prediction $\widehat{y}_t= \bb_{\bullet}^{\top} \bx_t $
    \item we observe the true $y_t$ following \Cref{model1}
    \item we incur in a loss $\left(y_t -\widehat{y}_t  \right)^2 = (y_t -  \bb_{\bullet}^\top \bx_t )^2$
\end{enumerate}
}
\end{mybox}
\fi

To address endogeneity in this problem, we propose an online form of the \tsls{} estimator. Modifying Eq.~\eqref{eq:tsls}, we obtain the \otsls{} estimator that is computed for the prediction at time $t+1$, using information up to time $t$:
\begin{align}\label{eq:IV-estimator}
 \bb_{t}
&\triangleq
    \left(\widehat{\bX}_{t}^\top  \widehat{\bX}_{t} \right)^{-1} \widehat{\bX}_{t}^\top \by
    =
    \left(\sum_{s=1}^{t} \widehat{\bx}_{s} \widehat{\bx}_s^\top\right)^{-1}\,  \sum_{s=1}^{t} \widehat{\bx}_s y_s  \notag\\
&=
    (\I - \bG_{\widehat{\bx},t}^{-1} \widehat \bx_{t} \widehat \bx_{t}^\top ) \bb_{t-1} +\bG_{\widehat{\bx},t}^{-1}  {\widehat\bx}_{t} y_{t}.\tag{\otsls}%
\end{align}
Here, $\bG_{\widehat{\bx},t} \triangleq  \sum_{s=1}^{t} \widehat\bx_s \widehat \bx_s^\top = \bG_{\widehat{\bx},t-1} + \widehat\bx_{t}  \widehat\bx_{t}^\top $. We need the predicted value of covariate $\widehat\bx_{t} \triangleq \widehat{\bt}_{t-1} \bz_{t}$ and output $y_{t}$ for the iterative update of $\bb_t$ given the estimates at previous time $\bb_{t}$ and $\bG_{\widehat{\bx},t}$.
To predict the covariate $\widehat \bx_{t}$, we leverage an iterative update rule of $\widehat{\bt}_{t}$ from $\widehat{\bt}_{t-1}$ by using $\bz_{t}$, $\bx_{t}$ (\Cref{app:2sls}). In brief, for $\lambda>0$,
\begin{align}\label{eq:bt}
\widehat\bt_{t}
&= \left(\sum_{s=1}^{t-1} {\bz}_{s} {\bz}_s^\top + \lambda \I_{d_{\bz}}\right)^{-1}  \sum_{s=1}^{t-1} {\bz}_s \bx_s^{\top} 
= 
(\I -\bG_{{\bz},t-1}^{-1} \bz_t \bz_t^\top) \widehat\bt_{t-1} +\bG_{{\bz},t-1}^{-1}  {\bz}_t \bx_t^\top.
\end{align}
Finally, we use the \otsls{} estimator at step $t+1$ for the prediction $\widehat{y}_{t+1}= \bb_{t}^{\top} \bx_{t+1}$, as in Algorithm~\ref{alg:o2sls}. 

%We elaborate \otsls{} in Algorithm~\ref{alg:o2sls}. 
\setlength{\textfloatsep}{2pt}% Remove \textfloatsep
\begin{algorithm}[h!]
\caption{\otsls{}}\label{alg:o2sls}
\begin{algorithmic}[1]
\State{\textbf{Input:} Initialisation parameters {$\bb_{0}, \widehat\bt_{0}, \lambda$}}
\For{$t = 1, \ldots, T$}
\State{Observe $\bz_t$ generated i.i.d. by Nature, and $\bx_t$ sampled from Eq.~\eqref{model1} given $\bz_{t}$}
\State{Compute first-stage and second-stage estimates $\bb_{t-1}$ and $\widehat\bt_{t-1}$ as per \Cref{eq:bt} and \Cref{eq:IV-estimator}
}
\State{Predict $\widehat{y}_t= \bb_{t-1}^{\top} \bx_t$}
\State{Observe $y_t$ generated by Nature}%, and compute loss ${\ell_t\left(\bb_{t-1}\right)}$}
\EndFor
\end{algorithmic}
\end{algorithm}

\textbf{Computational complexity.}  For second-stage, we compute $\bG_{\widehat{\bx},t}^{-1}$ from $\bG_{\widehat{\bx},t-1}^{-1}$ in $O(d_{\bx}^2)$ time using the Sherman–Morrison formula. We store $\bG_{\widehat{\bx},t}^{-1}$ after each step, which requires $O(d_{\bx}^2)$ memory. However, the first-stage update has $O(d_{\bz}^2+d_{\bz} d_{\bx})$ time complexity and requires $O(d_{\bz}^2)$ memory.
Thus, \otsls{} exhibits quadratic space and time complexity.

\begin{remark}[\textit{Proper online learnin}g]
As in improper online learning algorithms, we could use $\bx_t$ and $\bz_t$ that we observe before committing to the estimate $\bb_{t}$, and use it to predict $\widehat{y}_t$~\citep{vovk2001competitive}. Since we cannot use $y_t$ for this estimate, we have to modify \tsls{} to incorporate this additional knowledge. We avoid this modification, and follow a proper online learning approach to use $\bb_{t-1}$ to predict. 
\end{remark}
\subsection{Regrets for estimation \& prediction: Identification \& oracle regrets}
To analyse the online regression algorithms, it is essential to define proper performance metrics, specifically \textit{regrets}.
Regret quantifies what an online (or sequential) algorithm cannot achieve as it does not have access to the whole dataset rather observes it step by step.
Here, we discuss and define different regrets that we leverage in our analysis of \otsls{}.

In econometrics and bio-statistics, where \tsls{} is popularly used the focus is the accurate identification of the underlying structural model $\bb$. Identifying $\bb$ leads to understanding of the underlying economic or biological causal relations and their dynamics. In ML, \cite{venkatraman2016online} applied \otsls{} for online linear system identification. Thus, given a sequence of estimators $\{\bb_t\}_{t=1}^T$ and a sequence of covariates $\{\bx_t\}_{t=1}^T$, the cost of identifying the true parameter $\bb$ can be quantified by
\begin{align}\label{eq:id_regret}
\widetilde R_T(\bb) \triangleq \sum\nolimits_{t=1}^T(\bx_t^\top \bb_{t-1}-\bx_t^\top \bb )^2.
\end{align}
We refer to $\widetilde R_T(\bb)$ as \textit{identification regret} over horizon $T$. In the just identified setting that we are considering, the {identification regret} is equivalent to the regret of counterfactual prediction~(Eqn. 5, \cite{hartford2016counterfactual}). Counterfactual predictions are important to study the causal questions: what would have changed in the outcome if Treatment $a$ had been used instead of treatment $b$? One of the modern applications of IVs is to facilitate such counterfactual predictions~\citep{hartford2016counterfactual,bennett2019deep,zhu2022causal}.

Alternatively, one might be interested in evaluating and improving the quality of prediction obtained using an estimator $\{\bb_t\}_{t=1}^T$ with respect to an underlying oracle (or expert), which is typically the case in statistical learning theory and forecasting~\citep{foster1991prediction,cesa2006prediction}. If the oracle has access to the true parameters $\bb$, the cost in terms of prediction that the estimators pay with respect to the oracle is $\bar{r}_t \triangleq (y_t-\bx_t^\top \bb_{t-1})^2- (y_t-\bx_t^\top \bb )^2$.
Thus, the regret in terms of the quality of prediction is 
\begin{align}\label{eq:cumulativeregret}
    \overline{R}_T(\bb) \triangleq \sum\nolimits_{t=1}^T(y_t-\bx_t^\top \bb_{t-1})^2- (y_t-\bx_t^\top \bb )^2.
\end{align}
We refer to $\overline{R}_T(\bb)$ as the \textit{oracle regret}. This regret is studied for stochastic analysis of online regression~\citep{ouhamma2021stochastic} and for analysing bandits~\citep{foster2020beyond}.

As \otsls{} is interesting for learning causal structures~\citep{hartford2016counterfactual,bennett2019deep}, we focus on the identification regret. On the other hand, to compare with the existing results in online linear regression, we also analyse the oracle regret of \otsls{}. Though they are of similar order (w.r.t. $T$) in the exogenous setting, we show 
% in \Cref{app:2sls}
that they differ significantly for \otsls{} under endogeneity \red{(\Cref{sec:otsls_theory})}.

\begin{remark}[\textbf{Hardness of exogeneity vs. endogeneity for prediction}]
In online learning theory focused on Empirical Risk Minimisation (ERM), another type of regret is considered where the oracle has access to the best offline estimator $\bb_T \triangleq \argmin_{\bb} \sum_{t=1}^T(y_t-\bx_t^\top \bb )^2$ given the observations over $T$ steps~\cite{cesa2006prediction}. Thus, the new formulation of regret becomes
$R_T =  \sum_{t=1}^T(y_t-\bx_t^\top \bb_{t-1})^2- \min_{\bb} \sum_{t=1}^T(y_t-\bx_t^\top \bb )^2.$ 
We refer to it as the \textit{population regret}. \textit{Under exogeneity, \cite{ouhamma2021stochastic} show that oracle regret and population regret differs by $o(\log^2 T)$}. We show that under endogeneity, their expected values differ by $\Omega(T)$. Thus, we avoid studying population regret in this paper, and detail it in \red{Appendix~\ref{sec:reg_big}}.
\end{remark}

\subsection{Theoretical analysis}\label{sec:otsls_theory}
\textbf{Confidence set.} The central result in our analysis is concentration of \otsls{} estimates $\bb_t$ around $\bb$.
\begin{lemma}[Confidence ellipsoid for the second-stage parameters]\label{thm:confidencebeta}
Let us define the \emph{design matrix of IVs} to be  $\bG_{\bz,t} \triangleq \bZ_t^{\top} \bZ_t + \bG_{\bz,0} = \sum_{s=1}^t \bz_s \bz_{s}^\top+ \bG_{\bz,0}$ with $\bG_{\bz,0}=\lambda \mathbf{I}_{d_{\bz}}$ for some $\lambda>0$. Then, for some $\sigma_{\eta}$-sub-Gaussian second stage noise $\eta_t$ and for all $t>0$, the true parameter $\bb$ belongs to the confidence set
\begin{equation}\label{eq:conf_set}
	\mathcal E_{t} \triangleq \left\{\bb \in \mathbb{R}^{d_{\bx}}:
	\| \bb_{t} - \bb \|_{\widehat{\bH}_t}
	\leq    \sqrt{\mathfrak b_{t}(\delta)}
	\right\},
\end{equation}
with probability at least $1-\delta \in (0,1)$. Here, ${\mathfrak b_{t}(\delta)} \triangleq \frac{d_{\bz} \sigma_{\eta}^2}{4} {\log \left(\frac{1+t L^{2}_z / \lambda {d_{\bz}}}{\delta}\right)}$, $\widehat{\bH}_t \triangleq \widehat {\bt}_t^\top \bG_{\bz,t} \widehat {\bt}_t$, and $\widehat {\bt}_t$ is the estimate of the first-stage parameter at time $t$.
\end{lemma}
Lemma~\ref{thm:confidencebeta} extends the well-known elliptical lemma for OLS and Ridge estimators under exogeneity to the \otsls{} estimator under endogeneity. It shows that the size of the confidence intervals induced by \otsls{} estimate at time $T$ is $\bigO(\sqrt{d_{\bz} \log T})$, which is of the same order as that of the exogenous elliptical lemma~\citep{abbasi2011improved} but while applied on IVs.

\textbf{Identification regret bound.} Now, we state the identification regret upper bound of \otsls{}. %and a brief proof sketch.

\begin{theorem}[Identification regret of \otsls{}]\label{thm:reg_otsls_identification}
If Assumption~\ref{assumption:2sls} holds true,
for bounded IVs $\|\bz\|_2^2 \leq L_z^2$ and bounded first-stage parameters $\normiii{\bt}_2 \leq L_{\bt}^2$, the $\sigma_{\eta}$-sub-Gaussian second stage noise $\eta_t$ and the component-wise $\sigma_{\bep}$-sub-Gaussian first stage noise $\bep_{t}$, 
the regret of \otsls{} satisfies with probability at least $1-\delta$

\begin{align*}
\centering
    \widetilde{R}_{T} 
&\leq\quad 
    \sum_{t=1}^T \underset{\text{Estimation}}{\underbrace{\| \bb_{t} - \bb \|^2_{\widehat{\bH}_t}}}\times \underset{\substack{\text{Second-stage}\\ \text{feature norm}}}{\underbrace{ \norm{\bx_t}^2_{\widehat{\bH}_t^{-1}}}}
\leq
    \underset{\bigO(d_{\bz} \log T)}{\underbrace{\mathfrak b_{T-1}(\delta)}}\times \underset{\bigO(d_{\bx} \log T)}{\underbrace{ 
    \sum_{t=1}^T \norm{\bx_t}^2_{\widehat{\bH}_t^{-1}}
    }} = \bigO\left(d_{\bx}d_{\bz} \log^2(T)\right).
\end{align*}
\end{theorem}

\Cref{thm:reg_otsls_identification} entails a regret $\widetilde{R}_{T} = \bigO\left(d_{\bx}d_{\bz} \log^2(T)\right)$. This regret bound is $d_{\bz} \log T$ more than the regret of online ridge regression, i.e. $\bigO(d_{\bx} \log T)$~\citep{gaillard2019uniform}. This is because we perform $d_{\bx}$ linear regressions in the first-stage and use the predictions of the first-stage for the second-stage regression. These two regression steps in cascade induce the proposed regret bound.
%\end{remark}

\textbf{Oracle regret bound.} Now, we bound the oracle regret, i.e. the goodness of predictions yielded by \otsls{}. Detailed proof is deferred to Appendix~\ref{app:2sls}.%\vspace*{1em}
\begin{theorem}[Oracle regret of \otsls{}]\label{thm:reg_otsls_oracle}
Under the same hypothesis of  \Cref{thm:reg_otsls_identification}, Oracle Regret of \otsls{} at step $T > 1$, i.e. $\overline{R}_{T} $ is upper bounded by (ignoring $\log\log$ terms)
\begin{align*}
\centering
%    \overline{R}_{T} 
%\leq
&\underset{\substack{\text{Identif.}\\ \text{Regret} \\\bigO(d_{\bx}d_{\bz} \log^2 T)}}{\underbrace{\widetilde{R}_{T}}}%\\
    + 
    \underset{\substack{\text{Estimation} \\\bigO(\sqrt{d_{\bz} \log T})}}{\underbrace{\sqrt{\mathfrak b_{T-1}(\delta)}}}
    \underset{
        \substack{\text{First-stage}\\\text{feature norm} \\\bigO(\sqrt{\log T})}
        }
    {
\Big( 
    \underbrace{C_1 
    \sqrt{
    f(T)}}
    %\log\left(\frac{\log T}{\delta}\right)}}
    } 
    +
\underbrace{C_2\sqrt{
    2   
    d_{\bx} f(T) 
    }+ 
    \sqrt{d_{\bx}} C_3
    }_{\substack{\text{Correlated noise}\\\text{Concentration term} \\\bigO(\sqrt{d_{\bx}\log T})}} +
    \underset{\substack{\text{Correlated noise}\\ \text{Bias term}\\\bigO(\gamma \sqrt{T} )}}
    {
    \underbrace{
    % \|\boldsymbol{\gamma}\|_2
    \gamma C_4 
    \sqrt{T}
    }\Big)}
\end{align*}
with probability at least $1-\delta \in (0,1)$. Here, $\gamma \triangleq\|\boldsymbol{\gamma}\|_2 = \| \mathbb E[\eta_s \bep_s] \|_2$. $C_1$, $C_2$, $C_3$, and $C_4$ are $d_{\bz}$, $d_{\bx}$ and $T$-independent positive constants. $f(T) = \bigO(\log T)$ as defined in \Cref{cor:ft}.
\end{theorem}

\noindent\textbf{Discussion.}
1. \textit{Hardness of endogeneity vs. exogeneity.} 
Under exogeneity and unbounded stochastic noise, the oracle regret of online linear regression is $\bigO(d^2 \log^2 T)$~\citep{ouhamma2021stochastic}. If we take the just-identified IVs, i.e. $d_{\bx} = d_{\bz} = d$, due to endogeneity, \otsls{} incurs an additive factor of $\bigO(\gamma \sqrt{d_{\bz} T \log T})$ in the oracle regret. This term appears due to the correlation between the second and the first-stage noises, and it is proportional to the degree of correlation between the noises in these two stages. Thus, the bias due to the correlation of noises acts as the dominant term. In \tsls{} literature, this phenomenon is called the self-fulfilling bias~\citep{li2021self}. But we did not find any explicit bound on it in a stochastic and non-asymptotic analysis.

2. \textit{Tightness of analysis.} In the existing literature, we do not have any lower bound for online regression with exogeneity. But we can indicate tightness of the proposed analysis by observing two things. (i) Our identification regret under endoegeneity is of the same order as that of exogeneous case~\citep{abbasi2012online}. (ii) If we assume $\gamma = 0$, i.e. the noises are independent, and we retrieve also the oracle regret of the same order as that of the exogenous case~\citep{ouhamma2021stochastic}.

3. \textit{Bounded vs. unbounded noise.} For bounded noise and covariates, the adversarial setting would subsume the endogeneous setting where the responses are sampled according to dependent bounded noise since covariates and responses are adversarially chosen sequences. If we consider the adversarial linear regression for bounded responses $y_t\leq Y$ and covariates, the the Vovk-Azoury-Warmuth Regressor (VAWR) achieves almost optimal $\mathcal{O}(Y^2 \log T)$ regret~\citep{orabona2019modern}. Thus, the $\gamma \sqrt{T}$ term disappears as we obtain in our oracle regret bound with unbounded and stochastic noise. This contrast indicates non-triviality of our analysis for the unbounded noise and unbounded outcomes. %that we study. %, and might be of future interest.

\section{Linear bandits with endogeneity: \ofuliv{}}\label{sec:ofuliv}
\begin{repexample}{example1}[Pricing with price-sales dynamics]
    Let revisit the price-sales dynamics. Now, on a day $t$, a restaurant owner wants to decide on a price $x_t$ among a feasible set of prices $\mathcal{X}_t$ to increase the sales $y_t$. The owner also has access to a set of suppliers such that each price correspond to compatible a material cost $z_t$. If the restaurant owner wants the analyst to develop an algorithm to decide on the price dynamically such that the total sales over a year is maximised, her problem is exactly a linear bandit with endoegeneity, where a price corresponds to an arm (action).
\end{repexample}
We formulate \textit{stochastic Linear Bandits with Endogeneity} (\textbf{LBE}) with a two-stage linear model of data generation (Eq.~\eqref{model1}). Here, we illustrate the interactive protocol of LBE. 
\begin{mybox}{
At each round $t = 1, 2, \ldots, T$, the agent 
\begin{enumerate}[topsep=0pt,parsep=0pt,itemsep=0pt]
    \item Observes a sample $\bx_{t,a} \in \cX_t$ of contexts for all $a\in\mathcal A_t$ %computed according to \Cref{model2_1}
    \item\label{2} Chooses an arm $A_t \in \cA_t$ 
    % (and therefore the corresponding vector $\bz_t=\bz_{t,A_t}$ and $\bx_t=\bx_{t,A_t}$) 
    % according to \Cref{eq:optimism}
    % \item Observes a sample $\bz_{t,A_t}\in\cZ_t$, and obtains a reward $y_t$ %from \eqref{model2_2}
    \item Observes the sample $\bz_{t,A_t}\in\cZ_t$
    \item Obtains a reward $y_t$ %from \eqref{model2_2}
    \item Updates the parameter estimates $\widehat{\bt}_t$ and $\bb_t$
    % (example $\bb_t$)
\end{enumerate}
}
\end{mybox}
Here, $\cX_t \subset \mathbb R^{d_{\bx}}$ and $\cZ_t \subset \mathbb R^{d_{\bz}}$ are the sets of covariates and IVs corresponding to $\cA_t$, i.e. the set of feasible actions at time $t$.
Similar to regression under endoegenity, $\bep_t \indep \bz_t$ and $\bep_t \not\indep \eta_t$ (Eq.~\eqref{model1}). True parameters $\bb\in \mathbb R^{d_{\bx}}$ and $\bt\in\mathbb R^{d_{\bz}\times d_{\bx}}$ are unknown to the agents. This is an extension of stochastic linear bandit~\citep[Ch. 19]{lattimore2020bandit} to the endeogenous setting.

\begin{remark}[Alternative protocol of LBE]
We can alternatively represent the protocol, where at every step $t$, the agent observes covariates $\bx_{t,a}$ and IVs $\bz_{t,a}$ for all the arms $a \in \cA_t$. Then, the agent can use all these information to choose an arm $A_t$ and observes the corresponding outcome $y_t$. But we do not require $\bz_{t,a}$ except $\bz_{t,A_t}$ to update the parameters and select arms. Thus, the alternative is reducible to the protocol above, while asking for less information.% to be known.
\end{remark}

\noindent\textbf{\ofuliv{}: Algorithm design.}
If the agent had full information in hindsight, she could infer the best arm (a.k.a. action or intervention) in $\mathcal A_t$ as $a^*_t=\argmax_{a\in\mathcal{A}_t} \mathbb E[\bx_{t,a}^\top \bb].$
Choosing $a_t^*$ is equivalent to choosing $\bz_{t,a_t^*}$ and $\bx_{t,a_t^*}$.
But the agent does not know them and aims to select $\lbrace a_t\rbrace_{t=1}^T$ to minimise regret: $
R_{T}\triangleq \mathbb{E}[\sum_{t=1}^T\bb^{\top}(\bx_{t, a_t^*} -\bx_{t,a_t})]$. Now, we extend the \oful{} algorithm minimising regret in linear bandits with exogeneity~\citep{abbasi2011improved}. The core idea is that the algorithm maintains a confidence set $\mathcal B_{t-1} \subset \R^{d_{\bx}}$ around the estimated parameter $\bb_{t-1}$, which is computed only using the observed data. In order to tackle endogeneity, we choose to use the \otsls{} estimate $\bb_{t-1}$ computed using data observed till $t-1$ (Equation~\eqref{eq:IV-estimator}). Then, we build an ellipsoid $\mathcal B_{t-1}$ around it, such that $\mathcal B_{t-1}\triangleq \left\{\bb \in \mathbb{R}^{d_{\bx}}:	\| \bb_{t-1} - \bb \|_{\widehat{\bH}_{t-1}}\leq \sqrt{\mathfrak b'_{t-1}(\delta)}\right\}$, where $\mathfrak b'_{t-1}(\delta)\triangleq 2 \sigma_{\eta}^{2} \log \left(\nicefrac{\operatorname{det}\left(\bG_{\bz,t-1}\right)^{1 / 2} \lambda^{-d_{\bz}/ 2}}{\delta}\right)$ and $\widehat{\bH}_{t-1} = \widehat {\bt}_{t-1}^\top \bG_{\bz,t-1} \widehat {\bt}_{t-1}$. %It is required that $\mathcal B_{t-1}$ can be calculated from $\bx_{1}, \bx_{2}, \ldots, \bx_{t-1}$ and $\bz_{1}, \bz_{2}, \ldots, \bz_{t-1}$ plus the rewards $y_{1}, y_{2}, \ldots, y_{t-1}$ and ``with high probability'' $\bb$ lies in $\mathcal B_{t-1}$. 
Then, the algorithm chooses an optimistic estimate $\widetilde{\bb}_{t-1}$ from that confidence set: $\widetilde{\bb}_{t-1}
\triangleq \argmax_{\bb' \in \mathcal B_{t-1}}\left(\max_{\bx \in \mathcal{X}_t } \bx^\top \bb' \right).$
Then, she chooses the action $A_t$ corresponding to $\bx_{t, A_t}=\argmax_{\bx \in \mathcal{X}_t } \bx^\top \widetilde{\bb}_{t-1}$, which maximises the reward as per the estimate  $\widetilde{\bb}_{t-1}$. In brief, the algorithm chooses the pair
$
    (\bx_{t, A_t}, \widetilde{\bb}_{t-1})
=
    \argmax\nolimits_{{(\bx, \bb') \in \mathcal X_t \times \mathcal B_{t-1}}} \bx^\top \bb'.
$
%which jointly maximises the reward.
Given confidence interval, we optimistically choose $A_t$ by solving
%according the the following maximisation
\begin{align}\label{eq:optimism}
\hspace*{-1em}   \argmax_{a\in \mathcal A_t} \left\{\left\langle \bx_{t,a}, \bb_{t-1}  \right\rangle+  \sqrt{\mathfrak b'_{t-1}(\delta)} \left\|\bx_{t,a}\right\|_{ \widehat \bH_{t-1}^{-1}}\right\}
\end{align}
This arm selection index together with the \otsls{} estimator of $\bb_{t-1}$ constitute \ofuliv{} (Algo.~\ref{alg:ofuliv}).
\setlength{\textfloatsep}{2pt}%
\begin{algorithm}[h!]
\caption{\ofuliv}\label{alg:ofuliv}
\begin{algorithmic}[1]
% \Ensure $y = x^n$
\State \textbf{Input:} Initialisation parameters {$\bb_{0}, \widehat\bt_{0},\mathfrak b'_{0}, \lambda$} 
\For{$t = 1, 2, \ldots, T$}
% \State{record $\bz_t$, $\bx_t$}
\State{Observe $\bx_{t,a} \in \mathcal{X}_t$ for $a\in\mathcal A_t$}
\State{Compute $\bb_{t-1}$ according to \Cref{eq:IV-estimator} }
\State{Choose action $A_t $ that solves \Cref{eq:optimism}}
\State Observe $\bz_{t,A_t}$ and $y_t$
%\State{Predict $\widehat{y}_t= \bb_{t-1}^{\top} \bx_t$}
\State{Update $\bb_{t}\gets\bb_{t-1}, \widehat\bt_{t}\gets \widehat\bt_{t-1},\mathfrak b'_{t}\gets \mathfrak b'_{t-1}$ }
\EndFor
\end{algorithmic}
\end{algorithm}
%\subsection{Theoretical Analysis}\label{sec:theory_ofuliv}
\begin{theorem}[Regret upper bound of \ofuliv]\label{thm:reg_ofuliv}
Under the assumptions and notations of Thm.~\ref{thm:reg_otsls_identification} and~\ref{thm:reg_otsls_oracle}, for horizon $T > 1$ and with probability $1-\delta$, \Cref{alg:ofuliv} incurs regret
\begin{align*}
    R_T 
&\leq
    2 
    \sqrt{T} \underbrace{ \sqrt{\mathfrak b_{T-1}(\delta)}
    }_{\substack{\text{Estimation}\\~\bigO( \sqrt{d_{\bz} \log T)} }}
%     \\
% &\quad\quad\quad\quad
    \underset{\substack{\text{Second-stage feature norm}\\~\bigO( \sqrt{d_{\bx} \log T})}}{
    % \left(
    \underbrace{
    \left(\sum_{t=1}^T \norm{\bx_{t, A_t}}^2_{\widehat{H}_{t}^{-1}}\right)^{1/2}
    % \right.
    } 
    } 
=
    \bigO(\sqrt{d_{\bx} d_{\bz} T}\log T).
\end{align*}
\end{theorem}
For just-identified IVs, i.e. $d_{\bx} = d_{\bz} =d$, \ofuliv{} achieves regret of similar order under endogeneity as \oful{} achieves under exogeneity, i.e. $\widetilde{\bigO}(d \sqrt{T})$, and matches the lower bounds for linear bandits w.r.t. $d$ and $T$~\citep{lattimore2020bandit}. The proof details are in App.~\ref{app:ofuliv}.

1. \textit{Handling unbounded outcomes.} We do not need bounded outcomes ($\bb^\top \bx\in[-1,1]$) as OFUL \citep{abbasi2011improved}. Similar to~\citep{ouhamma2021stochastic}, which removes this assumption for exoegeneity, our analysis works for unbounded outcomes showing that the bounded outcome assumption is not needed even under endogeneity.\\
2. \textit{Two-stage OFUL-IV vs. one-stage OFUL under endogeneity.} An alternative proposal than using a two-stage approach for LBE is to reduce it to a one-level linear bandit, i.e. $y_t =(\bt\bb)^{\top} \bz_t+ \bb^\top \bep_t +\eta_t$.
This leads to a composite unknown parameter $\bt\bb$, and a new sub-Gaussian noise with component-wise variance $\sigma_{new}^2 = 2(\|\bb\|_2^2 \sigma_{\bep}^2 + \sigma_{\eta}^2)$. 
Now, if we apply OFUL to this setup, we observe that the confidence bound around the new hidden parameter $\bt\bb$ is proportional to $\sigma_{new}$.
Thus, the confidence interval blows up by $\|\bb\|_2$. Also, using this interval in Equation~\eqref{eq:optimism} requires us to know either $\|\bb\|_2^2$ or a tight upper bound on it. In real-life, we cannot assume access to such knowledge. Additionally, $\|\bb\|_2^2$ grows linearly with $d_{\bx}$. Thus, this approach can incur significantly higher regret for high-dimensional covariates. In contrast, \ofuliv{} does not require to know either $\|\bb\|_2$ or any bound on it. Additionally, the confidence interval of \ofuliv{} depends on $\sigma_{\eta}^2$, which is significantly smaller than $\sigma_{new}^2$. These observations demonstrate the benefit of using \ofuliv{} under endogeneity than \oful{} with a single-stage reduction. 
Single-stage reduction might be helpful if the confounding noise is covariate dependent but is independent of the unknown parameter~\citep{krishnamurthy2018semiparametric}. Numerical results in Fig.~\ref{fig:ofuliv_new_main} further validates this benefit of \ofuliv{} over an one-stage reduction.

\begin{figure}[t!]
  \centering
  \begin{subfigure}[b]{0.3\textwidth}
    \includegraphics[width=\textwidth]{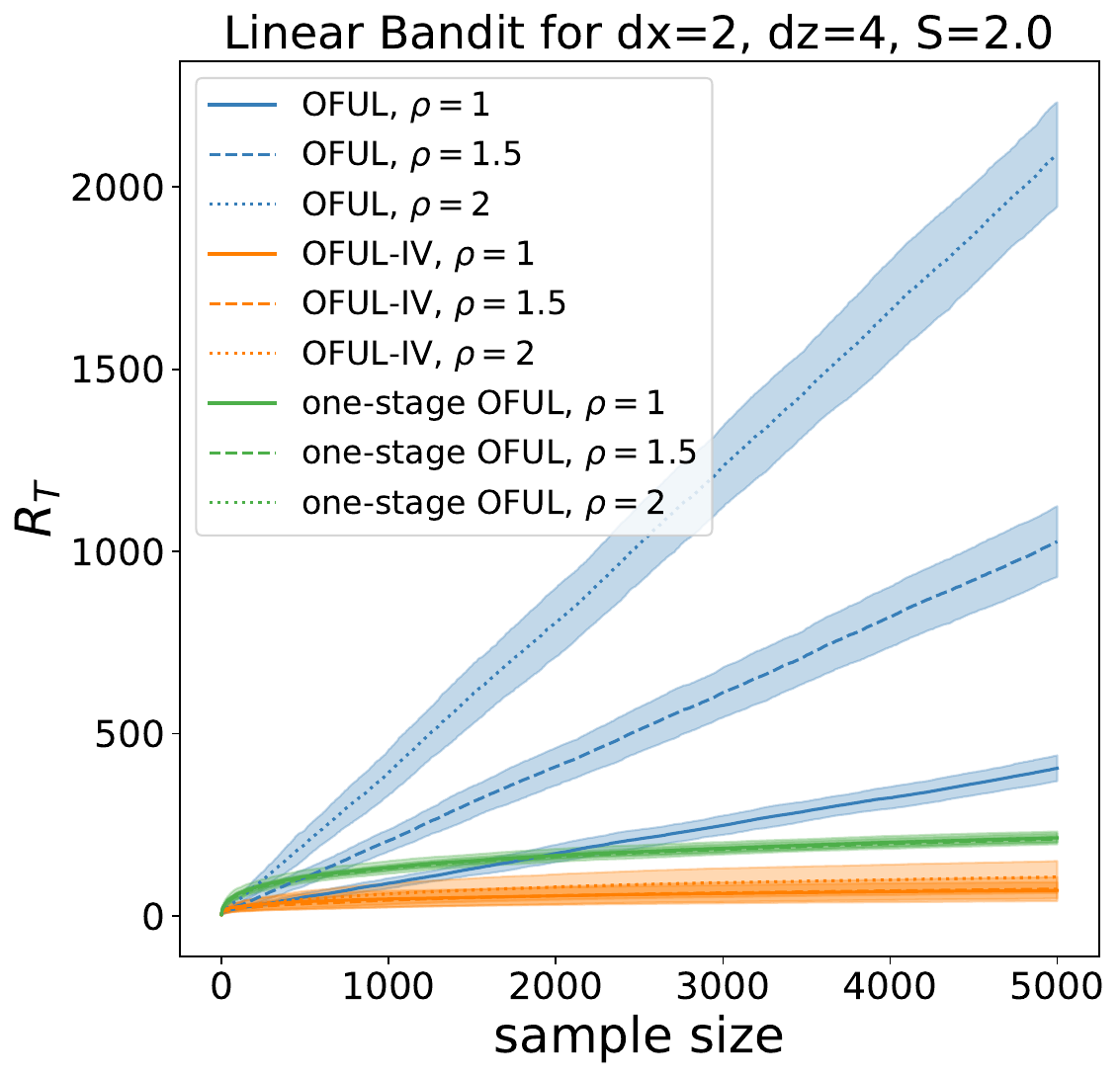}
  \end{subfigure}
  \hfill
  \begin{subfigure}[b]{0.3\textwidth}
    \includegraphics[width=\textwidth]{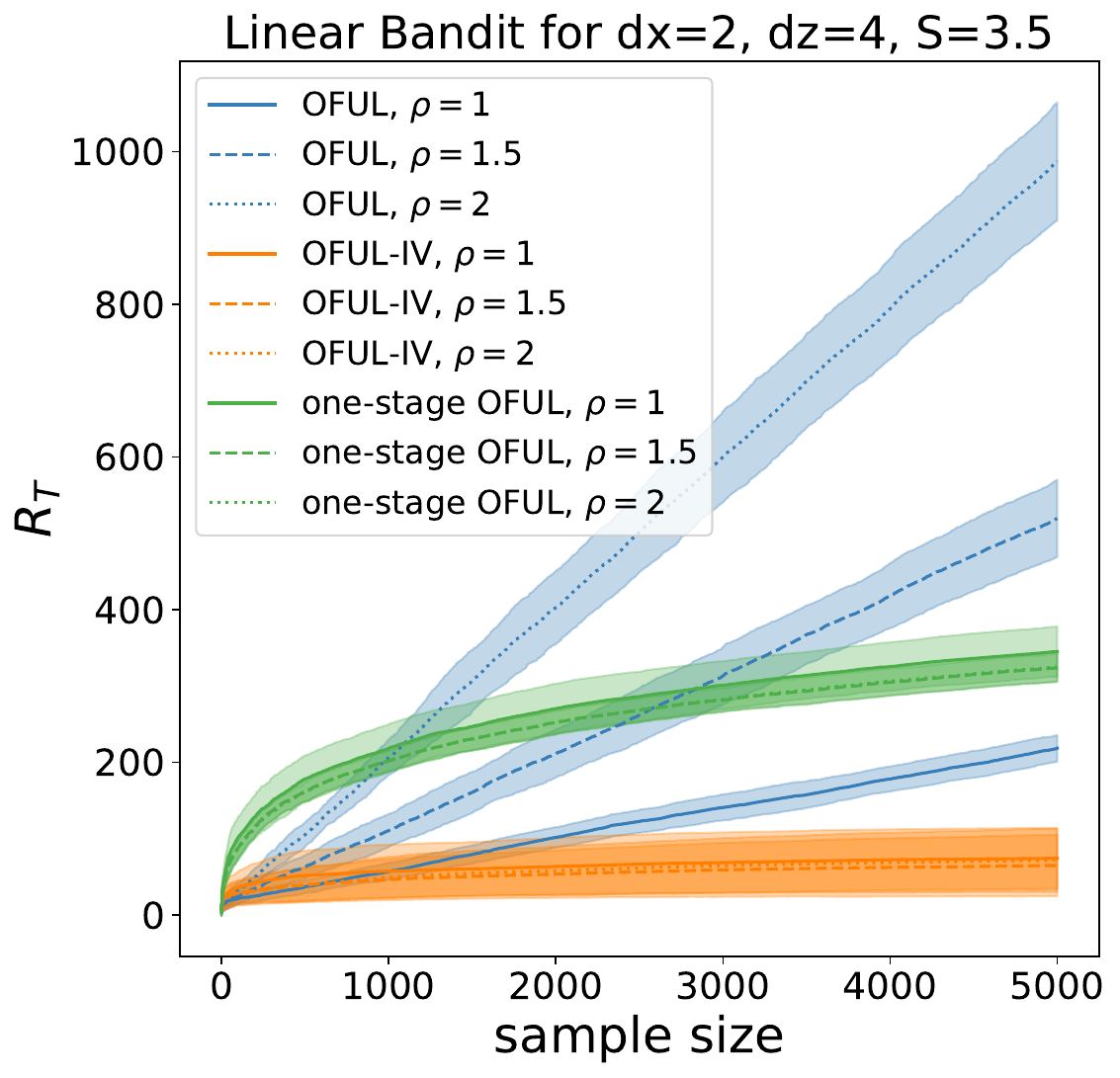}
  \end{subfigure}
  \hfill
  \begin{subfigure}[b]{0.3\textwidth}
    \includegraphics[width=\textwidth]{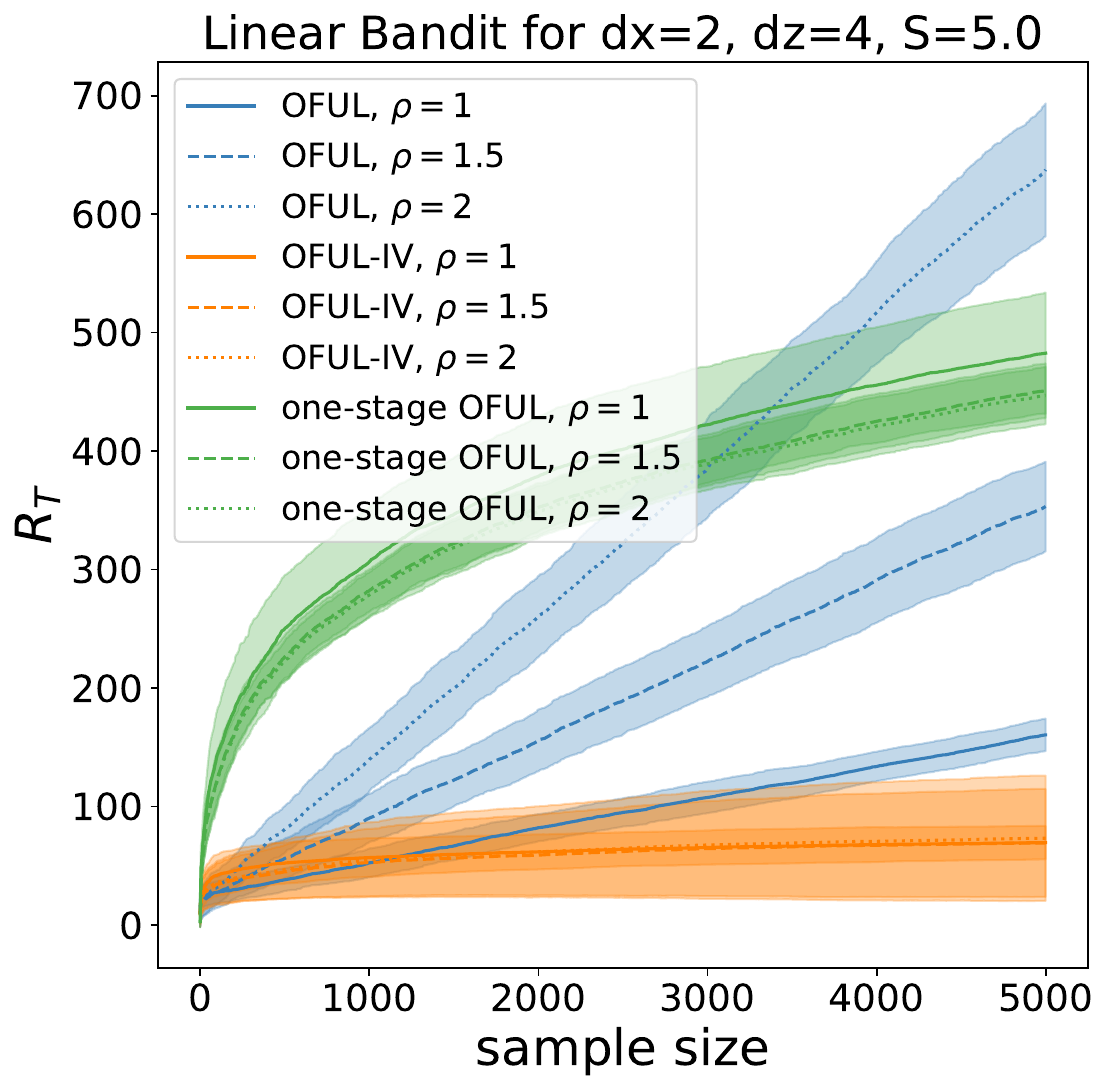}
  \end{subfigure}
  \caption{Regrets due to one-stage OFUL (green), OFUL (blue), and OFUL-IV (orange) for different norms of the hidden parameters ($S=1, 3, 5$ from left to right). Only OFUL-IV (orange) is independent of $S$ and incurs the lowest regret across endogeneity levels $\rho$.}\label{fig:ofuliv_new_main}
\end{figure}

\section{Experimental analysis}\label{sec:exp_main}
We present experimental results for both online regression and linear bandit in \Cref{Fig:learninc},~\ref{fig:lbe} and \ref{fig:ofuliv_new_main}. 
We compare the performance of \otsls{} and Online Ridge Regression (\ridge).
For LBEs, we compare the performance of \ofuliv{} and \oful{}~\citep{abbasi2011improved}.

\textbf{Experimental Setup.} 
% In \Cref{fig:ofuliv} we summarize the experimental setting for online regression and for the linear bandit case. 
We induce endogeneity in the problem in the following arbitrary way: by settings $\eta_{t} = \rho \epsilon_{t,1}+\widetilde\eta_{t}$ where $\epsilon_{t,1}$ indicates the first component of the vector $\bep_t$. Then, we can control the level of endogeneity of the first and second stages through $\rho$.
\setlength{\textfloatsep}{2pt}
\begin{table}[h!]
    \centering
    \begin{tabular}{l|l|l}
        & Online regression & Linear bandit \\\hline
        1st Stage & $\bx_{t} = \bt\, \bz_{t} + \bep_t$ & $\bx_{t,a} = \bt \bz_{t,a} + \bep_t$ \\
        2nd Stage & $y_t = \bb^\top \bx_{t} + \eta_t$ & $y_t = \bb^\top \bx_{t,A_t} + \eta_t$
    \end{tabular}
    \caption{Experimental settings.  $\rho$ controls endogeneity since $\eta_{t} = \rho  \epsilon_{t,1}+\widetilde\eta_{t}$ and $\widetilde\eta_{t}$ is exogenous noise.}\label{t:set}
\end{table}
We choose $d_{\bx}=\{2,5,8\}$ and $d_{\bz}=\{4,10,16\}$ respectively. In our experiments, we choose arbitrarily $\bb$ as a normalised vector with equal negative entries; therefore, the values in the components are uniquely determined by the dimension $d_{\bx}$.
We choose $\bt=\I_{d_{\bz},d_{\bx}}$ which has ones on  the entries $i=j$ and zeros for $i\neq j$.
Then, we sample at each time $t$ (and also for every arm $a$ for the LBE setting) the vectors $\bz_{t} \sim \mathcal N_{d_{\bz}}(\vec{0}, \I_{d_{\bz}})$ ($\bz_{t,a} \sim \mathcal N_{d_{\bz}}(\vec{0}, \I_{d_{\bz}})$), the vector noise   $\bep_{t} \sim \mathcal N_{d_{\bx}}(\vec{0}, \I_{d_{\bx}})$, and the scalar noise $\eta_{t} = \widetilde\eta_{t} + \rho \cdot \epsilon_{t,1}$ where $\widetilde\eta_{t} \sim \mathcal{N}_1(0,1)$. We run the algorithms with the same regularisation parameters, i.e. $\lambda =0.1$.
We repeat our experiments 20 times. 
We average the results, and for each algorithm, we report the mean and standard deviation of the cumulative regret (shaded areas correspond to one standard deviation).
For further experiments and results with both synthetic and real data, see \Cref{App:exp}.
% In \Cref{fig:ofuliv} we present the learning curves for the regret for different values of increasing endogeneity $\rho$.

\begin{figure}[t!]
     \centering
     \begin{subfigure}[b]{0.45\textwidth}
         \centering
         \includegraphics[width=\textwidth]{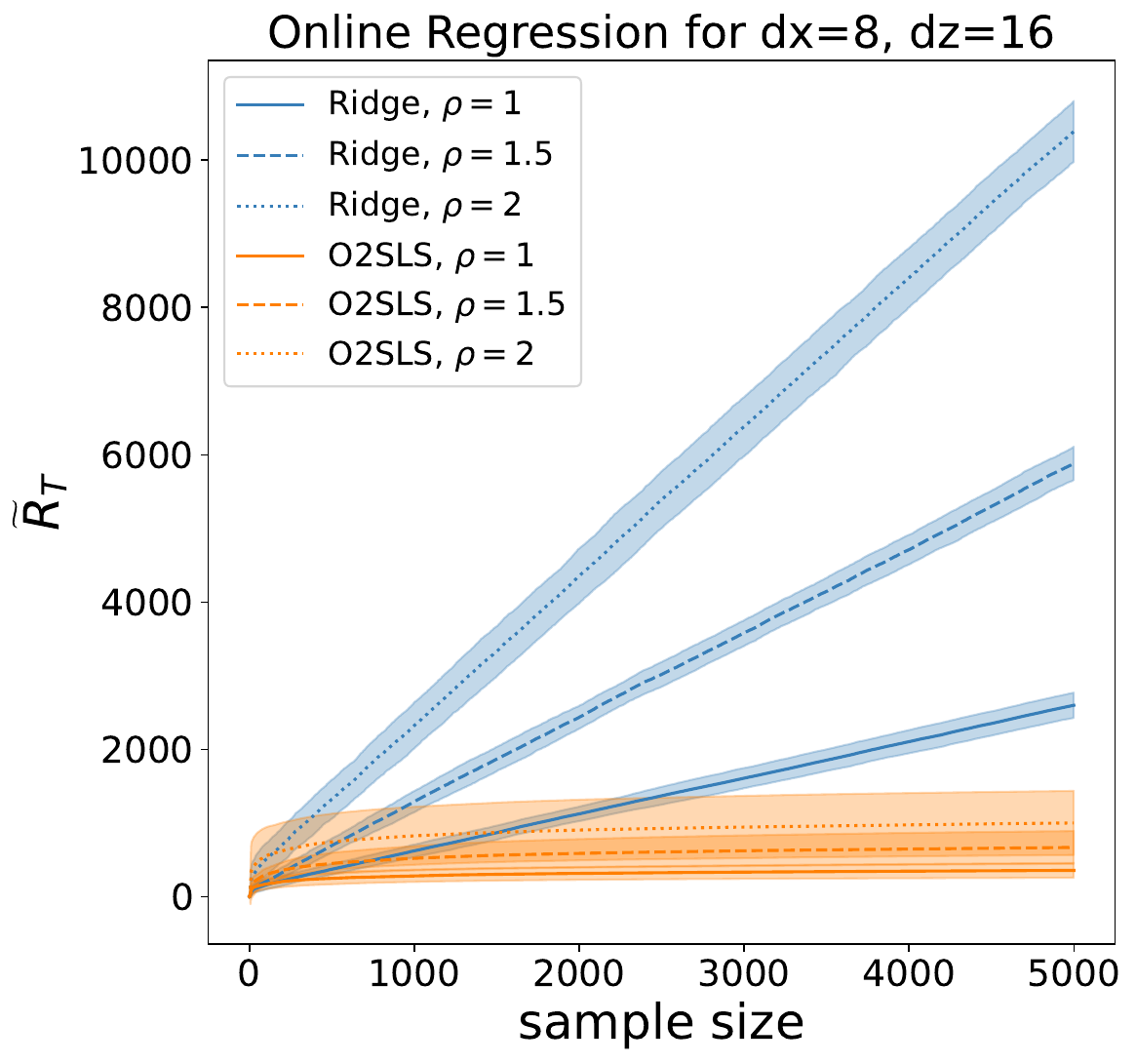}
         \caption{\textbf{Regression:} id. regret.}\label{Fig:learninc}
     \end{subfigure}
     \hfill
     \begin{subfigure}[b]{0.45\textwidth}
         \centering
        \includegraphics[width=\textwidth]{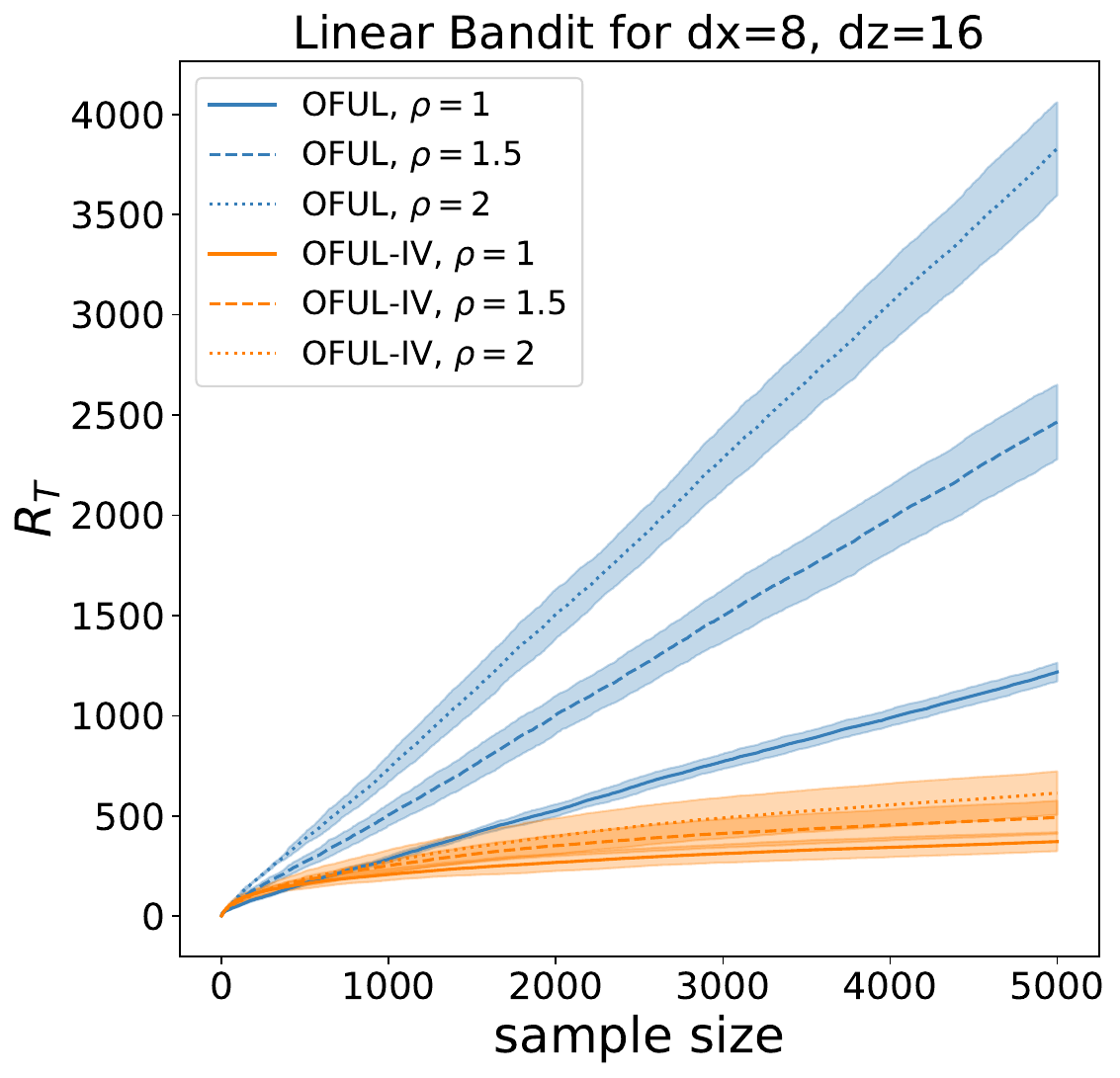}
         \caption{\textbf{LBE:} cumulative regret.}\label{fig:lbe}
     \end{subfigure}
        \caption{(Left) \textbf{Regression:} Identification regrets of Online Ridge (\textit{blue}) and \otsls{} (\textit{orange}) over $T=5000$ steps, and for $\rho = 1, 1.5, 2$. 
       \textit{With the increase in $\rho$, i.e. endogeneity, \otsls{} performs better.} 
        (Right) \textbf{LBE:} Cumulative regrets of \oful{} (\textit{blue}) and \ofuliv{} (\textit{orange}) over $T = 5000$ steps with $\rho = 1, 1.5, 2$. 
        \textit{\ofuliv{} always incurs lower regrets, and improvements w.r.t. \oful{} increases with $\rho$.}
        }\label{fig:ofuliv}%
\end{figure}

\textbf{Summary of Results.} \textit{1. Regression.} \otsls{} outperforms \ridge{} in the whole parameter space under study, and the performance-gain increases with increasing values of the $\rho$ parameter, i.e. with increasing level of endogeneity.
\textit{2. Bandits.} \oful{} builds a confidence ellipsoid centered at $\bb^{\ridge}_{t}$, while \ofuliv{} uses \otsls{} to build an accurate estimate and an ellipsoid that with high probability contains $\bb$. \Cref{fig:ofuliv} indicates that \ofuliv{} incurs lower regret than \oful{}.
\textit{3. Independence of \ofuliv{} from the norm of hidden parameters.} Figure~\ref{fig:ofuliv_new_main} shows that only OFUL-IV's regret is independent of the maximum norm of $\bb$, i.e. $S$. But OFUL (oblivious to endogeneity) and \oful{} applied on one-stage reduction of IV regression incurs higher regret with increasing values of $S$ as indicated by theory.
\iffalse
\noindent\textbf{(R3) Novelty.} 
% We highlighted in Sec 1 our contributions. 
% so we don't repeat them here; we limit to pointers to the technical hurdles we had to face.
% : 1. \textit{A non-asymptotic analysis of \otsls{}:} 
(a) Due to two-stages and endogeneity, we can't rely on standard techniques and need to prove concentration bounds for dependent r.v. (App. B.6) leading to a novel analysis for Thm.3.5 (App. B.8).
% ; in fact, we need to concentrate on correlated first and second-stage random noise 
% terms that is the product of sub-Gaussian r.v. with arbitrary dependency 
% (see Lemma B.13 \textit{}). 
(b) Lemma 3.3 derives a new confidence ellipsoid around O2SLS estimator $\bb_t$ with a new design matrix $\hat{\boldsymbol{H}}_t$. Following existing derivations for ridge would ask algorithm to know $\|\bb\bt\|_2$ unlike us.
% (instead of simply $\boldsymbol{G}_{\bx,t}$).
% Also, Lemma B.7, 
% bounding the inverse of design matrices in operatorial norm 
(c) Lemma B.7 is a new result needed for IVs.
(d) Identifying three regrets and their distinct behaviours under endogeneity is absent in literature considering exogeneity.
\fi

\section{Conclusion and future works}\label{sec:coda}
We study online IV regression, specifically online \tsls{}, for unbounded noise and endogenous data. We analyse the finite-time identification and oracle regrets of \otsls{}. We observe that \otsls{} incurs $\bigO(d_{\bx}d_{\bz} \log^2 T)$ identification regret, and $\bigO(\gamma \sqrt{d_{\bz} T \log T})$ oracle regret as the correlation between the noises in the two-stages dominate the identification regret. For just-identified IVs, identification regrets are of the same order for exogeneity and endogeneity.  We further propose \ofuliv{} for stochastic linear bandits with endogeneity that uses \otsls{} to estimate the parameters. We show that \ofuliv{} achieves $\bigO(\sqrt{d_{\bx}d_{\bz} T} \log T)$ regret. We experimentally show that \otsls{} and \ofuliv{} are more accurate than Ridge and \oful{}, respectively. 
In future, it would be interesting to extend the stochastic analysis of online IV regression and bandits with endogeneity to non-parametric~\citep{newey2003instrumental}, non-linear~\citep{xu2020learning}, and non-stationary~\citep{russac2019weighted} settings.

% References
\bibliography{uai2024-template}

\begin{thebibliography}{54}
\providecommand{\natexlab}[1]{#1}
\providecommand{\url}[1]{\texttt{#1}}
\expandafter\ifx\csname urlstyle\endcsname\relax
  \providecommand{\doi}[1]{doi: #1}\else
  \providecommand{\doi}{doi: \begingroup \urlstyle{rm}\Url}\fi

\bibitem[Abbasi-Yadkori et~al.(2011{\natexlab{a}})Abbasi-Yadkori, P{\'a}l, and
  Szepesv{\'a}ri]{abbasi2011improved}
Yasin Abbasi-Yadkori, D{\'a}vid P{\'a}l, and Csaba Szepesv{\'a}ri.
\newblock Improved algorithms for linear stochastic bandits.
\newblock \emph{Advances in neural information processing systems}, 24,
  2011{\natexlab{a}}.

\bibitem[Abbasi-Yadkori et~al.(2011{\natexlab{b}})Abbasi-Yadkori, P{\'a}l, and
  Szepesv{\'a}ri]{abbasi2011online}
Yasin Abbasi-Yadkori, D{\'a}vid P{\'a}l, and Csaba Szepesv{\'a}ri.
\newblock Online least squares estimation with self-normalized processes: An
  application to bandit problems.
\newblock \emph{arXiv preprint arXiv:1102.2670}, 2011{\natexlab{b}}.

\bibitem[Abbasi-Yadkori et~al.(2012)Abbasi-Yadkori, Pal, and
  Szepesvari]{abbasi2012online}
Yasin Abbasi-Yadkori, David Pal, and Csaba Szepesvari.
\newblock Online-to-confidence-set conversions and application to sparse
  stochastic bandits.
\newblock In \emph{Artificial Intelligence and Statistics}, pages 1--9. PMLR,
  2012.

\bibitem[Abeille and Lazaric(2017)]{Abeille2017LinearTS}
Marc Abeille and Alessandro Lazaric.
\newblock Linear thompson sampling revisited.
\newblock In \emph{AISTATS}, 2017.

\bibitem[Angrist and Imbens(1995)]{angrist1995two}
Joshua~D Angrist and Guido~W Imbens.
\newblock Two-stage least squares estimation of average causal effects in
  models with variable treatment intensity.
\newblock \emph{Journal of the American statistical Association}, 90\penalty0
  (430):\penalty0 431--442, 1995.

\bibitem[Angrist et~al.(1996)Angrist, Imbens, and
  Rubin]{angrist1996identification}
Joshua~D Angrist, Guido~W Imbens, and Donald~B Rubin.
\newblock Identification of causal effects using instrumental variables.
\newblock \emph{Journal of the American statistical Association}, 91\penalty0
  (434):\penalty0 444--455, 1996.

\bibitem[Bartlett et~al.(2015)Bartlett, Koolen, Malek, Takimoto, and
  Warmuth]{bartlett2015minimax}
Peter~L Bartlett, Wouter~M Koolen, Alan Malek, Eiji Takimoto, and Manfred~K
  Warmuth.
\newblock Minimax fixed-design linear regression.
\newblock In \emph{Conference on Learning Theory}, pages 226--239. PMLR, 2015.

\bibitem[Bennett et~al.(2019)Bennett, Kallus, and Schnabel]{bennett2019deep}
Andrew Bennett, Nathan Kallus, and Tobias Schnabel.
\newblock Deep generalized method of moments for instrumental variable
  analysis.
\newblock \emph{Advances in neural information processing systems}, 32, 2019.

\bibitem[Brier(1950)]{Brier1950VERIFICATIONOF}
Glenn~W. Brier.
\newblock Verification of forecasts expressed in terms of probability.
\newblock \emph{Monthly Weather Review}, 78:\penalty0 1--3, 1950.

\bibitem[Burgess et~al.(2017)Burgess, Small, and Thompson]{burgess2017review}
Stephen Burgess, Dylan~S Small, and Simon~G Thompson.
\newblock A review of instrumental variable estimators for mendelian
  randomization.
\newblock \emph{Statistical methods in medical research}, 26\penalty0
  (5):\penalty0 2333--2355, 2017.

\bibitem[Cesa-Bianchi and Lugosi(2006)]{cesa2006prediction}
Nicolo Cesa-Bianchi and G{\'a}bor Lugosi.
\newblock \emph{Prediction, learning, and games}.
\newblock Cambridge university press, 2006.

\bibitem[Foster(1991)]{foster1991prediction}
Dean~P Foster.
\newblock Prediction in the worst case.
\newblock \emph{The Annals of Statistics}, pages 1084--1090, 1991.

\bibitem[Foster and Rakhlin(2020)]{foster2020beyond}
Dylan Foster and Alexander Rakhlin.
\newblock Beyond {UCB}: Optimal and efficient contextual bandits with
  regression oracles.
\newblock In \emph{International Conference on Machine Learning}, pages
  3199--3210. PMLR, 2020.

\bibitem[Gaillard et~al.(2019)Gaillard, Gerchinovitz, Huard, and
  Stoltz]{gaillard2019uniform}
Pierre Gaillard, S{\'e}bastien Gerchinovitz, Malo Huard, and Gilles Stoltz.
\newblock Uniform regret bounds over $\mathbb{R}^{d}$ for the sequential linear
  regression problem with the square loss.
\newblock In \emph{Algorithmic Learning Theory}, pages 404--432. PMLR, 2019.

\bibitem[Greene(2003)]{greene2003econometric}
William~H Greene.
\newblock \emph{Econometric analysis}.
\newblock Pearson Education India, 2003.

\bibitem[Harris et~al.(2022)Harris, Ngo, Stapleton, Heidari, and
  Wu]{harris2022strategic}
Keegan Harris, Dung Daniel~T Ngo, Logan Stapleton, Hoda Heidari, and Steven Wu.
\newblock Strategic instrumental variable regression: Recovering causal
  relationships from strategic responses.
\newblock In \emph{International Conference on Machine Learning}, pages
  8502--8522. PMLR, 2022.

\bibitem[Hartford et~al.(2016)Hartford, Lewis, Leyton-Brown, and
  Taddy]{hartford2016counterfactual}
Jason Hartford, Greg Lewis, Kevin Leyton-Brown, and Matt Taddy.
\newblock Counterfactual prediction with deep instrumental variables networks.
\newblock \emph{arXiv preprint arXiv:1612.09596}, 2016.

\bibitem[Hazan and Koren(2012)]{hazan2012linear}
Elad Hazan and Tomer Koren.
\newblock Linear regression with limited observation.
\newblock In \emph{29th International Conference on Machine Learning, ICML
  2012}, pages 807--814, 2012.

\bibitem[Heij et~al.(2004)Heij, de~Boer, Franses, Kloek, and van
  Dijk]{heij2004econometric}
Christiaan Heij, Paul de~Boer, Philip~Hans Franses, Teun Kloek, and Herman~K
  van Dijk.
\newblock \emph{Econometric methods with applications in business and
  economics}.
\newblock Oxford University Press, 2004.

\bibitem[Hernan and Robins(2020)]{hernan2020causal}
MA~Hernan and J~Robins.
\newblock \emph{Causal Inference: What if}.
\newblock Boca Raton: Chapman \& Hill/CRC, 2020.

\bibitem[Kallus(2018)]{kallus2018instrument}
Nathan Kallus.
\newblock Instrument-armed bandits.
\newblock In \emph{Algorithmic Learning Theory}, pages 529--546. PMLR, 2018.

\bibitem[Kazerouni and Wein(2021)]{kazerouni2021best}
Abbas Kazerouni and Lawrence~M Wein.
\newblock Best arm identification in generalized linear bandits.
\newblock \emph{Operations Research Letters}, 49\penalty0 (3):\penalty0
  365--371, 2021.

\bibitem[Kivinen et~al.(2004)Kivinen, Smola, and Williamson]{kivinen2004online}
Jyrki Kivinen, Alexander~J Smola, and Robert~C Williamson.
\newblock Online learning with kernels.
\newblock \emph{IEEE transactions on signal processing}, 52\penalty0
  (8):\penalty0 2165--2176, 2004.

\bibitem[Krishnamurthy et~al.(2018)Krishnamurthy, Wu, and
  Syrgkanis]{krishnamurthy2018semiparametric}
Akshay Krishnamurthy, Zhiwei~Steven Wu, and Vasilis Syrgkanis.
\newblock Semiparametric contextual bandits.
\newblock In \emph{International Conference on Machine Learning}, pages
  2776--2785. PMLR, 2018.

\bibitem[Lattimore and Szepesv{\'a}ri(2020)]{lattimore2020bandit}
Tor Lattimore and Csaba Szepesv{\'a}ri.
\newblock \emph{Bandit algorithms}.
\newblock Cambridge University Press, 2020.

\bibitem[Li et~al.(2021)Li, Luo, and Zhang]{li2021self}
Jin Li, Ye~Luo, and Xiaowei Zhang.
\newblock Self-fulfilling bandits: Dynamic selection in algorithmic
  decision-making.
\newblock \emph{arXiv preprint arXiv:2108.12547}, 2021.

\bibitem[Littlestone et~al.(1991)Littlestone, Long, and
  Warmuth]{littlestone1991line}
Nicholas Littlestone, Philip~M Long, and Manfred~K Warmuth.
\newblock On-line learning of linear functions.
\newblock In \emph{Proceedings of the twenty-third annual ACM symposium on
  Theory of computing}, pages 465--475, 1991.

\bibitem[Liu et~al.(2020)Liu, Shang, and Cheng]{liu2020deep}
Ruiqi Liu, Zuofeng Shang, and Guang Cheng.
\newblock On deep instrumental variables estimate.
\newblock \emph{arXiv preprint arXiv:2004.14954}, 2020.

\bibitem[Maillard(2016)]{maillard2016self}
Odalric-Ambrym Maillard.
\newblock Self-normalization techniques for streaming confident regression,
  2016.

\bibitem[Mogstad et~al.(2021)Mogstad, Torgovitsky, and
  Walters]{mogstad2021causal}
Magne Mogstad, Alexander Torgovitsky, and Christopher~R Walters.
\newblock The causal interpretation of two-stage least squares with multiple
  instrumental variables.
\newblock \emph{American Economic Review}, 111\penalty0 (11):\penalty0
  3663--98, 2021.

\bibitem[Nareklishvili et~al.(2022)Nareklishvili, Polson, and
  Sokolov]{nareklishvili2022deep}
Maria Nareklishvili, Nicholas Polson, and Vadim Sokolov.
\newblock Deep partial least squares for iv regression.
\newblock \emph{arXiv preprint arXiv:2207.02612}, 2022.

\bibitem[Newey and Powell(2003)]{newey2003instrumental}
Whitney~K Newey and James~L Powell.
\newblock Instrumental variable estimation of nonparametric models.
\newblock \emph{Econometrica}, 71\penalty0 (5):\penalty0 1565--1578, 2003.

\bibitem[Orabona(2019)]{orabona2019modern}
Francesco Orabona.
\newblock A modern introduction to online learning.
\newblock \emph{arXiv preprint arXiv:1912.13213}, 2019.

\bibitem[Ouhamma et~al.(2021)Ouhamma, Maillard, and
  Perchet]{ouhamma2021stochastic}
Reda Ouhamma, Odalric Maillard, and Vianney Perchet.
\newblock Stochastic online linear regression: the forward algorithm to replace
  ridge.
\newblock \emph{arXiv preprint arXiv:2111.01602}, 2021.

\bibitem[Ouhamma et~al.(2022)Ouhamma, Basu, and Maillard]{ouhamma2022bilinear}
Reda Ouhamma, Debabrota Basu, and Odalric-Ambrym Maillard.
\newblock Bilinear exponential family of mdps: Frequentist regret bound with
  tractable exploration and planning.
\newblock \emph{arXiv preprint arXiv:2210.02087}, 2022.

\bibitem[Papini et~al.(2021)Papini, Tirinzoni, Restelli, Lazaric, and
  Pirotta]{papini2021leveraging}
Matteo Papini, Andrea Tirinzoni, Marcello Restelli, Alessandro Lazaric, and
  Matteo Pirotta.
\newblock Leveraging good representations in linear contextual bandits.
\newblock In \emph{International Conference on Machine Learning}, pages
  8371--8380. PMLR, 2021.

\bibitem[Rubin(1974)]{rubin1974estimating}
Donald~B Rubin.
\newblock Estimating causal effects of treatments in randomized and
  nonrandomized studies.
\newblock \emph{Journal of educational Psychology}, 66\penalty0 (5):\penalty0
  688, 1974.

\bibitem[Russac et~al.(2019)Russac, Vernade, and Capp{\'e}]{russac2019weighted}
Yoan Russac, Claire Vernade, and Olivier Capp{\'e}.
\newblock Weighted linear bandits for non-stationary environments.
\newblock \emph{Advances in Neural Information Processing Systems}, 32, 2019.

\bibitem[Shalev-Shwartz and Ben-David(2014)]{shalev2014understanding}
Shai Shalev-Shwartz and Shai Ben-David.
\newblock \emph{Understanding machine learning: From theory to algorithms}.
\newblock Cambridge university press, 2014.

\bibitem[Singh et~al.(2020)Singh, Hosanagar, and Gandhi]{singh2020machine}
Amandeep Singh, Kartik Hosanagar, and Amit Gandhi.
\newblock Machine learning instrument variables for causal inference.
\newblock In \emph{Proceedings of the 21st ACM Conference on Economics and
  Computation}, pages 835--836, 2020.

\bibitem[Stirn and Jebara(2018)]{stirn2018thompson}
Andrew Stirn and Tony Jebara.
\newblock Thompson sampling for noncompliant bandits.
\newblock \emph{arXiv preprint arXiv:1812.00856}, 2018.

\bibitem[Tewari and Murphy(2017)]{tewari2017ads}
Ambuj Tewari and Susan~A Murphy.
\newblock From ads to interventions: Contextual bandits in mobile health.
\newblock \emph{Mobile health: sensors, analytic methods, and applications},
  pages 495--517, 2017.

\bibitem[Tirinzoni et~al.(2022)Tirinzoni, Papini, Touati, Lazaric, and
  Pirotta]{tirinzoni2022scalable}
Andrea Tirinzoni, Matteo Papini, Ahmed Touati, Alessandro Lazaric, and Matteo
  Pirotta.
\newblock Scalable representation learning in linear contextual bandits with
  constant regret guarantees.
\newblock \emph{Advances in Neural Information Processing Systems},
  35:\penalty0 2307--2319, 2022.

\bibitem[Venkatraman et~al.(2016)Venkatraman, Sun, Hebert, Bagnell, and
  Boots]{venkatraman2016online}
Arun Venkatraman, Wen Sun, Martial Hebert, J~Bagnell, and Byron Boots.
\newblock Online instrumental variable regression with applications to online
  linear system identification.
\newblock In \emph{Proceedings of the AAAI Conference on Artificial
  Intelligence}, 2016.

\bibitem[Vershynin(2018)]{vershynin2018high}
Roman Vershynin.
\newblock \emph{High-dimensional probability: An introduction with applications
  in data science}, volume~47.
\newblock Cambridge university press, 2018.

\bibitem[Vovk(1997)]{vovk1997competitive}
Volodya Vovk.
\newblock Competitive on-line linear regression.
\newblock \emph{Advances in Neural Information Processing Systems}, 10, 1997.

\bibitem[Vovk(2001)]{vovk2001competitive}
Volodya Vovk.
\newblock Competitive on-line statistics.
\newblock \emph{International Statistical Review}, 69\penalty0 (2):\penalty0
  213--248, 2001.

\bibitem[Wainwright(2019)]{wainwright2019high}
Martin~J Wainwright.
\newblock \emph{High-dimensional statistics: A non-asymptotic viewpoint},
  volume~48.
\newblock Cambridge University Press, 2019.

\bibitem[Wald(1940)]{wald1940fitting}
Abraham Wald.
\newblock The fitting of straight lines if both variables are subject to error.
\newblock \emph{The annals of mathematical statistics}, 11\penalty0
  (3):\penalty0 284--300, 1940.

\bibitem[Wasserman(2004)]{wasserman2004all}
Larry Wasserman.
\newblock \emph{All of statistics: a concise course in statistical inference},
  volume~26.
\newblock Springer, 2004.

\bibitem[Wright(1928)]{wright1928tariff}
Philip~G Wright.
\newblock \emph{Tariff on animal and vegetable oils}.
\newblock Macmillan Company, New York, 1928.

\bibitem[Xu et~al.(2020)Xu, Chen, Srinivasan, de~Freitas, Doucet, and
  Gretton]{xu2020learning}
Liyuan Xu, Yutian Chen, Siddarth Srinivasan, Nando de~Freitas, Arnaud Doucet,
  and Arthur Gretton.
\newblock Learning deep features in instrumental variable regression.
\newblock \emph{arXiv preprint arXiv:2010.07154}, 2020.

\bibitem[Xu et~al.(2021)Xu, Kanagawa, and Gretton]{xu2021deep}
Liyuan Xu, Heishiro Kanagawa, and Arthur Gretton.
\newblock Deep proxy causal learning and its application to confounded bandit
  policy evaluation.
\newblock \emph{arXiv preprint arXiv:2106.03907}, 2021.

\bibitem[Zhu et~al.(2022)Zhu, Gultchin, Gretton, Kusner, and
  Silva]{zhu2022causal}
Yuchen Zhu, Limor Gultchin, Arthur Gretton, Matt~J Kusner, and Ricardo Silva.
\newblock Causal inference with treatment measurement error: a nonparametric
  instrumental variable approach.
\newblock In \emph{Uncertainty in Artificial Intelligence}, pages 2414--2424.
  PMLR, 2022.

\end{thebibliography}

\newpage

\appendix

% \todoR{}

\section{Notations and useful results}
% \subsection{Notations}
We indicate in bold vectors and matrices, e.g. the vector and matrix (matrices are also usually capitalised) $\bv\in \mathbb R^d,\bA\in \mathbb R^{d\times d}$; while scalars do not use the bold notation, e.g. the scalar $s\in \mathbb R$.
We indicate the determinant of matrix $\bA$ with $\operatorname{det}(\bA)$ and its trace with $\operatorname{Tr}(\bA)$. For a $x\in \mathbb{R}_{\geq 0}$, we indicate the function that  takes as input $x$, and gives as output the least integer greater than or equal to $x$ as $\lceil x \rceil$ (ceiling function). We indicate the identity matrix of dimension $d$ with $\I_d$ and when dimensions differ we indicate $\I_{d_1,d_2}$ a matrix  which has ones on  the entries $i=j$ and zeros for $i\neq j$.
We dedicate the following table to index all the notations used in this paper. Note that every notation is defined, when it is introduced as well.

\renewcommand{\arraystretch}{1.2}
\begin{longtable}{l l l}
\caption{Notations}\label{tab:Notation}\\
\hline
 $\bx_t$ &$\triangleq$ & Covariates\\
 $\bz_t$ &$\triangleq$ & Instrumental Variables (IVs) \\
 $y_t$ &$\triangleq$ & Outcome variable\\
 $d_{\bz}$ &$\triangleq$ & Dimension of the IVs\\
 $d_{\bx}$ &$\triangleq$ & Dimension of the covariates\\
 $\bb$ &$\triangleq$ & True parameter for the second-stage\\
 $\bt$ &$\triangleq$ & True parameter for the first-stage\\
 $\bb_t$ &$\triangleq$ & \otsls{} estimate of parameter for the second-stage at time $t$\\
 $\widehat{\bt}_t$ &$\triangleq$ & Estimate of parameter for the second-stage at time $t$\\
 $\mathfrak{r}$ &$\triangleq$ & Relevance parameter \\
 $L_{\bt}$ &$\triangleq$ & Bound on the operator norm of the true first-stage parameter \\
 $L_{\bz}$ &$\triangleq$ & Bound on the $\ell_2$ norms of IVs\\
 $\sigma_\eta^2$ &$\triangleq$ & Bound on the variance of the second-stage noise $\eta$ \\
 $\sigma_{\bep}^2$ &$\triangleq$ & Bound on the variance of the components of the first-stage noise $\bep$\\
 $C_1$ &$\triangleq$ & $\sqrt{2\left(1 + L_{\bt}^2 \mathfrak{r}^2 \right)}$.\\
 $C_2$ &$\triangleq$ & $8e^2\left(\sigma_\eta^2+\sigma_{\bep}^2\right) \frac{\mathfrak{r}}{L_{\bz}} \sqrt{\log \frac{2}{\delta}}$\\
 $C_3$&$\triangleq$ & $8e^2\left(\sigma_\eta^2+\sigma_{\bep}^2\right) \frac{\mathfrak{r}}{L_{\bz}} \sqrt{ \max\left\{ 
        \frac{1}{\lambda},
        \frac{2}{ \llmin{\mathbf{\boldsymbol{\Sigma}}}}
    \right\}} \log \frac{2}{\delta}$ \\
 $C_4$&$\triangleq$ & $ \frac{\mathfrak{r}}{L_{\bz}}
    \left(T_0 \sqrt{\frac{\llmin*{\bsig}}{\lambda}}
    +2 \sqrt{2} \right)$\\
 $C_5$ &$\triangleq$ & A $T$-independent constant of $\bigO(d_{\bx} \sigma_{\bep}^2)$ (Equation~\eqref{eq:c4})\\
 $\bG_{\bz,t}$ &$\triangleq$ & First-stage design matrix $\bG_{\bz,0} +\sum_{s=1}^{t} \bz_{s} \bz_{s}^{\top}$ \\
 $\widehat{\bH}_t$ &$\triangleq$ & Second-stage design matrix $\widehat{\bt}_t^{\top} \bZ_t^{\top} \bZ_t \widehat{\bt}_t$ \\
 $\mathbf{\boldsymbol{\Sigma}}$ &$\triangleq$ & The true covariance matrix of IVs $\E[\bz \bz^\top]$\\
 $\mathfrak b_{t}(\delta)$ &$\triangleq$ & Radius of the second-stage confidence ellipsoid at time $t$ (Equation~\eqref{eq:radi_second_stage})\\
 $\lambda$ &$\triangleq$ & Minimum eigenvalue of $\bG_{\bz,0}$, i.e. the (positive-valued) regularisation parameter\\
 $\boldsymbol{\gamma}$ &$\triangleq$ & Covariance vector of the first-and second-stage noises $\mathbb{E}\left[\bep \eta\right]$\\
 $\gamma$ &$\triangleq$ & $\ell_2$ norm of the covariance vector $\boldsymbol{\gamma}$\\
 $f(T)$ & $\triangleq$ & Bound on the sum of maximum eigenvalues of the inverse of first-stage design\\
 & & matrices till time $T$, i.e. $\mathcal{O}(\log(T))$ (Corollary~\ref{cor:ft})\\
\bottomrule
\end{longtable}
%\todo[inline]{fill in this}

\clearpage
\subsection{Vectors, matrices, and norms}
% \subsection{\protect$\ell_2$-norm for vectors and matrices}
\begin{definition}[$\ell_p$-norms]
For a vector $\bv\in \mathbb R^n$, we express its $\ell_p$-norm as $\norm{\bv}_p$ for $p\geq 0$. A special case is the Euclidean $\ell_2$-norm denoted as $\norm{\cdot}_2$, which is induced by classical scalar product on $\mathbb{R}^n$ denoted by $\langle \cdot,\cdot \rangle$.
\end{definition}

Given a rectangular matrix $\mathbf{A} \in \mathbb{R}^{n \times m}$ with $n \geq m$, we write its ordered singular values as
$$
\sigma_{\max }(\mathbf{A})=\sigma_1(\mathbf{A}) \geq \sigma_2(\mathbf{A}) \geq \cdots \geq \sigma_m(\mathbf{A})=\sigma_{\min }(\mathbf{A}) \geq 0
$$
The minimum and maximum singular values have the variational characterisation
$$
\sigma_{\max }(\mathbf{A})=\max _{\bv \in \mathbb{S}^{m-1}}\|\mathbf{A}\bv\|_2 \quad \text { and } \quad \sigma_{\min }(\mathbf{A})=\min _{\bv \in \mathbb{S}^{m-1}}\|\mathbf{A}\bv\|_2,
$$
where $\mathbb{S}^{d-1}\triangleq\left\{\bv \in \mathbb{R}^d \mid\|\bv\|_2=1\right\}$ is the Euclidean unit sphere in $\mathbb{R}^d$. 
\begin{definition}[$\ell_2$-operator norm]
The spectral or $\ell_2$-operator norm of $\mathbf{A}$ is defined as  
\begin{equation}
\normiii{\mathbf{A}}_2 \triangleq \sigma_{\max }(\mathbf{A})\;.
\end{equation}
\end{definition}
Since covariance matrices are symmetric, we also focus on the set of symmetric matrices in $\mathbb{R}^d$, denoted $\mathcal{S}^{d \times d}=\left\{\mathbf{Q} \in \mathbb{R}^{d \times d} \mid \mathbf{Q}=\mathbf{Q}^{\mathrm{T}}\right\}$, as well as the subset of positive semidefinite matrices given by
$$
\mathcal{S}_{+}^{d \times d} \triangleq \left\{\mathbf{Q} \in \mathcal{S}^{d \times d} \mid \mathbf{Q} \succeq 0\right\} .
$$
From standard linear algebra, we recall the facts that any matrix $\mathbf{Q} \in \mathcal{S}^{d \times d}$ is diagonalisable via a unitary transformation, and we use $\lambda(\mathbf{Q}) \in \mathbb{R}^d$ to denote its vector of eigenvalues, ordered as
$$
\lambda_{\max }(\mathbf{Q})=\lambda_1(\mathbf{Q}) \geq \lambda_2(\mathbf{Q}) \geq \cdots \geq \lambda_d(\mathbf{Q})=\lambda_{\min }(\mathbf{Q}) .
$$
Note that a matrix $\mathbf{Q}$ is positive semidefinite-written $\mathbf{Q} \succeq 0$ for short-if and only if $\lambda_{\min }(\mathbf{Q}) \geq 0$.

\begin{remark}[Rayleigh-Ritz variational characterisation of eigenvalues]
We remind also the Rayleigh-Ritz variational characterisation of the minimum and maximum eigenvalues-namely
$$
\lambda_{\max }(\mathbf{Q})=\max _{\bv \in \mathbb{S}^{d-1}} \bv^{\top} \mathbf{Q} \bv \quad \text { and } \quad \lambda_{\min }(\mathbf{Q})=\min _{\bv \in \mathbb{S}^{d-1}} \bv^{\top} \mathbf{Q} \bv .
$$
\end{remark}

\begin{remark}
For any symmetric matrix $\mathbf{Q}$, the $\ell_2$-operator norm can be written as
$$
\normiii{\mathbf{Q}}_2=\max \left\{\lambda_{\max }(\mathbf{Q}),\left|\lambda_{\min }(\mathbf{Q})\right|\right\},
$$
by virtue of which it inherits the variational representation
$
\normiii{\mathbf{Q}}_2=\max _{\bv \in \mathbb{S}^{d-1}}\left|\bv^{\top} \mathbf{Q} \bv\right| .
$
\end{remark}

\begin{corollary}\label{cor:eig_and_sing}
Given a rectangular matrix $\mathbf{A} \in \mathbb{R}^{n \times m}$ with $n \geq m$, suppose that we define the $m$ dimensional symmetric matrix $\mathbf{R}=\mathbf{A}^{\mathrm{T}} \mathbf{A}$. We then have the relationship
$$
\lambda_j(\mathbf{R})=\left(\sigma_j(\mathbf{A})\right)^2 \quad \text { for } j=1, \ldots, m
$$
\end{corollary}

% \subsection{Norms induced by Positive Semi-Definite Matrices}
We now introduce norms that are induced by positive semi-definite matrices in the following way.
\begin{definition}
For any vector $y \in \mathbb{R}^n$ and matrix $\bA \in \mathcal{S}_{+}^{n \times n}$, let us define the norm $\| \by \|_\bA \triangleq \sqrt{\by^T \bA \by}=\sqrt{\langle\by, \bA \by\rangle} $.
\end{definition}
 
Throughout the paper we will need often to bound matrix induced norms using $\ell_2$-norms for operators, the following results show how this can be done easily for generic matrices. We specialise this result as we need in the text.
\begin{proposition}\label{prop:tricks_norms}
Take $\bA, \bB\in \mathbb{R}^{n \times n}$, with $\bB\succeq 0 $ positive semi-definite and $\bx\in\mathbb R^n$
\begin{equation}
    \norm{\bA \bx}_\bB^2  
= 
    \norm{\bx}_{\bA^\top \bB \bA}^2 \leq \normiii{\bB}_2 \normiii{\bA}_2^2 \norm{\bx}_2^2
\end{equation}
\end{proposition}
\begin{proof}
    The equality holds since we can rewrite
    \begin{align*}
        \norm{\bA \bx}_\bB^2
    = 
        \langle \bA \bx,\bB \bA \bx\rangle
    = 
        \langle \bx, \bA^\top \bB \bA \bx\rangle
    =
         \norm{\bx}_{\bA^\top \bB \bA}^2.
    \end{align*}
The inequality follows by the definition of $\ell_2$-norms, where we further substitute $\by = \bA \bx$, to get
\begin{align*}
        \langle \bA \bx,\bB \bA \bx\rangle
    &
    =
        \frac{\langle \bA \bx,\bB \bA \bx\rangle}{\norm{\bA \bx}_2^2}\frac{\norm{\bA \bx}_2^2}{\norm{\bx}_2^2}\norm{\bx}_2^2 
    =    
        \frac{\langle \by,\bB \by\rangle}{\norm{\by}_2^2}\frac{\norm{\bA \bx}_2^2}{\norm{\bx}_2^2}\norm{\bx}_2^2 
       \leq 
        \normiii{\bB}_2 \normiii{\bA}_2^2 \norm{\bx}_2^2.
\end{align*}
We note that the inequality holds trivially for $x$ in the null space of $\bA$, 
% or $y$ in the null space of $\bB$
 therefore, in the previous case, we can safely divide by $\norm{\bA \bx}_2$ and $\norm{\bx}_2$.
\end{proof}

% \subsection{Technical Lemmas}\label{sec:confidencebeta}

\begin{lemma}[Determinant-Trace Inequality]\label{lem:det-trace} 
Suppose $\bz_{1}, \bz_{2}, \ldots, \bz_{t} \in \mathbb{R}^{d}$ and for any $1 \leq s \leq t$, $\left\|\bz_{s}\right\|_{2} \leq L_z$. Let $\bG_{\bz,t}=\lambda \I_{d_{\bz}}+\sum_{s=1}^{t} \bz_{s} \bz_{s}^{\top}$ for some $\lambda>0$. Then,
\begin{equation}\label{eq:trace_det}
        \operatorname{det}\left(\bG_{\bz,t}\right) 
\leq
    \left(\lambda+t L^{2}_z / d_{\bz}\right)^{d_{\bz}}
\end{equation}
\end{lemma}
\begin{proof} 
Let $\alpha_{1}, \alpha_{2}, \ldots, \alpha_{d_{\bz}}$ be the eigenvalues of $\bG_{\bz,t}$. Since $\bG_{\bz,t}$ is positive definite, its eigenvalues are positive. Also, note that $\operatorname{det}(\bG_{\bz,t})=\prod_{s=1}^{{d_{\bz}}} \alpha_{s}$ and $\operatorname{Tr}\left(\bG_{\bz,t}\right)=\sum_{s=1}^{{d_{\bz}}} \alpha_{s}$. By inequality of arithmetic and geometric means,
\[
\sqrt[{d_{\bz}}]{\alpha_{1} \alpha_{2} \cdots \alpha_{d_{\bz}}} \leq \frac{\alpha_{1}+\alpha_{2}+\cdots+\alpha_{d_{\bz}}}{{d_{\bz}}} .
\]
Therefore, $\operatorname{det}\left(\bG_{\bz,t}\right) \leq\left(\operatorname{Tr}\left(\bG_{\bz,t}\right) / {d_{\bz}}\right)^{{d_{\bz}}}$. 
Now, it remains to upper bound the trace:
\[
    \operatorname{Tr}\left(\bG_{\bz,t}\right)
=
    \operatorname{Tr}(\lambda \I_{d_{\bz}})+\sum_{s=1}^{t} \operatorname{Tr}\left(\bz_{s} \bz_{s}^{\top}\right)
=
    {d_{\bz}} \lambda+\sum_{s=1}^{t}\left\|\bz_{s}\right\|_{2}^{2} 
\leq
    {d_{\bz}} \lambda+t L^{2}_z.
\]
\end{proof}

As a direct corollary of Weyl's theorem for eigenvalues, we have the following \citep{wainwright2019high}. 
\begin{lemma}\label{lem:weyl}
For two symmetric matrices $\mathbf A$ and $\mathbf B$
\begin{equation}\label{ineq_eigen_norm}
    \left|\llmin{\mathbf{A}}-\llmin{\mathbf{B}}\right| 
\leq 
    \normiii*{\mathbf{A}-\mathbf{B}}_2.
\end{equation}
\end{lemma}

\newpage

\subsection{Random variables and concentration bounds}
Random variables follow the previous convention if they are scalar, vectors, or random matrix variables.
We adopt the following convention when we talk about sub-Gaussians and sub-exponential random variables.
\begin{definition}[Sub-Gaussian r.v.]
   A random variable $X$ with mean $\mu=\mathbb{E}[X]$ is sub-Gaussian if there is a positive number $\sigma$ such that
$$
\mathbb{E}\left[e^{\lambda(X-\mu)}\right] \leq e^{\sigma^2 \lambda^2 / 2} \quad \text { for all } \lambda \in \mathbb{R} \text {. }
$$
\end{definition}

\begin{definition}[Sub-exponential r.v.]
    A random variable $X$ with mean $\mu=\mathbb{E}[X]$ is sub-exponential if there are non-negative parameters $(\nu, \alpha)$ such that
$$
\mathbb{E}\left[e^{\lambda(X-\mu)}\right] \leq e^{\nu^2 \lambda^2 / 2}\quad \text { for all }|\lambda|<\frac{1}{\alpha}
$$
\end{definition}

\begin{lemma}[Square and product of non-independent sub-Gaussian random variables]\label{lem:squareproductnnind} Given two non-independent random variables $X\sim$ sub-Gaussian$(\sigma_X)$ and $Y\sim$ sub-Gaussian$(\sigma_Y)$, we have the following properties:\\
\noindent\textit{1.} the squared random variable
$
X^2 $ is sub-exp$\left( 4\sqrt{2}  \sigma^2, 4 \sigma^2\right);
$\\
\noindent\textit{2.} the recentered random variable
$
XY$ is sub-exp$\left(4\sqrt{2} \left(\sigma_X^2+\sigma_Y^2\right), 2\left(\sigma_X^2+\sigma_Y^2\right)\right)
$.
\end{lemma}

\begin{proof}
We establish the proof for both statements sequentially, leveraging the first result to demonstrate the second in the scenario of non-independent random variables. This distinction gives rise to distinct constants compared to the outcome observed for independent random variables.

For the first point  we bound the rescaled $p$-th power of $X$
\begin{align*}
    \mathbb{E}\left[\left|\nicefrac{X}{\sigma_{X} \sqrt{2}}\right|^p\right]
&=
    \int_0^{\infty} \mathbb{P}\left\{\left|\nicefrac{X}{\sigma_{X} \sqrt{2}}\right|^p 
    \geq
    u\right\} d u
    \tag{integral identity for positive r.v.}
    \\
&=
    \frac{1}{\sqrt{2 }\sigma_{X}} \int_0^{\infty} \mathbb{P}\left\{|X| 
    \geq
    t\sqrt{2} \sigma_{X}\right\} pt^{p-1} d t 
    \tag{change of variable $u=t^p$}
    \\
&\leq 
    \int_0^{\infty} 2 e^{-t^2} p t^{p-1} d t \tag{by $\sigma_{X}$-sub-Gaussianity} \\
&=
    p \Gamma(p / 2) \tag{ set $t^2=s$ and use definition of Gamma function) }
%     \\
% &\leq 
%     p(p / 2)^{p / 2} \tag{since $\boldsymbol{\gamma}(x) \leq x^x$ by Stirling's approximation)}
\end{align*}
By multiplying both sides by $(\sqrt{2}\sigma)^p$ we obtain
$
    {\mathbb{E}\left[\left| X\right|^p\right]}
\leq 
    p 2^{\frac{p}{2}} \sigma^p \Gamma(p/2).$
    
Let $W=X^2$ and $\mu_W=\mathbb{E}[W]$. By power series expansion and since $\Gamma(r)=(r-1)!$ for an integer $r$, we have:
\begin{align*}
    \mathbb{E}\left[e^{\lambda\left(W-\mu_W\right)}\right] 
&=
    1+\lambda\mathbb{E}\left[W-\mu_W\right]+\sum_{r=2}^{\infty} \frac{\lambda^r \mathbb{E}\left[\left(W-\mu_W\right)^r\right]}{r !} 
 \leq 
    1+\sum_{r=2}^{\infty} \frac{\lambda^r \mathbb{E}\left[|X|^{2 r}\right]}{r !} \\
 &\leq
    1+\sum_{r=2}^{\infty} \frac{\lambda^r 2 r 2^r \sigma^{2 r} \Gamma(r)}{r !} 
=
    1+\sum_{r=2}^{\infty}\lambda^r 2^{r+1} \sigma^{2 r} 
=
    1+\frac{8\lambda^2 \sigma^4}{1-2\lambda \sigma^2}
\end{align*}
By making $|\lambda| \leq 1 /\left(4 \sigma^2\right)$, we have $1 /\left(1-2\lambda \sigma^2\right) \leq 2$. Finally, since $(\forall \alpha) 1+\alpha \leq e^\alpha$, we have that  a sub-Gaussian variable $X$ with parameter $\sigma$ is sub-exponential with parameters $(4\sqrt{2} \sigma^2, 4 \sigma^2)$, in fact we have:
\begin{align*}
     \mathbb{E}\left[e^{\lambda\left(X^2-\mathbb{E}\left[X^2\right]\right)}\right] 
\leq 
     e^{16\lambda^2 \sigma^4}
\quad
    \forall|\lambda| \leq 1 /\left(4 \sigma^2\right).
\end{align*}
To prove the second point 
we define $Z_1 \triangleq \left(\frac{X-Y}{2}\right)^2-\mathbb{E}\left[\left(\frac{X-Y}{2}\right)^2\right]$ and $Z_2 \triangleq \left(\frac{X+Y}{2}\right)^2-\mathbb{E}\left[\left(\frac{X+Y}{2}\right)^2\right]$, then
$    X Y-\mathbb{E}[X Y]
=
%     \left(\frac{X-Y}{2}\right)^2-\mathbb{E}\left[\left(\frac{X-Y}{2}\right)^2\right]
% -
%     \Bigg( \left(\frac{X+Y}{2}\right)^2-\mathbb{E}\left[\left(\frac{X+Y}{2}\right)^2\right]
%     \Bigg)
 Z_1-Z_2$.
We take care of the dependence between $X$ and $Y$ using Cauchy-Schwarz inequality instead of the product rule for expectations of independent random variables:
\begin{align*}
    \mathbb{E}\left[e^{\lambda(X+Y)}\right]
{\leq} 
    \sqrt{\mathbb{E}\left[e^{2 \lambda X}\right]} \sqrt{\mathbb{E}\left[e^{2 \lambda Y}\right]} 
\leq 
    \sqrt{e^{\frac{4 \lambda^2 \sigma_X^2}{2}}} \sqrt{e^{\frac{4 \lambda^2 \sigma_X^2}{2}}}
=
    e^{\frac{\lambda^2}{2}\left[2\left(\sigma_X^2+\sigma_Y^2\right)\right]}
\end{align*}
where we used the sub-Gaussianity of $X$ and $Y$ separately.
This proves for rescaled variables that
both
$ X+Y$  and $X-Y$ are sub-Gaussian$(\sqrt{2} \sqrt{\sigma_X^2+\sigma_Y^2})$, therefore their rescaled versions
 $\nicefrac{(X+Y)}{2}$ and  $\nicefrac{(X-Y)}{2}$ are   sub-Gaussian$(\nicefrac{\sqrt{\sigma_{X}^2+\sigma_Y^2}}{\sqrt{2}})$.
We use the previous result on the square of sub-Gaussian random variables to show that $Z_1$ and $Z_2$ are both 
$\text{sub-exp}\left( 2\sqrt{2} (\sigma_X^2+\sigma_Y^2), 2(\sigma_X^2+\sigma_Y^2)\right).
$
We use again Cauchy-Schwarz since it holds also for dependent $Z_1$ and $Z_2$ 
$$
    \mathbb{E}\left[e^{\lambda(X Y-E[X Y)]}\right]
=
    \mathbb{E}\left[e^{\lambda\left(Z_1-Z_2\right)}\right]
\leq 
    \sqrt{\mathbb{E}\left[e^{2 \lambda Z_1}\right] \mathbb{E}\left[e^{- 2 \lambda Z_2}\right]} 
\leq 
    e^{+\frac{4 \lambda^2}{2} 8\left(\sigma_X^2+\sigma_Y^2\right)^2} 
$$
which holds for $\lambda \leq \nicefrac{1}{2\left(\sigma_X^2+\sigma_Y^2\right)}$. This proves that  the sub-exponential parameters are indeed
$\nu= 4\sqrt{2} \left(\sigma_X^2+\sigma_Y^2\right)$ and $\alpha= 2\left(\sigma_X^2+\sigma_Y^2\right)$.
\end{proof}

The following are known concentration results, see  \citep{wainwright2019high}.

\begin{theorem}[Concentration of martingale difference sequences]\label{thm:wain_martingale_diff}
Let $\left\{\left(D_k, \mathcal{F}_k\right)\right\}_{k=1}^{\infty}$ be a martingale difference sequence, and suppose that $\mathbb{E}\left[e^{\lambda D_k} \mid \mathcal{F}_{k-1}\right] \leq e^{\lambda^2 \nu_k^2 / 2}$ almost surely for any $|\lambda|<1 / \alpha_k$. Then the following hold:\\
% \begin{itemize}[topsep=0pt,parsep=0pt,itemsep=0pt]
\noindent\textit{1.} the sum $\sum_{k=1}^n D_k$ is sub-exponential$(\sqrt{\sum_{k=1}^n \nu_k^2},\alpha_*)$ where $\alpha_*\triangleq \max\nolimits_{k=1, \ldots, n} \alpha_k$;\\
\noindent \textit{2.} the sum satisfies the concentration inequality
$$
\mathbb{P}\left[\left|\sum_{k=1}^n D_k\right| \geq t\right] \leq \begin{cases}2 e^{-\frac{r^2}{2 \sum_{k=1}^n \nu_k^2}} &  \text{if}  \quad 0 \leq t \leq \frac{\sum_{k=1}^n \nu_k^2}{\alpha_*} \\ 2 e^{-\frac{1}{2 \alpha_*}} &  \text{if} \quad t>\frac{\sum_{k=1}^n \nu_k^2}{\alpha_*} \end{cases}.
$$
\end{theorem} 

\begin{theorem}[Estimation of covariance matrices]\label{thm:estimation_cov}
Let $\bz_1, \ldots, \bz_t$ be i.i.d. zero-mean random vectors with covariance $\boldsymbol{\Sigma}$ such that $\left\|\bz_s\right\|_2 \leq L_{\bz}$ almost surely. Then for all $\delta>0$, the sample covariance matrix $\wbs_t=\frac{1}{t} \sum_{s=1}^t \bz_s \bz_s^{\top}$ satisfies
\begin{equation*}
    \mathbb{P}\left[
    \normiii*{
    \wbs_t-\boldsymbol{\Sigma}
    }_2 
\geq 
    \delta\right] \leq 2 d \exp \left(-\frac{t \delta^2}{\left(2 L_{\bz}^2 \normiii*{\boldsymbol{\Sigma}}_2+\delta\right)}\right),
\end{equation*}
this means that with probability at least $1-\delta$
\begin{equation*}
    \normiii*{
    \wbs_t-\boldsymbol{\Sigma}
    }_2 
\leq
    \frac{4L_{\bz}^2}{t} \log \left( \frac{2d}{\delta}\right)
    + 
    2 \sqrt{
    \frac{2L_{\bz}^2}{t} \log \left( \frac{2d_{\bz}}{\delta}\right) \normiii*{\boldsymbol{\Sigma}}_2
    }.
\end{equation*}
\end{theorem}

Finally, we have two important theorems whose proofs can be found in \citep{abbasi2011improved}.

\begin{theorem}[Self-Normalised Bound for Vector-Valued Martingales]\label{thm:vecmartbound}
Let $\left\{\mathcal F_{t}\right\}_{t=0}^{\infty}$ be a filtration. Let $\left\{\eta_{t}\right\}_{t=1}^{\infty}$ be a real-valued stochastic process such that $\eta_{t}$ is $\mathcal F_{t}$-measurable and $\eta_{t}$ is conditionally $\sigma_{\eta}$-sub-Gaussian for some $\sigma_{\eta}\geq 0$ i.e.
$
\forall \lambda \in \mathbb{R} $ holds 
$\mathbb E\left[e^{\lambda \eta_{t}} \mid \mathcal F_{t-1}\right] \leq \exp \left(\nicefrac{\lambda^{2} \sigma_{\eta}^{2}}{2}\right)
.$
Let $\left\{\bz_t\right\}_{t=1}^{\infty}$ be an $\mathbb{R}^{d}$-valued stochastic process such that $\bz_t$ is $\mathcal F_{t-1}$-measurable. For any $t \geq 0$, define
$
    \bv_{t}=\sum_{s=1}^{t} \eta_{s} \bz_s.
$
% Let $\tau$ be a stopping time with respect to the filtration $\left\{\mathcal F_{t}\right\}_{t=0}^{\infty}$. 
Then, for any $\delta>0$, with probability at least $1-\delta$, for all  $t \geq 0$,
\[
    \left\|\bv_{t}\right\|_{\bG_{\bz,t}^{-1}} 
\leq 
    \sqrt{2 \sigma_{\eta}^{2} \log \left(\frac{\operatorname{det}\left(\bG_{\bz,t}\right)^{1 / 2} \lambda^{-d / 2}}{\delta}\right)}
\leq 
    \sqrt{ 2 {d_{\bz}\sigma_{\eta}^2} \log \left(\frac{1+t L^{2}_z / d_{\bz} \lambda}{\delta}\right)}
\]
\end{theorem}

\begin{theorem}[Self-normalised Bound for Scalar Valued Martingales]\label{thm:uniform-bound}
Under the same assumptions as the previous theorem, let $\left\{\mathcal F_{t}\right\}_{t=0}^{\infty}$ be a filtration. Let $\left\{\eta_{t}\right\}_{t=1}^{\infty}$ be a real-valued stochastic process such that $\eta_{t}$ is $\mathcal F_{t}$-measurable and $\eta_{t}$ is conditionally $\sigma_{\eta}$-sub-Gaussian for some $\sigma_{\eta}\geq 0$ i.e.
$
\forall \lambda \in \mathbb{R} $ holds 
$\mathbb E\left[e^{\lambda \eta_{t}} \mid \mathcal F_{t-1}\right] \leq \exp \left(\nicefrac{\lambda^{2} \sigma_{\eta}^{2}}{2}\right)
.$
Let $\left\{w_t\right\}_{t=1}^{\infty}$ be a scalar-valued stochastic process such that $w_t$ is $\mathcal F_{t-1}$-measurable. 
Then, for any $\delta>0$, with probability at least $1-\delta$, for all  $t \geq 0$,

\begin{equation}
    \left|\sum_{s=1}^{t} \eta_{t} w_t \right| 
\leq
    \sigma_{\eta} \sqrt{2
    \left(
    1 / \sigma_{\eta}^2
    +\sum_{s=1}^{t}w_s
    \right) 
    \log \left(\frac{\sqrt{1+\sigma_{\eta}^2 \sum_{s=1}^{t}w_s^{2}}}{\delta}\right)}
\end{equation}
\end{theorem}

\newpage
\section{Regret analysis of \otsls{}}\label{app:2sls}

Instrumental variables is a powerful technique employed in econometrics to address endogeneity concerns when estimating causal relationships between variables. One computational approach commonly used is \tsls{}. In the first stage of 2SLS, each endogenous covariate in the equation of interest is regressed on all exogenous variables in the model, including both exogenous covariates within the equation of interest and the excluded instruments. This step aims to establish the relationship between the endogenous variables and the instruments. In the second stage, the regression of interest is estimated, wherein each endogenous covariate is substituted with the predicted values obtained from the first stage. By employing this two-stage procedure, 2SLS effectively addresses endogeneity issues and provides consistent estimates of the causal effects of interest.

In the following we show the relation between the just-identified and over-identified case of IVs and show the form that the \tsls{} assumes in the two cases.

\subsection{A primer on just-identified and over-identified IVs and \tsls{}}

\paragraph{Just-identified case [$d_{\bx}=d_{\bz}=d$]:} In the just-identified case, instruments and covariates have the same dimension $d$, i.e. for each covariate, there is an instrument that is correlated to it. Given an observational dataset $\{\bx_s, y_s\}_{s=1}^t$ consisting of $t$ pairs of input features and outcomes, such that $y_s \in \R$ and $\bx_s \in \R^d$. These inputs and outcomes are stochastically generated using a linear model
\begin{equation*}
	y_{s}=\bb^\top \bx_{s} +\eta_{s},
\end{equation*}
It is assumed that the error terms $\eta_i$ are independently and identically distributed and have bounded variance $\sigma^2$.
In the presence of endogeneity ($\bx$ and the noise $\beeta$ are correlated) Ordinary Least Square $\bb$~\citep{wasserman2004all} defined as $\widehat{\bb} \triangleq \argmin_{\bb'} \sum_{s=1} (y_s - \bb'^\top \bx_s)^2$
is biased and asymptotically inconsistent.
To compute an unbiased estimate of $\bb$ under endogeneity, a popular technique is to introduce the Instrumental Variables (IVs). IVs are chosen such that they are highly correlated with endogenous components of $\bx$ (relevance condition) but are independent of the noise $\eta$ (exogeneity condition for $\bz$).
In practice, this dependence is modelled with a confounding unobserved random variable $\bep$ by 
\begin{equation*}
    \bX_{t}= \bZ_{t}\bt  + \bE_t,
\end{equation*}
where $\bt\in\mathbb R^{d_{\bz}\times d_{\bx}}$ is an unknown first-stage parameter matrix and $\bE_t \triangleq [\bep_{1}, \ldots, \bep_{t}]^{\top}$ is the unobserved noise matrix leading to confounding in the second stage.
The most common IV specification uses the following estimator:
\begin{equation}
    \widehat{\bb}_{\mathrm{IV},t}
\triangleq
    \left(\bZ_{t}^\top  \bX_{t}\right)^{-1} \bZ_{t}^\top \by_{t}. \tag{IV}
\end{equation}
%This specification approaches the true parameter as the sample gets large, 
As long as $\mathbb E[\bz_i  \eta_i]=\boldsymbol{0}$ in the true model, 
$    \widehat{\bb}_{\mathrm{IV},t}
% =
%     \underbrace{\left(\bZ_{t}^\top  \bX_{t}\right)^{-1} \bZ_{t}^\top \by_{t}}_{\text{dim: }(k \times T)\times (T \times k)\times (k \times T)\times (T\times 1) }
=
    \left(\bZ_{t}^\top  \bX_{t}\right)^{-1} \bZ_{t}^\top  \bX_{t} \bb+\left(\bZ_{t}^\top  \bX_{t}\right)^{-1} \bZ_{t}^\top  \beeta_t \overset{p}{\rightarrow} \bb,$
as $t \to \infty$.
This works because IV solves for the unique parameter that satisfies $\frac{1}{n}\bZ_{t}^\top \beeta_t\overset{p}{\rightarrow}\boldsymbol{0}$. Since $\bx$ and $\beeta$ are correlated, IV estimator is not unbiased in finite-time. 

\red{\emph{Relevance condition} for IVs ensures this. It basically means that \textbf{IVs and covariates must have some correlation}, and thus, IVs explain some of the variance observed in the endogenous covariates $\bx_t$~\citep[Chapter 16]{hernan2020causal}. Since we are interested in finite time guarantees, we assume ``positive-definite covariance" between IVs and covariates at any finite time. %We add a detailed discussion. %Now, we further explain the reviewers' concerns on three different levels.\\

Specifically, we use the following form of the relevance condition:
\begin{align}\label{eq:relevance}
    \sigma_{\min}\left( \mathrm{Cov}(\bZ_t, \bX_t) \right) = \sigma_{\min}\left(\frac{1}{t} \sum_{s=1}^t \bz_s\bx_s^{\top}\right)\geq {\mathfrak{r}} > 0\,,\qquad \forall t \in \mathbb{N}.
\end{align}
}
\begin{remark}[Requirement of the relevance condition]
\red{Here, we discuss different perspectives and results justifying the use of the relevance condition in the online IV regression setting like its offline counterpart.

\textit{1)} \textit{Causality perspective:} Relevance is one of the three \textit{strictly necessary conditions for the existence of IVs} even in offline IV-regression~\citep{greene2003econometric,hernan2020causal}. Without relevance, IV regressor is not defined. For linear case, [Martens et al.'06, Epidemiology] show that relevance exists for IVs causally associated with covariates.

\textit{2)} \textit{Experimental grounding:} We confirm that this assumption is good and realistic. Because our regression analysis and the bandit algorithm's design are based on this assumption, and in both settings, our algorithms outperform the natural baselines (plus the one-stage reduction in fig) while not needing to know the exact amount of relevance.

\textit{3)} \textit{Expressing relevance:} \citep{greene2003econometric} states that the relevance condition can also be written as a condition of rank maximality for covariance matrices, which we express as the minimum singular value being bounded away from zero a.s.}
\end{remark}

\paragraph{Over-identified case [$d_{\bx}<d_{\bz}$]:} Suppose that there are more instruments than there are covariates in the equation of interest, so that $\bZ_t$ is a $t \times d_{\bz}$ matrix with $d_{\bz}>d_{\bx}$. This is often called the over-identified case. In this case, the generalised method of moments (GMM) can be used. The GMM IV estimator is
\begin{align}\label{eq:gmm_estimate}
    {\bb}_{\mathrm{GMM},t}\triangleq\left(\bX_t^\top  \bP_{t,\bZ_t} \bX_t\right)^{-1} \bX_{t}^\top  \bP_{t,\bZ_t} \by_t,\tag{GMM IV}
\end{align}
where $\bP_{t,\bZ_t}$ refers to the projection matrix $\bP_{t,\bZ_t}\triangleq\bZ_t\left(\bZ_t^\top  \bZ_t\right)^{-1} \bZ_t^\top $.
This expression collapses to the first when the number of instruments is equal to the number of covariates in the equation of interest. The over-identified IV is, therefore a generalisation of the just-identified IV. Defining GMM IV estimator requires an additional well-behavedness assumption on IVs.
\begin{assumption}[Well-behaved IVs]
    For any $t>0$, the matrix $\bZ_t^{\top}\bZ_t$ is full rank with $\lambda_{\min}(\bG_{\bz,0}) \triangleq \lambda_0 >0$.
\end{assumption}

\begin{proposition}
If $d_{\bx} = d_{\bz}$, i.e. in the just-identified case, $\bb_{\mathrm{GMM},t}$ reduces to $\bb_{\mathrm{IV},t}$ .
\end{proposition}
\begin{proof}
Expanding the $\bb_{GM M}$ expression:
\[
{\bb}_{\mathrm{GMM},t}=\left(\bX_{t}^\top  \bZ_t\left(\bZ_t^\top  \bZ_t\right)^{-1} \bZ_t^\top  \bX_{t}\right)^{-1} \bX_{t}^\top  \bZ_t\left(\bZ_t^\top  \bZ_t\right)^{-1} \bZ_t^\top  \by_t
\]
In the just-identified case, we have as many instruments as covariates, so that the dimension of $\bX_{t}$ is the same as that of $\bZ_t$. Hence, $\bX_{t}^\top  \bZ_t, \bZ_t^\top  \bZ_t$ and $\bZ_t^\top  \bX_{t}$ are all squared matrices of the same dimension. We can expand the inverse, using the fact that, for any invertible $n$-by-n matrices $\bA$ and $\bB,(\bA \bB)^{-1}=\bB^{-1} \bA^{-1}$:
\begin{align*}
{\bb}_{\mathrm{GMM},t} &=\left(\bZ_t^\top \bX_{t}\right)^{-1}\left(\bZ_t^\top  \bZ_t\right)\left(\bX_{t}^\top  \bZ_t\right)^{-1} \bX_{t}^\top  \bZ_t\left(\bZ_t^\top  \bZ_t\right)^{-1} \bZ_t^\top  \by_t \\
&=\left(\bZ_t^\top  \bX_{t}\right)^{-1}\left(\bZ_t^\top  \bZ_t\right)\left(\bZ_t^\top  \bZ_t\right)^{-1} \bZ_t^\top  \by_t \\
&=\left(\bZ_t^\top  \bX_{t}\right)^{-1} \bZ_t^\top  \by_t \\
&=\bb_{\mathrm{IV},t}
\end{align*}
\end{proof}

\paragraph{Interpretation as two-stage least-squares.}
One computational method which can be used to calculate IV estimates is two-stage least squares \tsls{}. In the first stage, each explanatory variable that is an endogenous covariate in the equation of interest is regressed on all of the exogenous variables in the model, including both exogenous covariates in the equation of interest and the excluded instruments. The predicted values from these regressions are obtained. Then, in the second stage, the regression of interest is estimated as usual, except that in this stage each endogenous covariate is replaced with the predicted values from the first stage.

\begin{proposition}
For $d_{\bx} \leq d_{\bz}$, i.e. in both just- and over-identified cases, $\bb_{\mathrm{GMM},t}= \bb_{\textrm{2SLS},t} $.
\end{proposition}

\begin{proof}
The proof follows in two steps by simple substitution of the relevant quantities.

\textit{Stage 1:} Regressing each column of $\bX_{t}$ on $\bZ_t$ with a regularisation parameter $\lambda$, in the first stage 
$$\bX_{t}=\bZ_t \boldsymbol{\Theta}+\bE_t$$
yields
\begin{align}\label{eq:first_stage_estimate}
\widehat{\bt}_t\triangleq\left(\bZ_t^\top  \bZ_t+ \bG_{\bz,0}\right)^{-1} \bZ_t^\top  \bX_t    
\end{align}
and save the predicted values:
$$
\widehat{\bX}_t\triangleq\bZ_t \widehat{\bt}_t=\bZ_t\left(\bZ_t^\top  \bZ_t + \bG_{\bz,0}\right)^{-1} \bZ_t^\top  \bX_t \triangleq \bP_{t,\bZ_t} \bX_t .
$$

\textit{Stage 2:} Regress $\by_t$ on the predicted values from the first stage:
$$
\by_t=\widehat{\bX}_{t} \bb+ \be_t
$$
which gives:
$$
\bb_{\textrm{2SLS},t}=\left(\widehat \bX_{t}^\top \widehat \bX_{t}\right)^{-1} \widehat \bX_{t}^\top  \by_t = \left(\bX_{t}^\top  \bP_{t,\bZ_t} \bX_{t}\right)^{-1} \bX_{t}^\top  \bP_{t,\bZ_t} \by_t
$$
where the second equality follows from the fact that 
the usual OLS estimator is: $(\widehat{\bX}_{t}^\top  \widehat{\bX}_{t})^{-1} \widehat{\bX}_{t}^\top  \by_t$, and replacing $\widehat{\bX}_{t}=\bP_{t,\bZ_t} \bX_{t}$ and noting that $\bP_{t,\bZ_t}$ is a symmetric and idempotent matrix, so that $\bP_{t,\bZ_t}^\top  \bP_{t,\bZ_t}=\bP_{t,\bZ_t} \bP_{t,\bZ_t}=\bP_{t,\bZ_t}$ we have 
\begin{align*}
    \bb_{\textrm{2SLS},t}
&=
    \left(\widehat{\bX}_{t}^\top  \widehat{\bX}_{t}\right)^{-1} \widehat{\bX}_{t}^\top  \by_t=\left(\bX_{t}^\top  \bP_{t,\bZ_t}^\top  \bP_{t,\bZ_t} \bX_{t}\right)^{-1} \bX_{t}^\top  \bP_{t,\bZ_t}^\top  \by_t\\
&=
    \left(\bX_{t}^\top  \bP_{t,\bZ_t} \bX_{t}\right)^{-1} \bX_{t}^\top  \bP_{t,\bZ_t} \by_t\\
&=
    \bb_{\textrm{GMM},t}
\end{align*}
\end{proof}

For brevity, in the proofs hereafter, we denote an \otsls{} estimate at time $t$, i.e. $\bb_{\text{2SLS}, t}$ as $\bb_t$.

\newpage
\subsection{Parameter estimation and concentration in the first-stage}\label{app:fs}

The first stage consists of a set of multiple regressions, which with the  choice of regularizer $\lambda \operatorname{Tr}(\bt'^{\top} \bt')$, keeps the $d_{\bx}$ regressions independent and corresponds to a ridge regularization for each of them. The derivation of the confidence set is an extension of the results from~\citep{abbasi2011improved} for the first stage parameter and uses techniques from self-normalized processes to estimate the confidence ellipsoid. Define
$$
L_t(\bt')=\lambda \operatorname{Tr}\left(\bt'^{\top} \bt'\right)+\sum_{s=1}^{t} \operatorname{Tr}\left(\left(\bx_{s}-\bt'^{\top} \bz_s\right)\left(\bx_{s}-\bt'^{\top} \bz_s\right)^{\top}\right) .
% =
%    \left(\sum_{s=1}^t \bz_s\bz_s^{\top}+\lambda \mathbf{I}_{d_{\bz}}\right)^{-1} \sum_{s=1}^t \b
%     .
$$
Let $\widehat{\bt}_t$ be the $\ell^2$-regularized least-squares estimate of $\bt$ with regularization parameter $\lambda>0$ :
\begin{align}
    \widehat{\bt}_t
=
    \underset{\bt'}{\operatorname{argmin}}\, L_t(\bt')=\bG_{\bz,t}^{-1} \bZ_t^{\top} \bX_t \: , \label{eq:first-stage-ridge}
\end{align}
where $\bG_{\bz,t}\triangleq \bG_{\bz,0} + \sum_{s=1}^{t} \bz_s \bz_s^{\top} = \lambda \mathbf{I}_{d_{\bz}}+\sum_{s=1}^{t} \bz_s \bz_s^{\top}$.
% where $\bZ_t$ and $\bX_t$ are the matrices whose rows are $\bz_1^{\top}, \ldots, \bz_{t}^{\top}$ and $\bx_1^{\top}, \ldots, \bx_t^{\top}$, respectively.

\begin{theorem}[Confidence ellipsoid for columns in first-stage]\label{thm:first-stage-ridge}
Consider the $\ell^2$-regularized least-squares parameter estimate $\widehat{\bt}_t$ with regularization coefficient $\lambda>0$. Let
$
\bG_{\bz,t}=\lambda \mathbf{I}_{d_{\bz}}+\sum_{s=1}^{t} \bz_s \bz_s^{\top}
$
be the regularized design matrix of the IVs and $\operatorname{Tr}\left(\bt^{\top} \bt\right) \leq S^2$. Define
$$
\sqrt{\beta_t(\delta)}=d_{\bx} \sigma_{\bep}  \sqrt{2 \log \left(\frac{\operatorname{det}\left(\bG_{\bz, t}\right)^{1 / 2} \operatorname{det}(\lambda \mathbf{I}_{d_{\bz}})^{-1 / 2}}{\delta}\right)}+\lambda^{1 / 2} S .
$$
Then, for any $0<\delta<1$, with probability at least $1-\delta$,
$$
\operatorname{Tr}\left(\left(\widehat{\bt}_t-\bt\right)^{\top} \bG_{\bz,t}\left(\widehat{\bt}_t-\bt\right)\right) \leq \beta_t(\delta) .
$$
In particular, $\mathbb{P}\left(\bt \in \mathcal{C}_t(\delta), t=1,2, \ldots\right) \geq 1-\delta$, where
$$
\mathcal{C}_t(\delta)=\left\{\bt \in \mathbb{R}^{d_{\bz} \times d_{\bx}}: \operatorname{Tr}\left(\left(\bt-\widehat{\bt}_t\right)^{\top} \bG_{\bz,t}\left(\bt-\widehat{\bt}_t\right)\right) \leq \beta_t(\delta)\right\} .
$$
\end{theorem}

% \todo[inline]{add some transition sentence}

\clearpage

\subsection{Elliptical lemma for the second-stage in endogenous setting}

\begin{replemma}{thm:confidencebeta}[Second-stage confidence ellipsoid for endogeneous setting]
Given an \otsls{} estimate $\bb_t$ at step $t>0$, the true parameter $\bb$ with probability at least $1-\delta$ belongs to
\begin{equation}
	\mathcal E_{t}=\left\{\bb \in \mathbb{R}^{d_{\bx}}:
	\| \bb_{t} - \bb \|_{\widehat{\bH}_t}
	\leq    \sqrt{\mathfrak b_{t}(\delta)}
	\right\},
\end{equation}
Here, we define the second-stage design matrix to be $\widehat{\bH}_t \triangleq \widehat{\bt}_t^{\top} \bZ_t^{\top} \bZ_t \widehat{\bt}_t$, and the first-stage design matrix to be $\bG_{\bz,t} \triangleq \bZ_t^{\top} \bZ_t + \bG_{\bz,0} = \sum_{s=1}^t \bz_s \bz_{s}^\top+ \bG_{\bz,0}$ with $\bG_{\bz,0}=\lambda \mathbf{I}_{d_{\bz}}$ for any $t>0$, and
\begin{align}\label{eq:radi_second_stage}
\mathfrak b_{t}(\delta)
&\triangleq 2 {d_{\bz}\sigma^2_\eta} \log \left(\frac{1+t L^{2}_{\bz} / d_{\bz}\lambda}{\delta}\right) \: .
\end{align}
\end{replemma}

\begin{proof} $ $ \\
\textbf{Step 1:} First, we start by observing the \otsls{} estimate of the second-stage parameter $\bb$:
\begin{align*}
\bb_t &=\left(\bX_t^{\top} \bZ_t\left(\bZ_t^{\top} \bZ_t + \bG_{\bz,0}\right)^{-1} \bZ_t^{\top} \bX_t\right)^{-1} \bX_t^{\top} \bZ_t\left(\bZ_t^{\top} \bZ_t + \bG_{\bz,0}\right)^{-1} \bZ_t^{\top} \by_t \\
& =\bb+\left(\bX_t^{\top} \bZ_t\left(\bZ_t^{\top} \bZ_t + \bG_{\bz,0}\right)^{-1} \bZ_t^{\top} \bX_t\right)^{-1}\bX_t^{\top} \bZ_t\left(\bZ_t^{\top} \bZ_t + \bG_{\bz,0}\right)^{-1} \bZ_t^{\top} \be_t
\end{align*}
We would like to remind that following the classical GMM estimator, we assume $\bZ_t^{\top} \bZ_t +\bG_{\bz,0}$ is invertible for any $t>0$, and $\lambda_{\min}(\bG_{\bz,0}) = \lambda >0$.

\textbf{Step 2:} Thus, for any vector $\bx \in \R^{d_{\bx}}$, we can express
\begin{align*}
\bx^{\top}\left(\bb_t -\bb\right)
& =\bx^{\top}\left(\bX_t^{\top} \bZ_t\left(\bZ_t^{\top} \bZ_t + \bG_{\bz,0}\right)^{-1} \bZ_t^{\top} \bX_t\right)^{-1} \left(\bX_t^{\top} \bZ_t\right)\left(\bZ_t^{\top} \bZ_t + \bG_{\bz,0}\right)^{-1} \bZ_t^{\top} \be_t \\
& =\bx^{\top}\left(\widehat{\bt}_t^{\top}\left(\bZ_t^{\top} \bZ_t + \bG_{\bz,0}\right)^{\top}\left(\bZ_t^{\top} \bZ_t + \bG_{\bz,0}\right)^{-1}\left(\bZ_t^{\top} \bZ_t + \bG_{\bz,0}\right) \widehat{\bt}_t\right)^{-1} \\
& \quad \left(\bX_t^{\top} \bZ_t\right)\left(\bZ_t^{\top} \bZ_t + \bG_{\bz,0}\right)^{-1} \bZ_t^{\top} \be_t \\
& =\bx^{\top}\left(\widehat{\bt}_t^{\top} (\bZ_t^{\top} \bZ_t + \bG_{\bz,0}) \widehat{\bt}_t\right)^{-1} \left(\bX_t^{\top} \bZ_t\right)\left(\bZ_t^{\top} \bZ_t + \bG_{\bz,0}\right)^{-1} \bZ_t^{\top} \be_t \\
&=\langle\left(\bZ_t^{\top} \bX_t\right)(\underbrace{\widehat{\bt}_t^{\top} (\bZ_t^{\top} \bZ_t + \bG_{\bz,0}) \widehat{\bt}_t}_{\widehat{\bH}_t})^{-\top} \bx, (\underbrace{\bZ_t^{\top} \bZ_t + \bG_{\bz,0}}_{\bG_{\bz,t}})^{-1} \bZ_t^{\top} \be_t\rangle \\
& =\left\langle\left(\bZ_t^{\top} \bX_t\right) \widehat{\bH}_t^{-\top} \bx, \bG_{\bz,t}^{-1} \bZ_t \be_t\right\rangle\\
&\underset{(a)}{\leq} \norm{\left(\bZ_t^{\top} \bX_t\right) \widehat{\bH}_t^{-\top} \bx}_{{\bG_{\bz,t}}^{-1}} \norm{{\bG_{\bz,t}}^{-1} \bZ_t \be_t}_{{\bG_{\bz,t}}}\\
&\underset{(b)}{=} \norm{\left(\bZ_t^{\top} \bX_t\right) \widehat{\bH}_t^{-\top} \bx}_{{\bG_{\bz,t}}^{-1}} \norm{\bZ_t \be_t}_{{\bG_{\bz,t}}^{-1}}\\
&= \left(\bx^{\top} \widehat{\bH}_t^{-1} (\bX_t^{\top} \bZ_t) \left(\bZ_t^{\top} \bZ_t\right)^{-1}\left(\bZ_t^{\top} \bX_t\right) \widehat{\bH}_t^{-\top} \bx\right)\norm{\bZ_t \be_t}_{{\bG_{\bz,t}}^{-1}}\\
&\underset{(c)}{=} \left(\bx^{\top} \widehat{\bH}_t^{-1} (\bX_t^{\top} \bZ_t) \widehat{\bt}_t^{\top} \widehat{\bH}_t^{-\top} \bx\right)\norm{\bZ_t \be_t}_{{\bG_{\bz,t}}^{-1}}\\
&= \left(\bx^{\top} \widehat{\bH}_t^{-1} \widehat{\bt}_t^{\top} (\bZ_t^{\top} \bZ_t)^{-1} \widehat{\bt}_t \widehat{\bH}_t^{-\top} \bx\right)\norm{\bZ_t \be_t}_{{\bG_{\bz,t}}^{-1}}\\
&= \left(\bx^{\top} \widehat{\bH}_t^{-1} \widehat{\bH}_t \widehat{\bH}_t^{-1} \bx\right)\norm{\bZ_t \be_t}_{{\bG_{\bz,t}}^{-1}}\\
&= \norm{\bx}_{\widehat{\bH}_t^{-1}} \norm{\bZ_t \be_t}_{{\bG_{\bz,t}}^{-1}}
\end{align*}
Here, we define the first-stage design matrix to be $\bG_{\bz,t} \triangleq \bZ_t^{\top} \bZ_t$, and the second-stage design matrix to be $\widehat{\bH}_t \triangleq \widehat{\bt}_t^{\top} (\bZ_t^{\top} \bZ_t + \bG_{\bz,0})\widehat{\bt}_t= \widehat{\bt}_t^{\top} \bG_{\bz,t} \widehat{\bt}_t$ for any $t>0$.

Inequality (a) is obtained from the Cauchy-Schwartz inequality.
Equality (b) is due to the fact that $\norm{\mathbf{A}^{-1} \bx}_{\mathbf{A}} =\norm{\bx}_{\mathbf{A}^{-1}}$. Equality (c) is due to the definition of $\widehat{\bt}_t$ (Eq.~\eqref{eq:first_stage_estimate}).

\textbf{Step 3:} Now, let us choose $\bx \triangleq \widehat{\bH}_t\left(\bb_t-\bb\right)$.

This leads to
\begin{align*}
    \mathrm{LHS} = \left(\bb_t-\bb\right)^{\top}\widehat{\bH}_t\left(\bb_t-\bb\right) = \left\|\bb_t-\bb\right\|^2_{\widehat{\bH}_t},
\end{align*}
and
\begin{align*}
\mathrm{RHS} &= \norm{\widehat{\bH}_t(\bb_t-\bb)}_{\widehat{\bH}_t^{-1}} \normiii{\bZ_t \be_t}_{{\bG_{\bz,t}}^{-1}}\\
&=\norm{\bb_t-\bb}_{\widehat{\bH}_t} \normiii{\bZ_t \be_t}_{{\bG_{\bz,t}}^{-1}}
\end{align*}

\textbf{Step 4:} Combining the two results above lead to
\begin{align*}
\left\|\bb_t-\bb\right\|_{\widehat{\bH}_t} &\leq\left\|\bZ_t \be_t\right\|_{\bG_{\bz,t}^{-1}}%\\
\end{align*}
 We bound the last term assigning $\eta_0=0$, using \Cref{thm:vecmartbound} for the first inequality, and \Cref{lem:det-trace} in the second inequality.
\begin{align*}
    \norm*{\bZ_{t}^\top \be_{t}}_{\bG_{\bz,t}^{-1}}
=
    \norm*{\sum_{s=1}^{t} \eta_{s} \bz_s}_{\bG_{\bz,t}^{-1}}
&\leq 
    \sqrt{2 \sigma_{\eta}^{2} \log \left(\frac{\operatorname{det}\left(\bG_{\bz,t}\right)^{1 / 2} \lambda^{-d_{\bz} / 2}}{\delta}\right)}\\
&\leq 
    \sqrt{ 2 {d_{\bz}\sigma_{\eta}^2} \log \left(\frac{1+t L^{2}_z / d_{\bz} \lambda}{\delta}\right)}.
\end{align*}

The second inequality is obtained from the concentration of the vector-valued martingales with probability at least $1-\delta$.
The final inequality is obtained by assuming $\bG_{\bz,0}=\lambda \mathbf{I}_{d_{\bz}}$. 
\end{proof}

\begin{remark}[$t$ and $d$ dependence.]
We note that the ellipsoid bound has the following order in $d$ and $t$ while neglecting the constants:
\begin{align*}
    % \norm*{ \bb_{t} - \bb }^2_{ {\widehat{\bt}_t}^\top \bG_{\bz,t} {\widehat{\bt}_t}}
    \mathfrak b_{t}(\delta)
&=  
    \mathcal O
    \left(
        d_{\bz} \log(t)
    \right).
\end{align*}
Thus, $\mathfrak b_{t}(\delta)$ is monotonically increasing with $t$. We widely leverage this property in the proofs.
\end{remark}

\newpage
\subsection{Concentration of the minimum eigenvalue of the design matrices}\label{app:concentrationdesignmatrix}
The aim of the section is to find a concentration result for the minimum eigenvalue of the design matrix, which, in turn, gives us a concentration of the $\ell_2$-norm of the inverse of the design matrix $\normiii*{\bG_{\bz,t}^{-1}}_2$. 
We use a classical concentration result for the covariance matrix \Cref{thm:estimation_cov} together with the bound on the difference of the minimum eigenvalues of two symmetric matrices  \Cref{lem:weyl} in order to bound the maximum eigenvalue of the inverse of the design matrix.

\begin{remark}[Diversity condition of IVs]
\red{In this section, we consider the covariance $\boldsymbol{\Sigma}$, i.e. the true covariance matrix of the IVs vectors, to be only positive definite a.s.
This is widely used in the stochastic online regression and linear bandit literature~\citep{papini2021leveraging,tirinzoni2022scalable}, and is known as the \emph{diversity condition}.}
\end{remark}

\begin{lemma}[Well-behavedness of First-stage Design Matrix]\label{prop:concentration_eigenvalue} 
Let $\bz_1, \ldots, \bz_t$ be i.i.d. zero-mean random vectors with covariance $\boldsymbol{\Sigma}$ 
such that $\left\|\bz_s\right\|_2 \leq L_{\bz}$ almost surely. We denote the regularised design matrix as $\bG_{\bz,t}=\lambda \I_{d_{\bz}}+\sum_{s=1}^{t} \bz_{s} \bz_{s}^{\top}$. 
For all $\delta>0$ and regularisation parameter $\lambda>0$,
we observe that
\begin{align*}
     \normiii*{\bG_{\bz,t}^{-1}}_2  
=   \llmax*{ \bG_{\bz,t}^{-1}} 
\leq 
    \begin{cases}
        \frac{1}{\lambda} \text{ if } t\leq T_0\\
        \frac{2}{t 
    \llmin{\mathbf{\boldsymbol{\Sigma}}}} \text{ if } t > T_0
    \end{cases}.
\end{align*}
Here, $T_0>0$ is a constant defined by Equation~\eqref{eq:defn_c3} and $\llmin{\mathbf{\boldsymbol{\Sigma}}}$ is the minimum eigenvalue of the true covariance matrix of $\bz$, i.e. $\boldsymbol{\Sigma} \triangleq \E[\bz \bz^\top]$.
\end{lemma}

\begin{proof}
First, we aim to find a lower bound for the smallest eigenvalue of the design, matrix where we set the regularisation parameter $\lambda$ to zero. We denote the `non-regularised' design matrix as $\bG_{\bz,t}^{\lambda=0}$. For $t\geq 1$, we observe that $\bG_{\bz,t}^{\lambda=0}/t \triangleq \wbs_t$.
Thus, by applying Equation~\eqref{ineq_eigen_norm}, we obtain
\begin{align*}
    \left|\llmin*{\nicefrac{ \bG_{\bz,t}^{\lambda=0}}{t}}-\llmin{\mathbf{\boldsymbol{\Sigma}}}\right| 
\leq
    \frac{4L_{\bz}^2}{t} \log \left( \frac{2d_{\bz}}{\delta}\right)
    + 
    2 \sqrt{
    \frac{2L_{\bz}^2}{t} \log \left( \frac{2d_{\bz}}{\delta}\right) \normiii*{\boldsymbol{\Sigma}}_2
    } \, .
\end{align*}
Further substituting $A \triangleq 2L_{\bz}^2 \log \left( \frac{2d_{\bz}}{\delta}\right)$ leads to the following lower bound for the minimum eigenvalue
\begin{align*}
    \llmin*{ \bG_{\bz,t}^{\lambda=0}}
\geq
     \max \left\lbrace 0, t 
     \left(
    \llmin{\mathbf{\boldsymbol{\Sigma}}} 
    -
    \nicefrac{2A}{t}
    - 
    2 \sqrt{
    \nicefrac{A\llmax*{\bsig} }{t}
    }
    \right)\right\rbrace.
\end{align*}
Here, $\llmax{\mathbf{\boldsymbol{\Sigma}}}$ and $\llmin{\mathbf{\boldsymbol{\Sigma}}}$ is the maximum and minimum eigenvalues of the true covariance matrix of $\bz$, i.e. $\boldsymbol{\Sigma} \triangleq \E[\bz \bz^\top]$. 
% By well-behavedness assumption of the IV, both of them are positive and bounded reals. \todo{do we need well behaved assumption? just comment this part probably. Yes, we should comment}

Now, from the variational definition of the minimum eigenvalues, we have  $\llmin*{ \bG_{\bz,t}} \geq \llmin*{ \bG_{\bz,t}^{\lambda=0}} + \lambda,$ which implies that $\llmin*{ \bG_{\bz,t}}\geq \lambda$ for all $t \geq 0$, with equality for $t=0$. Thus, we have
\begin{align}\label{eq:lminll}
    \llmin*{ \bG_{\bz,t}}
&\geq
    \max \left\lbrace \lambda, 
    \lambda +
    t \left(
    \llmin{\mathbf{\boldsymbol{\Sigma}}} 
    -
    \nicefrac{2A}{t}
    - 
    2 \sqrt{
    \nicefrac{A  \llmax*{\bsig}}{t}
    }
    \right)\right\rbrace.
\end{align}

Let us consider the second term inside the maximum of Equation~\eqref{eq:lminll}, and we split it in the following way
\begin{align*}%\label{eq:Cdelta_term}
    \lambda +
    t
    \llmin{\mathbf{\boldsymbol{\Sigma}}} 
    -
    2A
    - 
    2 \sqrt{
    tA  \llmax*{\bsig}
    }
&=
    t
    \underbrace{\frac{\llmin{\mathbf{\boldsymbol{\Sigma}}}}{2}
    }_{\text{Term (A)}}
    +
    \underbrace{\left(
    \frac {t \llmin{\mathbf{\boldsymbol{\Sigma}}} }{2}
    - 
    2 \sqrt{t
    A  \llmax*{\bsig}
    }
    +\lambda -
    2A
    \right)}_{\text{Term (B)}}
\end{align*}

Now we study for which values Term (B) is non-negative. The corresponding second-order polynomial equation is obtained substituting $u = \sqrt{t}$, and is
$    u^2\llmin{\mathbf{\boldsymbol{\Sigma}}}
    - 
    4u \sqrt{
    A \llmax*{ \boldsymbol{\Sigma}}
    }
    +
    2(\lambda-2A)
=
    0\;,$
which has two solutions given by 
% \begin{align*}
   $ u_{\pm} = 
    \nicefrac{
    2\sqrt{
    A \llmax*{ \boldsymbol{\Sigma}}}
    \pm 
    \sqrt{
    4A \llmax*{ \boldsymbol{\Sigma}}
    +2(2A-\lambda)\llmin{\mathbf{\boldsymbol{\Sigma}}}}
    }
    {\llmin{\mathbf{\boldsymbol{\Sigma}}}}\;.$
% \end{align*}
In particular 
for $t> \lceil u_+\rceil $, $\text{Term (B)} \geq 0$, and \Cref{eq:lminll} simplifies in
\begin{align*}
    \llmin*{ \bG_{\bz,t}}
&\geq
     \max \left\lbrace \lambda, 
    % \lambda +
    t 
    \llmin{\mathbf{\boldsymbol{\Sigma}}} 
    /2 \right\rbrace.
\end{align*}
Therefore, for $t>\lceil 2\lambda/\llmin{\boldsymbol{\Sigma}}\rceil$ and $t> \lceil u_+\rceil $ we have that 
$    \llmin*{ \bG_{\bz,t}}
\geq
    t 
    \llmin{\mathbf{\boldsymbol{\Sigma}}} 
    /2 $.

Putting the results together, we conclude that 
\begin{equation}\label{eq:defn_c3}
     \llmin*{ \bG_{\bz,t}}
\geq
    t 
    \llmin{\mathbf{\boldsymbol{\Sigma}}} 
    /2 
\quad \text{for}\quad 
    t>T_0 \triangleq \max\left \{ \lceil 2\lambda/\llmin{\boldsymbol{\Sigma}}\rceil,  \lceil u_+\rceil  \right\},
\end{equation}
while for $t\leq T_0$, we retain the trivial lower bound of the minimum eigenvalue, i.e. $\lambda$.

In summary, we have
\begin{align*}
    \llmin*{ \bG_{\bz,t}} 
\geq 
    \begin{cases}
        \lambda \text{ if } t\leq T_0\\
        t 
    \llmin{\mathbf{\boldsymbol{\Sigma}}} 
    /2 \text{ if } t > T_0
    \end{cases}
\iff
    \llmax*{ \bG_{\bz,t}^{-1}} 
\leq 
    \begin{cases}
        \frac{1}{\lambda} \text{ if } t\leq T_0\\
        \frac{2}{t 
    \llmin{\mathbf{\boldsymbol{\Sigma}}}} \text{ if } t > T_0
    \end{cases}.
\end{align*}
\end{proof}

\begin{corollary}[Bound on the sum of maximum eigenvalues of the inverses of the first-stage design matrices]\label{cor:ft}
For every $T>0$, we have
\begin{align*}
    \sum_{s=0}^T\normiii*{\bG_{\bz,s}^{-1}}_2
\leq
    f(T) 
= \mathcal{O}(\log(T))\: ,
\end{align*} 
where $f(t)\triangleq\frac{T_0+1}{\lambda} + \frac{2 (\log(T)+1)}{\llmin*{\bsig}}$.
\end{corollary}
\begin{proof}
For $T\leq T_0$,
    \begin{equation}
    \sum_{s=0}^T\normiii*{\bG_{\bz,s}^{-1}}_2
=   
    \sum_{s=0}^{T}
    \frac{1}{\lambda}
\leq
    \frac{T_0+1}{\lambda}
=
    \mathcal O (1).
\end{equation}
If $T>T_0$
\begin{equation}
    \sum_{s=0}^T\normiii*{\bG_{\bz,s}^{-1}}_2
=   
    \sum_{s=0}^{T_0}
    \frac{1}{\lambda}
    +
    \sum_{s=T_0+1}^{T} \frac{2}{\llmin*{\bsig}s}
\leq 
    \frac{T_0+1}{\lambda}
    +
    \frac{2 (\log(T)+1)}{\llmin*{\bsig}} 
=
    \mathcal O (\log(T))
\end{equation}

where the last inequality follows from $\sum_{k=1}^n \frac{1}{k}\leq\int_1^n \frac{d x}{x}+1=\log(n)+1 $.
\end{proof}

\newpage
\subsection{Upper bounding the feature norms}

In this section, we present some useful Lemmas that we used in the proofs of the regret bounds of \otsls{} and \ofuliv{}. 

%\todo[inline]{change reference}
\begin{remark}
In \Cref{app:fs}, we describe that the first stage regression in \otsls{} can be expressed as running $d$ independent ridge regressions for each column of $\bt$ (Equation~\eqref{eq:first-stage-ridge}). Since the standard analysis of each of the ridge regressions assume independent and sub-Gaussian noise added in the linear model (cf. \Cref{thm:first-stage-ridge};~\citep{ouhamma2021stochastic}), we assume that each component of the first stage noise, i.e. $\bep_{t,i}$, corresponding to the $i$-th ridge regression is $\operatorname{sub-Gauss(\sigma_{\bep})}$. Thus, we obtain that $\mathbb E \norm{\bep_{t}}^2_2 \leq d_{\bx}\sigma_{\bep}^2$. We use this result throughout this section.
\end{remark}

\begin{lemma}[Bounding the First-stage Estimates]\label{lem:boundingfirststage}
% Given the relevance condition in \Cref{assumption:2sls} and a regularization parameter $\lambda>0$,
% in the form of a lower bound   for the minimum eigenvalue of the product of the empirical cross-covariance matrix of the vectors $\bz$ and $\bx$
Under Assumptions~\ref{assumption:2sls} and IVs being bounded $\|\bz\|_2^2\leq L_{\bz}^2$, the minimum singular value of the estimated parameter in the first-stage regression  (\Cref{eq:first-stage-ridge})  is lower bounded by:
\begin{align}\label{eq:minm_singular_firststage}
     \sigma_{\min}(\widehat \bt_t)
\geq
    \frac{\mathfrak{r}^2}{L_{\bz}^2},
\end{align}
%  \todoR{Add the following but write it in such a way to introduce r recalling the relevance condition}
 \red{where we have from the relevance condition of IVs, $ \sigma_{\min}\left( \mathrm{Cov}(\bZ_t, \bX_t) \right) = \sigma_{\min}\left(\frac{1}{t} \sum_{s=1}^t \bz_s\bx_s^{\top}\right)\geq {\mathfrak{r}} > 0$ (Eq.~\eqref{eq:relevance})}.
\end{lemma}
\begin{proof}
We start with the variational definition of the minimum singular value and then we lower bound this quantity:
\begin{align*}
    \sigma_{\min}(\widehat \bt_t)
&=  
    \sigma_{\min}\left(\left(\bZ_{t}^{\top} \bZ_{t}\right)^{-1}  \left( \bZ_{t}^{\top} \bX_{t}\right)\right)\\
&=   
    \min_{\bv\in \mathbb S^{d_{\bx}-1}} 
    \norm{
    \left(\bZ_{t}^{\top} \bZ_{t}\right)^{-1}  \left( \bZ_{t}^{\top} \bX_{t}\right) \bv}_2\\
&\underset{(a)}{=}   
    \min_{\bv\in \mathbb S^{d_{\bx}-1}}
    \frac{\norm{
    \left(\bZ_{t}^{\top} \bZ_{t}\right)^{-1}  \left( \bZ_{t}^{\top} \bX_{t}\right) \bv}_2}{\norm{\left( \bZ_{t}^{\top} \bX_{t}\right) \bv}_2}
    \norm{\left( \bZ_{t}^{\top} \bX_{t}\right) \bv}_2
    \\
&\geq
    \sigma_{\min}\left(\left(\bZ_{t}^{\top} \bZ_{t}\right)^{-1} \right) 
    \sigma_{\min}\left( \bZ_{t}^{\top} \bX_{t}\right)
    \\
&=
    \lambda_{\min}\left(\left(\bZ_{t}^{\top} \bZ_{t}\right)^{-1} \right) 
    \sigma_{\min}\left( \bZ_{t}^{\top} \bX_{t}\right)\\
&=
    \frac{\sigma_{\min}\left( \bZ_{t}^{\top} \bX_{t}\right)}{\lambda_{\max}\left(\bZ_{t}^{\top} \bZ_{t} \right) } \: .
% &=
%     \lambda^2_{\max}\left(\bZ_{t}^{\top} \bZ_{t} \right) 
%     \sigma_{\min}\left( \bZ_{t}^{\top} \bX_{t}\right)
\end{align*}
$(a)$ holds because according to \Cref{assumption:2sls} it is assumed that $\left( \bZ_{t}^{\top} \bX_{t}\right)$ has maximum rank and therefore we are not dividing by a zero norm vector.
Then, by the variational definition of the biggest eigenvalue
\begin{align*}
    \lambda_{\max} \left(\bZ_{t}^{\top} \bZ_{t}\right) 
&= 
    \max _{\bv \in \mathbb{S}^{d_{\bz}-1}} \left\langle \bv , \sum_{s=1}^t\bz_{s} \bz_s^\top \bv  \right\rangle
= 
    \max _{\bv \in \mathbb{S}^{d_{\bz}-1}} \sum_{s=1}^t \left\langle \bv , \bz_{s}  \right\rangle^2
\leq 
    \sum_{s=1}^t \norm{\bz_{s}}_2^2  
\leq 
    t L_{\bz}^2
\end{align*}
Finally, we note that the quantity  $
    \sigma_{\min}\left(\bZ_{t}^{\top} \bX_{t}\right) 
\geq
    t \mathfrak{r}
$ 
by the definition of relevance, which implies 
\begin{align*}
     \sigma_{\min}(\widehat \bt_t)
&=
    \frac{\sigma_{\min}\left( \bZ_{t}^{\top} \bX_{t}\right)}{\lambda^2_{\max}\left(\bZ_{t}^{\top} \bZ_{t} \right) }
\geq
    \frac{t \mathfrak{r}^2}{t L_{\bz}^2}
\geq 
    \frac{\mathfrak{r}^2}{L_{\bz}^2} \: .
\end{align*}
\end{proof}

%\

\red{

\begin{lemma}\label{lem:operator_norms_invers_H}
For any $t\geq 1$, we get
    \begin{align}
    \normiii*{\widehat{\bH}_{t-1}^{-1}}_2
=
    \llmax*{\widehat \bH_{t-1}^{-1}}
=   
    \frac{1}{\llmin*{\widehat \bH_{t-1}}}
&\leq  
    \frac{1}{\llmin*{\bG_{\bz,t-1}}   \: \sigma_{\min}{^2}\left({\widehat{\bt}_{t-1}}\right)}.\label{eq:lower_eig_seond_stage}
\end{align}
\end{lemma}
\begin{proof}
The equalities follow directly from \Cref{cor:eig_and_sing}, by the definition of the operator norm of a symmetric matrix, and using the fact that the maximum eigenvalue of an inverse matrix equals the inverse of the minimum eigenvalue of the matrix.
For the inequality, we start by lower bounding the minimum eigenvalue of the second stage design matrix $\widehat \bH_{t-1}$ and applying the definitions:
\begin{align*}
\llmin*{\widehat \bH_{t-1}}
=    
\llmin*{\widehat{\bt}_{t-1}^\top \bG_{\bz,t-1} \widehat{\bt}_{t-1}}
&=
    \min_{\bv\in\mathbb S^{d_{\bx}-1}}
    \left\langle 
    \bv,
    \widehat{\bt}_{t-1}^\top \bG_{\bz,t-1} \widehat{\bt}_{t-1}
    \bv
    \right\rangle \notag\\
&=  
    \min_{\bv\in\mathbb S^{d_{\bx}-1}}
    \frac{
    \left\langle 
    \bv,
    \widehat{\bt}_{t-1}^\top \bG_{\bz,t-1} \widehat{\bt}_{t-1}
    \bv
    \right\rangle}{\left\langle
    {\widehat{\bt}_{t-1}}
    \bv,
    \widehat{\bt}_{t-1}\bv
    \right\rangle} 
    \left\langle 
    {\widehat{\bt}_{t-1}}
    \bv,
    \widehat{\bt}_{t-1}
    \bv
    \right\rangle \notag\\
&\geq  
     \llmin*{\bG_{\bz,t-1}} 
     \min_{\bv\in\mathbb S^{d_{\bx}-1}}
    \norm{
    \widehat{\bt}_{t-1}
    \bv
    }^2 \notag\\
&\geq  
    \llmin*{\bG_{\bz,t-1}}   \: \sigma_{\min}{^2}\left({\widehat{\bt}_{t-1}}\right)
\end{align*}
and we conclude taking the inverses.
\end{proof}

}

\begin{lemma}[Bounding the impact of first-stage noise]\label{lem:bound_noises}
For first stage noises that is component-wise sub-Gaussian($\sigma_{\bep}$), and bounded IVs, i.e. $\norm{\bz}_2^2 \leq L_{\bz}^2$,
we have that 
\begin{align*}
    \sum_{t=1}^{T}\left\|\bep_{t}\right\|_{\widehat{\bH}_{t-1}^{-1}}^{2}
&\leq
    \underbrace{\frac{\mathfrak{r}\red{^4}}{L_{\bz}\red{^4}} \left( d_{\bx} \sigma^{2}_{\bep} 
    f(T) + C_5 \right)}_{\bigO(d_{\bx} \log T)}
\end{align*}
with probability at least $1-\delta$. Here, $C_5$ is a constant of $\bigO(d_{\bx})$ as defined in \Cref{eq:c4}.
\end{lemma}
\begin{proof}

The proof follows using chain of inequalities
\begin{align*}
    \sum_{t=1}^{T}\left\| \bep_t\right\|_{\widehat \bH_{t-1}^{-1}}^{2} 
    &\leq \sum_{t=1}^{T}\left\| \bep_t\right\|_2^2 \lambda_{\max}\left({\widehat \bH_{t-1}^{-1}}\right)
    \red{= \sum_{t=1}^{T}\frac{\left\| \bep_t\right\|_2^2 }{\lambda_{\min}\left({\widehat \bH_{t-1}}\right)}}\\
    &\leq \sum_{t=1}^{T} 
    \frac{\norm{\bep_t}^2_2}{\llmin*{\bG_{\bz,t-1}}  \: 
    \sigma_{\min}{^2}\left({\widehat{\bt}_{t-1}}\right)}
    \leq 
    \frac{\mathfrak{r}{^4}}{L_{\bz}{^4}} 
    \sum_{t=1}^{T} 
    \frac{\norm{\bep_t}^2_2}{\llmin*{\bG_{\bz,t-1}}}\\
    &=
     \frac{\mathfrak{r}{^4}}{L_{\bz}{^4}} 
     \left(\underbrace{
     \sum_{t=1}^{T}
     \normiii*{\bG_{\bz,t-1}^{-1}}_2  
     \left(
     \left\| \bep_t\right\|_2^2
     - \mathbb E \left\| \bep_t\right\|_2^2
     \right)
     }_{\textbf{Term I}}
     + 
     \underbrace{
     \sum_{t=1}^{T}
     \normiii*{\bG_{\bz,t-1}^{-1}}_2 
     \mathbb E
     \left\| \bep_t\right\|_2^2
     }_{\textbf{Term II}}\right)
\end{align*}
\red{where in the second inequality we used \Cref{lem:operator_norms_invers_H}, \Cref{lem:boundingfirststage} in the third, and we simply split in two terms in the final passage the summation.}

\paragraph{Term I:} The squared $\ell_2$-norms of sub-Gaussian random variables are sub-exponential \citep{wainwright2019high}, therefore $\left\|\bep_t\right\|_2^2-\mathbb{E}\left\|\bep_t\right\|_2^2 \sim \operatorname{sub-exp} (\nu, \alpha)$ where the correct values of $\nu,\alpha$ are derived in \Cref{lem:squareproductnnind}, in the result on the square of sub-Gaussians variables, and gives $\nu\triangleq 4\sqrt{2} d_{\bx} \sigma^2_{\bep}$ and $\alpha \triangleq 4\sigma^2_{\bep}$.
Furthermore, given an $X \sim \operatorname{sub-exp}(\nu, \alpha) $ and a constant $c$, then $ \nicefrac{X}{c} \sim \operatorname{sub-exp}  \left(\nicefrac{v}{c}, \nicefrac{c}{\alpha}\right)
$ which follows by considering a new parameter $\lambda/c$ 
in the definitions of sub-exponentiality
   $ \mathbb{E}\left[e^{\lambda X/c}\right] \leq e^{\lambda^2 \nu^2 / 2 c^2}$,
where
$
    |\lambda / c| \leq \frac{1}{\alpha} 
$
iff
$|\lambda| \leq \frac{c}{\alpha}$
Using this, we rescale by the factor $t$ 
\begin{align*}
    \mathbb{P}\left[\sum_{t=T_0}^T \frac{\left\|\bep_t\right\|_2^2-\mathbb{E}\left\|\bep_t\right\|_2^2}{t} \geq \mu\right] 
&\underset{(a)}{\leq}               
    \mathbb{E}\left[e^{\lambda \sum_{t=T_0}^T \frac{\left\| \bep_t\right\|^2-\mathbb{E}\left\| \epsilon_t \right\|^2}{t}}\right]e^{-\lambda \mu} \\
&\underset{(b)}{\leq}
    e^{\sum_{t=T_0}^T \lambda^2 \nu^2 / 2 t^2-\lambda \mu}\\
&\underset{(c)}{\leq} 
    e^{\frac{\lambda^2 \nu^2}{2}\left(\frac{1}{T_0-1}-\frac{1}{T}\right)-\lambda \mu}
\end{align*} 
where (a) is Markov's inequality, (b) uses the product rule for expectations and  
$\left\|\bep_t\right\|_2^2-\mathbb{E}\left\|\bep_t\right\|_2^2 \sim \operatorname{sub-exp} (\nu=4\sqrt{2} d_{\bx} \sigma^2_{\bep}, \alpha=4\sigma^2_{\bep})$, and (c) holds $\forall|\lambda| \leq \frac{m}{\alpha}$ thanks to the following series of inequalities 
$
\sum_{T_0}^T \frac{1}{t^2} \leq \int_{t_0-1}^T 1 / t^2 d t=(-\frac{1}{t}|_{T_0-1}^T=-\frac{1}{T}+\frac{1}{T_0-1}.
$
This proves that 
\begin{equation}
    \sum_{t=T_0}^T \frac{\left\|\eta_t\right\|_2^2-\mathbb{E}\left\|\eta_t\right\|_2^2}{t} 
\sim 
    \operatorname{sub-exp} \Big({\nu \sqrt{\frac{1}{T_0-1}-\frac{1}{T}}}, {\frac{T_0}{\alpha}}\Big)
\end{equation}
We bound the following summation using \Cref{prop:concentration_eigenvalue}:
\begin{multline*}
\sum_{t=1}^{T}
     \normiii*{\bG_{\bz,t-1}^{-1}}_2  \left(
     \left\| \bep_t\right\|_2^2
     - \mathbb E \left\| \bep_t\right\|_2^2
     \right)\\
\begin{aligned}
&=
    % |\eta_1| \norm{\bep_1}_2 \sqrt{\llmax*{  \bG_{\bz,0}^{-1} } } 
    \sum_{t=0}^{T_0}  \llmax*{  \bG_{\bz,t}^{-1}   }
    \left(
     \left\| \bep_{t+1}\right\|_2^2
     - \mathbb E \left\| \bep_{t+1}\right\|_2^2
     \right)
+
    \sum_{t=T_0+1}^{T-1}  \llmax*{  \bG_{\bz,t}^{-1}   } 
    \left(
     \left\| \bep_{t+1}\right\|_2^2
     - \mathbb E \left\| \bep_{t+1}\right\|_2^2
     \right)
    \\
&\leq 
        \frac{1}{\lambda}
        \sum_{t=0}^{T_0} 
    \left(
     \left\| \bep_{t+1}\right\|_2^2
     - \mathbb E \left\| \bep_{t+1}\right\|_2^2
     \right)
      +
      \frac{2}{\llmin*{\bsig}} \sum_{t=T_0+1}^{T-1} \frac{\left\| \bep_{t+1}\right\|_2^2
     - \mathbb E \left\| \bep_{t+1}\right\|_2^2}{t}\\
&\leq
        \frac{T_0+1}{\lambda} 
        \left(
        4\sqrt{2} \sigma^2_{\bep} d_{\bx}
        \sqrt{2 \log (1 / \delta)}+\frac{1}{2 \sigma^2_{\bep}} \log (1 / \delta)
        \right)
        \\
&\qquad\qquad\qquad 
    + 
    \frac{2}{\llmin*{\bsig}}
    \left(
    \sqrt{2 \nu^2\left(\frac{1}{T_0}-\frac{1}{T}\right) \log (1 / \delta)}+\frac{2 T_0}{\alpha} \log (1 / \delta)
    \right)
\end{aligned}
\end{multline*}
Since Term I can be upper-bounded by a constant $\bigO(d_{\bx})$, we just name this constant $C_5$ where we also substitute back the definitions $\nu = 4\sqrt{2} d_{\bx} \sigma^2_{\bep}$ and $\alpha = 4\sigma^2_{\bep}$. Specifically,
\begin{align}
    \textbf{Term I} &\leq 
     \frac{T_0+1}{\lambda} 
        \left(
        4\sqrt{2} d_{\bx} \sigma^2_{\bep}
        \sqrt{2 \log (1 / \delta)}+\frac{1}{2 \sigma^2_{\bep}} \log (1 / \delta)
        \right)\notag\\
    &\quad\quad\quad\quad\quad+ 
    \frac{2}{\llmin*{\bsig}}
    \left(
    4\sqrt{2} d_{\bx} \sigma^2_{\bep}
    \sqrt{\frac{2}{T_0} \log (1 / \delta)}+\frac{2 T_0}{4\sigma^2_{\bep}} \log (1 / \delta)
    \right) \triangleq C_5,\label{eq:c4}
\end{align}
which is a $T$-independent constant of $\bigO(d_{\bx} \sigma_{\bep}^2)$.

\paragraph{Term II:}
The proof step-wise applies: (d) sub-Gaussianity of $\bep_t$, (e) the high-probability bound on the sum of the minimum eigenvalue of the first-stage design matrix in \Cref{cor:ft}
\begin{align*}
 \textbf{Term II} &\leq \sum_{t=1}^{T}
     \normiii*{\bG_{\bz,t-1}^{-1}}_2 
     \mathbb E
     \left\| \bep_t\right\|_2^2 \\
&\underset{(d)}{\leq} 
    d_{\bx}\sigma_{\bep}^2 
    \sum_{t=1}^{T} 
    \llmax*{\bG_{\bz,t-1}^{-1} }\\
&\underset{(e)}{\leq} 
    d_{\bx} \sigma^{2}_{\bep} 
    f(T) \: .
\end{align*}

Therefore, putting \textbf{Term I} and \textbf{Term II} together, we obtain
\begin{align*}
   \sum_{t=1}^{T}\left\| \bep_t\right\|_{\widehat \bH_{t-1}^{-1}}^{2} 
    &\leq \underbrace{ \frac{\mathfrak{r}{^4}}{L_{\bz}{^4}}  \left( d_{\bx} \sigma^{2}_{\bep} 
    f(T) + C_5 \right)}_{\bigO(d_{\bx} \log T)}
\end{align*}

\end{proof}

\begin{lemma}[Bounding the sum of second-stage feature norms]\label{lem:bound_featsx}
Under the same conditions of \Cref{lem:bound_noises} 
plus first-stage parameters with bounded $\ell_2$-norm $\normiii*{\bt}_2
\leq L_{\bt}$ and bounded IVs $\|\bz\|^2 \leq L_z^2$
we have that 
\begin{align*}
\sum_{t=1}^{T}\left\|\bx_{t}\right\|_{\widehat{\bH}_{t-1}^{-1}}^{2}  
&\leq 
\underbrace{
\left(L_{\bt}^2 \red{\frac{\mathfrak{r}^4}{L_{\bz}^2}} + \red{\frac{\mathfrak{r}^4}{L_{\bz}^4}} d_{\bx} \sigma_{\bep}^2\right)
f(T)  
+ \red{\frac{\mathfrak{r}^4}{L_{\bz}^4}} C_5
}_{\mathcal{O}({d_{\bx} \log T})}
\end{align*}
\end{lemma}

\begin{proof}
    We start by substituting the first-stage equations inside the norm and using triangle inequality
\begin{align*} 
    \sum_{t=1}^{T}\left\|\bx_{t}\right\|_{\widehat{\bH}_{t-1}^{-1}}^{2}  
&=
    \sum_{t=1}^{T}\left\|\bt^\top\bz_t  + \bep_t
    \right\|_{\widehat{\bH}_{t-1}^{-1}}^{2}
    % \tag{Equation~\eqref{eq:first_stage}}
    \\
&\leq
     \underbrace{\sum_{t=1}^{T}\left\|{\bt}^\top\bz_t 
    \right\|_{\widehat{\bH}_{t-1}^{-1}}^{2}}_{\text{Term 1}}
    +
    \underbrace{\sum_{t=1}^{T}\left\| \bep_t\right\|_{\widehat{\bH}_{t-1}^{-1}}^{2}}_{\text{Term 2}}
\end{align*}

\textbf{Step 1: Bounding Term 1.} First, we upper bound the Term 1
\begin{align*}
\sum_{t=1}^{T}\left\|{\bt}^\top\bz_t 
    \right\|_{ \widehat \bH_{t-1}^{-1}}^{2}
&=
    \sum_{t=1}^T
    \left\langle \bt^\top \bz_t, \widehat \bH_{t-1}^{-1} \bt^\top \bz_t \right\rangle\\
&\leq
    \sum_{t=1}^{T} 
    \norm{\bt^\top \bz_t}_2^2
    \llmax*{\widehat \bH_{t-1}^{-1} }%\\
\end{align*}

\red{Therefore, by application of \Cref{lem:operator_norms_invers_H} in the first inequality, we have }
\begin{align}
    \sum_{t=1}^{T}\left\|{\bt}^\top \bz_t 
    \right\|_{ \widehat \bH_{t-1}^{-1}}^{2}
&\leq
    \sum_{t=1}^{T} 
    \frac{\norm{\bt^\top \bz_t}^2_2}{\llmin*{\bG_{\bz,t-1}}  
    \: \sigma_{\min}^{\red{2}}\left({\widehat{\bt}_{t-1}}\right)} \notag \\
&\underset{(a)}{\leq}
    \red{\frac{\mathfrak{r}^4}{L_{\bz}^4}}
    \sum_{t=1}^{T} 
    \frac{\norm{\bt^\top \bz_t}^2_2}{\llmin*{\bG_{\bz,t-1}}  } \notag\\
&\underset{(b)}{\leq}
    L_{\bt}^2 L_{\bz}^2
    \frac{\mathfrak{r}^{\red{4}}}{L_{\bz}^{\red{4}}}
    \sum_{t=1}^{T} 
    \frac{1}{\llmin*{\bG_{\bz,t-1}}  }\notag\\
&=
    \red{
    \frac{L_{\bt}^2 
    \mathfrak{r}^4}{L_{\bz}^2}
    }
    \sum_{t=1}^{T} 
    \frac{1}{\llmin*{\bG_{\bz,t-1}}  } \label{eq:theta_z}\\
&\underset{(c)}{\leq}
    \red{
    \frac{L_{\bt}^2 
    \mathfrak{r}^4}{L_{\bz}^2}
    f(T)
    }\notag
\end{align}
Inequality (a) is due to Lemma~\ref{lem:boundingfirststage} that bounds the minimum singular value of the first-stage parameter estimates.
Inequality (b) is direct consequence of boundedness of IVs and the true first-stage parameter, \red{and (c) follows from \Cref{cor:ft}}.

\noindent\textbf{Step 2: Bounding Term 2.} This term is bounded directly in in Lemma~\ref{lem:bound_noises} and we conclude that

\red{\begin{align*}
    \sum_{t=1}^{T}\left\|\bep_{t}\right\|_{\widehat{\bH}_{t-1}^{-1}}^{2}
&\leq
    \frac{\mathfrak{r}\red{^4}}{L_{\bz}\red{^4}} \left( d_{\bx} \sigma^{2}_{\bep} 
    f(T) + C_5 \right)
\end{align*}
}

\noindent\textbf{Step 3: Assembling the results.} Now combining the upper bounds for Term 1 and Term 2, we get

\begin{align*}
\sum_{t=1}^{T}\left\|\bx_{t}\right\|_{\widehat \bH_{t-1}^{-1}}^{2}  
&\leq 
    \left(L_{\bt}^2 \red{\frac{\mathfrak{r}^4}{L_{\bz}^2}} + \red{\frac{\mathfrak{r}^4}{L_{\bz}^4}} d_{\bx} \sigma_{\bep}^2\right)f(T) 
    + \red{\frac{\mathfrak{r}^4}{L_{\bz}^4}} C_5
\end{align*}
\end{proof}

\newpage

\subsection{Upper bounding correlation between first and second stage noises}\label{sec:concentration_correlated}

We start deriving a concentration result for the quantity 
$ 
    S_t
\triangleq
    \sum_{s=1}^t \Delta \bb_{s-1}^{\top} \allowbreak \left(\bep_s \eta_s-\boldsymbol{\gamma}\right).
$
We can prove that $S_t$ is a martingale adapted to the filtration $
    \mathcal F_t
\triangleq
    \sigma\left(\bep_{s: t}, \eta_{1: t}, z_{1: t}\right) 
$,
by proving that $\mathbb{E}\left[\left|S_t\right|\right] \leq \infty$ and $\mathbb{E}\left[S_{t+1} \mid  \mathcal F_t\right]=S_t$. The first condition is immediate and the second can be easily verified since: 
    \begin{align*}
    \mathbb{E}\left[S_{t+1} \mid \mathcal F_t\right]
&=
    \mathbb{E}\left[\Delta \bb_t^{\top}\left(\bep_{t_1} \eta_{t+1}-\boldsymbol{\gamma}\right)+\sum_{s=1}^t \Delta \bb_{s-1}^{\top}\left(\bep_s \eta_s-\boldsymbol{\gamma}\right) \Bigg| \mathcal F_t\right] 
 =
    0+S_t.
\end{align*}
% Then, the idea is to apply \Cref{thm:wain_martingale_diff} for the first term in \Cref{eq:decomp_subg_3}. 

% In our case the martingale difference sequence is $\{(\Delta \bb_{s-1}^{\top}\left(\bep_s \eta_s-\boldsymbol{\gamma}\right), \mathcal F_s)\}_{s=1}^\infty$.

% To apply the previous theorem, we derive in the following the sub-exponentiality parameters $\nu_s^*, \alpha_s^*$ 
% for the martingale difference
% $D_s\triangleq\Delta \bb_{s-1}^{\top}\left(\bep_s \eta_s-\boldsymbol{\gamma}\right)$ such that
% $\mathbb{E}\left[e^{\lambda \Delta \bb_{s-1}^{\top}\left(\bep_s \eta_s-\boldsymbol{\gamma}\right)} \Big| \mathcal F_{s-1}\right] \leq e^{\lambda^2 \nu_s^{*_2} / 2} 
% $ a.s. $ \forall|\lambda|<1 / \alpha_s^*
% $
% and then we apply the previous theorem. The bound is derived in the following lemma.

\begin{lemma}[Concentration of correlated first and second-stage noise]\label{lem:selffulfilling}
Under Assumption~\ref{assumption:2sls}, and for sub-Gaussian first- and second-stage noises with parameters $\sigma_\eta$ and $\sigma_{\bep}$, for given $T > 1$
\begin{align*}
    \left|\sum_{t=1}^T \Delta \bb_{t-1}^{\top}\left(\bep_t \eta_t-\boldsymbol{\gamma}\right)\right|
& \leq
    8e^2\left(\sigma_\eta^2+\sigma_{\bep}^2\right) 
    \sqrt{d_{\bx}} 
    \red{\frac{\mathfrak{r}^2}{L_{\bz}^2}}
    \sqrt{\mathfrak b_{T-1}(\delta)}
    \left(
    \sqrt{2 f(T) \log\frac{2}{\delta}}
    \right. \\
&\qquad\qquad\qquad\qquad\qquad\qquad\qquad\qquad +
    \left.
    \sqrt{
    \max\left\{ 
        \frac{1}{\lambda},
        \frac{2}{ \llmin{\mathbf{\boldsymbol{\Sigma}}}}
    \right\}
    }
\log \frac{2}{\delta}
    \right)
\end{align*}
with probability at least $1-\delta \in [0,1)$.
\end{lemma}

\begin{proof}

The proof proceeds in three steps and the main technical difficulty arises from the dependencies between the random variables, which we tackle using \Cref{lem:squareproductnnind} plus some techniques derived from the equivalent characterisations of sub-Gaussian and sub-exponential random variables \citep{vershynin2018high}.

\textbf{Step 1: Finding the sub-exponential parameters.}  We start deriving the sub-exponential parameters for $\Delta \bb_{s-1}^T (\bep_s \eta_s-\boldsymbol{\gamma})$. 

First, we bound the p-th moment of the random variable
$\phi_i \triangleq\left(\epsilon_{s, i} \eta_s-\gamma_i\right)$. This is the centred product of two sub-Gaussians non-independent random variables $\eta_s$ and $\epsilon_{s,i}$ and from \Cref{lem:squareproductnnind} is $\phi_i\sim\operatorname{sub-exp} \left(\nu_i, \alpha_i\right)$ with $\nu_i\triangleq 4\sqrt{2} \left(\sigma_{\bep}^2+\sigma_{\eta}^2\right), \alpha_i\triangleq  2 \left(\sigma_{\bep}^2+\sigma_{\eta}^2\right)$. 
% From the definitions we have
% $
% \mathbb{E}\left[e^{\lambda \phi_i}\right] \leq e^{\frac{\lambda^2 \nu_i^2}{2}} $ for all $|\lambda| \leq \frac{1}{\alpha_i} 
% $.
We can rewrite the previous sub-exponentiality condition using a unique parameter $
    K_1 \triangleq \max \left\{\alpha_i, \nicefrac{\nu_i}{\sqrt{2}}\right\}
 =
    4\left(\sigma_{\bep}^2+\sigma_{\eta}^2\right)$
which implies
$    \mathbb E\left[e^{\lambda \phi_i}\right] \leq e^{\lambda^2 K_1^2}$ for all $|\lambda| \leq \nicefrac{1}{K_1}$. Using the general inequality $|x|^p \leq p^p\left(e^x+e^{-x}\right)$ valid for $x\in \mathbb R$ we obtain 
$$
\mathbb{E}\left[\left|\frac{\phi_i}{K_1}\right|^p\right] \leq \mathbb{E}\left[p^p\left(\exp\left({\frac{\phi_i}{K_1}}\right)+\exp\left({-\frac{\phi_i}{K_1}}\right)\right)\right] {\leq} p^p 2 e^{K_1^2 / K_1^2} \leq 2 e p^p
$$
The previous inequality directly implies an inequality on the p-th norm of the $\phi_i$
\begin{equation}
    \sqrt[p]{\mathbb{E}\left[\left|\epsilon_{s, i} \eta_s-\gamma_i\right|^p\right]} 
\leq 
    \sqrt[p]{2 e} K_1 p \leq 2 e K_1 p  
=
    4 e\left(\sigma_{\bep}^2+\sigma_{\eta}^2\right) p
\triangleq
    K_2 p
\end{equation}\label{eq:pnorm}
Now, we find the sub-exponential parameters for the scalar product $\Delta \bb_{s-1}^T (\bep_s \eta_s-\boldsymbol{\gamma})$. Using the sub-exponential characterisation with $L^p$ norm that we just derived in \Cref{eq:pnorm} we have $\forall p \geq 1 $
\begin{align}\label{eq:Ks}
    \left\|\Delta \bb_{s-1}^T\left(\bep_s \eta_s-\boldsymbol{\gamma}\right)\right\|_p
% &=
%     \left\|\sum_{i=1}^d \Delta \bb_{s-1, i}\left(\epsilon_{s, i} \eta_s-\boldsymbol{\gamma}\right)\right\|_p 
\leq 
    \sum_{i=1}^d\left|\Delta \bb_{s-1, i}\right|\left\|\epsilon_{s, i} \eta_s-\boldsymbol{\gamma}\right\|_p
\leq 
    \sum_{i=1}^d\left|\Delta \bb_{s-1, i}\right| K_2p
\triangleq
    K_{3,s} p \quad 
\end{align}
where $K_{3,s}\triangleq\sum_{i=1}^d\left|\Delta \bb_{s-1, i}\right| K_2 =  4 e\left(\sigma_{\bep}^2+\sigma_{\eta}^2\right) \sum_{i=1}^d\left|\Delta \bb_{s-1, i}\right|$.

We are ready to bound the moment-generating function and derive the 
\begin{align*}
    \mathbb{E}\left[e^{\lambda \sum_i \Delta \bb_{s-1, i}\left(\epsilon_{s, i} \eta_s-\gamma_i\right)} \mid \mathcal F_{s-1}\right]
&\underset{(a)}{=} 
    \mathbb{E}\left[\sum_{p=0}^{\infty} \frac{\lambda^p\left(\Delta \bb^{\top}(\bep_s \eta_s-\boldsymbol{\gamma})\right)^p}{p !}\right]
\underset{(b)}{=}
    \sum_{p=0}^{\infty} \frac{\lambda^p}{p !} \mathbb{E}\left[\left(\Delta \bb^{\top}(\bep_s \eta_s-\boldsymbol{\gamma})\right)^p\right]\\
&\underset{(c)}{=}
    1+\sum_{p=2}^{\infty} \frac{\lambda^p \mathbb{E}\left[\left(\Delta \bb^{\top}(\bep_s \eta_s-\boldsymbol{\gamma})\right)^p\right]}{p !}
\underset{(d)}{\leq}
    1+\sum_{p=2}^{\infty} \frac{\lambda^p K_{3,s}^p p^p}{p!} \\
&\underset{(e)}{\leq}
    1+\sum_{p=2}^{\infty} \lambda^p {K_{3,s}}^p e^p 
\underset{(f)}{=}
    1+\frac{\left(\lambda K_{3,s} e\right)^2}{1-\lambda K_{3,s} e} \\
&\underset{(g)}{\leq} 
    1+2\left(\lambda K_{3,s} e\right)^2 
\underset{(h)}{\leq} 
    e^{2 \lambda^2 K_{3,s}^2 e^2}
\end{align*}
where $(a)$ is a series expansion, $(b)$ is from the linearity of expectation, $(c)$ comes from the first moment equal to zero, $(d)$ is obtained by substituting \Cref{eq:Ks}, $(e)$ is Stirling's approximation $p!\geq\frac{p^p}{e^p}$, $(f)$ is valid $\forall| \lambda | \leq \frac{1}{K_{3,s} e}$, $(g)$ is valid for $| \lambda | \leq \frac{1}{2 K_{3,s} e}$, $(h)$ uses $1+x\leq e^{x}  \forall x\in\mathbb R$ .

Therefore we have that $\Delta \bb_{s-1}^T (\bep_s \eta_s-\boldsymbol{\gamma})$ is sub-exp($2K_{3,s}e$, $2K_{3,s}e$).
% which holds a.s. for $| \lambda | \leq \frac{1}{2 K_{3,s} e}$ implying that $ \alpha_s^*=2 K_{3,s} e, v^*=2 K_{3,s} e \Rightarrow \alpha_*=v^*=\tilde K_{3,s}$. 
We finally substitute for $K_{3,s}$ and we obtain
\begin{align*}
    \Delta \bb_{s-1}^T (\bep_s \eta_s-\boldsymbol{\gamma})
\sim 
    \text{sub-exp}\left(  
    % =8 e^2\left(\sigma_{\bep}^2+\sigma_{\eta}^2\right) \sum_{i=1}^d\left|\Delta\bb_{s-1, i}\right|
    8 e^2\left(\sigma_{\bep}^2+\sigma_{\eta}^2\right)\left\|\Delta \bb_{s-1}\right\|_1, \, 8 e^2\left(\sigma_{\bep}^2+\sigma_{\eta}^2\right)\left\|\Delta \bb_{s-1}\right\|_1 
\right)
\end{align*}

\textbf{Step 2: Concentration.} We are now ready to derive the concentration for the sum of the martingale using \Cref{thm:wain_martingale_diff} and we obtain 
$\left|\sum_{t=1}^T S_t\right| \leq \sqrt{2 \log (2 / \delta) \sum_t \nu_s^2}+2 \alpha^* \log (2 / \delta)$
 with probability bigger than $1-\delta$ where
\begin{align*}
    S_t &\triangleq \Delta \bb_{t-1}^{\top}\left(\bep_t \eta_t-\boldsymbol{\gamma}\right),\\
    \alpha_*&\triangleq 8 e^2\left(\sigma_{\bep}^2+\sigma_{\eta}^2\right)\max_s{\left\|\Delta \bb_{s-1}\right\|_1},\\
    \nu_s &\triangleq 8 e^2\left(\sigma_{\bep}^2+\sigma_{\eta}^2\right){\left\|\Delta \bb_{s-1}\right\|_1}
\end{align*}
we obtain that with probability bigger than $1-\delta$
\begin{align*}
    \left|\sum_{t=1}^T \Delta \bb_{t-1}^{\top}\left(\bep_t \eta_t-\boldsymbol{\gamma}\right)\right|
& \leq
    \sqrt{2 \log \left(\frac{2}{\delta}\right)\left(8e^2\right)^2\left(\sigma_{\eta}^2+\sigma_{\bep}^2\right)^2 \sum_{t=1}^T\left\|\Delta \bb_{t-1}\right\|_1^2} \\
&\qquad\qquad\qquad\qquad\qquad\quad+8e^2\left(\sigma_\eta^2+\sigma_{\bep}^2\right) \max _t\left\|\Delta \bb_{t-1}\right\|_1 \log (2 / \delta)
\end{align*}

\textbf{Step 3: Bounding the norms of the estimator. } We study the term $\left\|\Delta {\bb_{t-1}}\right\|_1$, its sum and maximum over $t$, which we need to substitute in the previous concentration bound. We can bound $\left\|\Delta {\bb_{t-1}}\right\|_1^2 \leq d_{\bx} \left\|\Delta {\bb_{t-1}}\right\|_2^2$ and then use some matrix tricks in the following for the individual terms
\begin{align*}
    \left\|\Delta \bb_{t-1}\right\|_2^2 
& =
    \left\|\widehat{\bH}^{-1 / 2} \widehat{\bH}^{1 / 2} \Delta \bb_{t-1}\right\|_2^2 
\underset{(a)}{\leq}
    \normiii*{\widehat{\bH}_{t-1}^{-1 / 2}}_2^2\left\|\Delta \bb_{t-1}\right\|_{\widehat{\bH}_{t-1}}^2\\
 &\underset{(b)}{\leq}
    \normiii*{\widehat{\bH}_{t-1}^{-1}}_2 \mathfrak b_{t-1}(\delta)
\underset{(c)}{\leq}
   \frac{\mathfrak{r}\red{^4}}{L_{\bz}\red{^4}}
   \normiii*{\bG_{\bz,t-1}^{-1}}_2 
   \mathfrak b_{T-1}(\delta)
\end{align*}
where $(a)$   follows from \Cref{prop:tricks_norms}, $(b)$ follows from the definition of the confidence ellipsoid in \Cref{eq:conf_set} \red{and $(c)$ from \Cref{lem:operator_norms_invers_H}}. Then, we compute the maximum and sum of these norms over $T$ rounds.

\textit{Sum:} When we take the sum over the rounds, we use in $(a)$ \Cref{prop:tricks_norms}, \red{\Cref{lem:operator_norms_invers_H} and \Cref{lem:boundingfirststage}}, and in $(b)$ \Cref{prop:concentration_eigenvalue}:
\begin{align*}
    \sum_{t=1}^T\left\|\Delta \bb_{t-1}\right\|_1^2
&\leq
    d_{\bx} \sum_{t=1}^T\left\|\Delta \bb_{t-1}\right\|_2^2
\underset{(a)}{\leq}   
   \frac{ d_{\bx}\mathfrak{r}\red{^4}}{L_{\bz}\red{^4}}
    \mathfrak b_{T-1}(\delta)
    \sum_{t=1}^T
    \normiii*{\bG_{\bz,t-1}^{-1}}_2 
\underset{(b)}{\leq}   
    \frac{ d_{\bx} \mathfrak{r}\red{^4}}{L_{\bz}\red{^4}}
    \mathfrak b_{T-1}(\delta)
    f(T)
\end{align*}

\textit{Maximum:} For the maximum, we have instead
\begin{align*}
    \max _t\left\|\Delta \bb_{t-1}\right\|_1
&\leq   
    \sqrt{d_{\bx}}
    \frac{\mathfrak{r}\red{^2}}{L_{\bz\red{^2}}}
    \sqrt{\mathfrak b_{T-1}(\delta)
    \normiii{\bG_{\bz,t-1}^{-1}}_2} 
\leq
    \sqrt{d_{\bx}}
    \frac{\mathfrak{r}\red{^2}}{L_{\bz\red{^2}}}
    \sqrt{\mathfrak b_{T-1}(\delta)
    \max\left\{ 
        \frac{1}{\lambda},
        \frac{2}{ \llmin{\mathbf{\boldsymbol{\Sigma}}}}
    \right\}}
\end{align*}
\textbf{Step 4: Assembling the results.} Finally, we can substitute back these expressions in the bound at the previous point. We have that with a probability bigger than $1-\delta$
\begin{align*}
   &\quad \left|\sum_{t=1}^T \Delta \bb_{t-1}^{\top}\left(\bep_t \eta_t-\boldsymbol{\gamma}\right)\right| \\
   & \leq
    \sqrt{2 \log (2 / \delta)\left(8e^2\right)^2\left(\sigma_{\eta}^2+\sigma_{\bep}^2\right)^2 {d_{\bx}}
    \frac{\mathfrak{r}\red{^4}}{L_{\bz}\red{^4}}
    \mathfrak b_{T-1}(\delta)
    f(T)
    } \\
&~~~~~~~~~
+\left(8e^2\right)\left(\sigma_\eta^2+\sigma_{\bep}^2\right) 
\sqrt{d_{\bx}}
    \frac{\mathfrak{r}\red{^2}}{L_{\bz\red{^2}}}
    \sqrt{\mathfrak b_{T-1}(\delta)
    \max\left\{ 
        \frac{1}{\lambda},
        \frac{2}{ \llmin{\mathbf{\boldsymbol{\Sigma}}}}
    \right\}}
\log (2 / \delta)\\
&= \left(8e^2\right)\left(\sigma_\eta^2+\sigma_{\bep}^2\right) 
\sqrt{d_{\bx}} 
\frac{\mathfrak{r}\red{^2}}{L_{\bz\red{^2}}}
\sqrt{\mathfrak b_{T-1}(\delta)}
\left(\sqrt{2 f(T) \log \frac{2}{\delta}} 
+ \sqrt{\max\left\{\frac{1}{\lambda}, \frac{2}{ \llmin{\mathbf{\boldsymbol{\Sigma}}}}\right\}}\log \frac{2}{\delta}\right)
\end{align*}

\end{proof}

\begin{lemma}[Bias of Correlated First and Second-stage Noise]\label{lem:biasOsqrtT}
    Under the same hypothesis as that of \Cref{thm:confidencebeta}, the bias term for any $T>1$ is bounded as
    \begin{align*}
      \sum_{t=1}^T \Delta \bb_{t-1}^{\top} \boldsymbol{\gamma}  
\leq 
    \|\boldsymbol{\gamma}\|_2 
    \frac{\mathfrak{r}\red{^2}}{L_{\bz}\red{^2}} 
    \sqrt{\mathfrak b_{T-1}(\delta)}
    \left(
    \frac{T_0}{\sqrt{\lambda}}
    +2
    \frac{\sqrt{2T}}{\sqrt{\llmin*{\bsig}}}
    \right)
    \end{align*}
\end{lemma}
\begin{proof}
The bias term has a much simpler analysis 

\begin{align*}
    \sum_{t=1}^T \Delta \bb_{t-1}^{\top} \boldsymbol{\gamma}
% &\leq
%     \sum_{t=1}^T \|\Delta \bb_{t-1}\|_2 \|\boldsymbol{\gamma}\|_2\\
&\underset{(a)}{\leq}   
    \sum_{t=1}^T \|\Delta \bb_{t-1}\|_{\widehat{\bH}_t} \|\boldsymbol{\gamma}\|_{\widehat{\bH}_t^{-1}}
\underset{(b)}{\leq}   
    \sqrt{\mathfrak b_{T-1}(\delta)}
    \sum_{t=1}^T \|\boldsymbol{\gamma}\|_{\widehat \bt_t^{-1} \bG_{\bz,t}^{-1}  \widehat \bt_t^{-\top}}
    \\
&\underset{(c)}{\leq}   
    \sqrt{\mathfrak b_{T-1}(\delta)}
     \|\boldsymbol{\gamma}\|_{2} 
     \sum_{t=1}^T 
    \sqrt{ 
    \frac{\mathfrak{r}\red{^4}}{L_{\bz}\red{^4}}
    \normiii*{\bG_{\bz,t-1}^{-1}}_2 
    }
    \\
&\underset{(d)}{\leq}   
      \|\boldsymbol{\gamma}\|_{2} 
      \frac{\mathfrak{r}\red{^2}}{L_{\bz}\red{^2}}
      \sqrt{\mathfrak b_{T-1}(\delta)}
    \left(
    \sum_{t=1}^{T_0}
    \frac{1}{\sqrt{\lambda}}
    +
    \sum_{t=T_0+1}^{T-1} 
    \sqrt{\frac{2}{\llmin*{\bsig}t}}
    \right)
    \\
&\leq 
   \|\boldsymbol{\gamma}\|_{2} 
   \frac{\mathfrak{r}\red{^2}}{L_{\bz}\red{^2}}
   \sqrt{\mathfrak b_{T-1}(\delta)}
    \left(
    \frac{T_0}{\sqrt{\lambda}}
    +2
    \frac{\sqrt{2T}}{\sqrt{\llmin*{\bsig}}}
    \right)
\end{align*}
where in $(a)$ we used Cauchy-Schwarz, in $(b)$ \Cref{thm:confidencebeta} and monontonicty of $\mathfrak b_{t}$, in $(c)$ \Cref{prop:tricks_norms}, in $(d)$ \Cref{prop:concentration_eigenvalue} and finally in
the last inequality follows from $\sum_{k=1}^n \frac{1}{\sqrt{k}}=\sum_{k=1}^n \int_{k-1}^k \frac{d x}{\sqrt{k}} \leq \sum_{k=1}^n \int_{k-1}^k \frac{d x}{\sqrt{x}}=\int_0^n \frac{d x}{\sqrt{x}}=2 \sqrt{n}$.

\end{proof}

\newpage
\subsection{Upper bounding the identification regret of \otsls{}}

In this section, we analyse the identification regret of \otsls, and derive a high-probability upper bound on it.

\begin{reptheorem}{thm:reg_otsls_identification}[Identification regret of \otsls{}] %\label{thm:identification_regret}
If Assumption~\ref{assumption:2sls} holds true,
first-stage noise is a component-wise sub-Gaussian($\sigma_{\bep}$) r.v., second-stage noise is also sub-Gaussian($\sigma_{\eta}$), the true first-stage parameter has bounded $\ell_2$-norm $\normiii*{\bt}_2
\leq L_{\bt}$, IVs are bounded $\|\bz\|^2 \leq L_z^2$, and a two-stage correlation level $\mathbb E[\eta_s \bep_s]\triangleq \boldsymbol{\gamma} \in \mathbb R^d$, the regret of \otsls{} at step $T$ is
\begin{align*}
    \widetilde{R}_{T} 
&\leq
   \underbrace{\mathfrak b_{T-1}(\delta)}_{\substack{\text{Estimation}\\ d_{\bz}\log T}} \underbrace{\left(
   \left(
    L_{\bt}^2 \red{\frac{\mathfrak{r}^4}{L_{\bz}^2}} + \red{\frac{\mathfrak{r}^4}{L_{\bz}^4}} d_{\bx} \sigma_{\bep}^2\right) f(T)  
    + \red{\frac{\mathfrak{r}^4}{L_{\bz}^4}} C_5
    \right)}_{\substack{\text{Second-stage feature norm}\\d_{\bx}\log T}}= \bigO(d_{\bx}d_{\bz}\log^2 T) \: .
\end{align*}
with probability at least $1-\delta \in [0,1)$. 
Here, $\mathfrak b_{T-1}(\delta)$ is the confidence interval defined by \Cref{thm:confidencebeta}, $f(T)$ is defined in \Cref{cor:ft}, $C_5$ in \Cref{lem:bound_noises}.
% , $\lambda>0$  is the regularisation parameter of the first stage, and $\llmin{\mathbf{\boldsymbol{\Sigma}_{\bz}}}$ is the minimum eigenvalue of the true covariance matrix of IVs, i.e. $\boldsymbol{\Sigma_{\bz}} \triangleq \E[\bz \bz^\top]$.
\end{reptheorem}

\begin{proof}[Proof of Theorem~\ref{thm:reg_otsls_identification}]
We bound it (a) using the confidence bound to control the concentration of $\bb_t$ around 
$\bb$, and (b) by bounding the sum of feature norms according to the following decomposition.

\textit{Step 1:} By applying Cauchy-Schwarz inequality, we first decouple the effect of parameter estimation and the feature norms
\begin{align}
	\left(\bb_{t-1}-\bb\right)^{\top} \bx_{t} 
&{\leq}
    \left\|\bb_{t-1}-\bb\right\|_{\widehat{\bH}_t} 
    \left\|\bx_{t}\right\|_{\widehat{\bH}_t^{-1}}  
\leq
	 \sqrt{\mathfrak b_{t-1}(\delta)}\left\|\bx_{t}\right\|_{\widehat{\bH}_t^{-1}} \label{eq:decouple_2nd}
\end{align}
The above inequality is due to \Cref{thm:confidencebeta}, and holds 
with probability at least $1-\delta$. 
Since $\mathfrak b_{t}$ is monotonically increasing in $t$, by Equation~\eqref{eq:decouple_2nd},
\begin{align*}
    \sum_{t=1}^T \left(\left(\bb_{t-1}-\bb\right)^{\top} \bx_{t} \right)^2
&\leq 
    \mathfrak b_{T-1}(\delta) \sum_{t=1}^{T}\left\|\bx_{t}\right\|_{\widehat{\bH}_t^{-1}}^{2}.
\end{align*}
% and we know that for any $T>1$, from \Cref{lem:appconfidencebeta}, the confidence interval at step $T-1$ is 
% \begin{align*}
%     \mathfrak b_{T-1}(\delta)
% = 
%      \frac{d \sigma_{\eta}^2}{4} {\log \left(\frac{1+(T-1) L^{2}_z / \lambda}{\delta}\right)}
% \end{align*}

\textit{Step 2:} Now, we need to bound the sum of the feature norms, for which we can directly use the result of \Cref{lem:bound_featsx}

\begin{align*} %\label{eq:feature_norm_bound}
    \sum_{t=1}^{T}\left\|\bx_{t}\right\|_{\widehat{\bH}_t^{-1}}^{2} 
\leq
    \left(L_{\bt}^2 \red{\frac{\mathfrak{r}^4}{L_{\bz}^2}} + \red{\frac{\mathfrak{r}^4}{L_{\bz}^4}} d_{\bx} \sigma_{\bep}^2\right)
f(T)  
+ \red{\frac{\mathfrak{r}^4}{L_{\bz}^4}} C_5
= 
    \bigO(d_{\bx}\log T)
\end{align*}

\textit{Step 3:} By combining the results of the previous steps and considering the definition of $\mathfrak b_{T-1}(\delta)$, we conclude that we can bound the Identification Regret as follows, and its orders is $\mathcal{O}\left(d_{\bx}d_{\bz} \log^2(T)\right)$
\begin{align*}
    \sum_{t=1}^T \left(\left(\bb_{t-1}-\bb\right)^{\top} \bx_{t} \right)^2
% &\leq 
%     \mathfrak b_{T-1}(\delta) \sum_{t=1}^{T}\left\|\bx_{t}\right\|_{\left({ \widehat{\bt}_{t-1}^\top \bG_{\bz,T-1}  \widehat{\bt}_{t-1}} \right)^{-1}}^{2}   \\
&\leq
    \underbrace{\mathfrak b_{T-1}(\delta)}_{d_{\bz}\log T} 
    \underbrace{\left(
    \left(L_{\bt}^2 \red{\frac{\mathfrak{r}^4}{L_{\bz}^2}} + \red{\frac{\mathfrak{r}^4}{L_{\bz}^4}} d_{\bx} \sigma_{\bep}^2\right)
f(T)  
+ \red{\frac{\mathfrak{r}^4}{L_{\bz}^4}} C_5
    \right)}_{d_{\bx}\log T}\\
&= \bigO(d_{\bz}d_{\bx}\log^2 T) \: .
\end{align*}
\end{proof}

\newpage

\subsection{Upper bounding the oracle regret of \protect{\otsls{}}}\label{ra2sls}
In this section, we elaborate on the proofs and techniques to bound the oracle regret of \otsls{}.

\begin{reptheorem}{thm:reg_otsls_oracle}[Oracle regret of \otsls]
   Under the same assumptions of  \Cref{thm:reg_otsls_identification}, the oracle regret of \otsls{} at step $T > 1$ is upper bounded by
\begin{align*}
\centering
%    \overline{R}_{T} 
%\leq
\underset{\substack{\text{Identif.}\\ \text{Regret} \\\bigO(d_{\bx}d_{\bz} \log^2 T)}}{\underbrace{\widetilde{R}_{T}}}%\\
    + 
    \underset{\substack{\text{Estimation} \\\bigO(\sqrt{d_{\bz} \log T})}}{\underbrace{\sqrt{\mathfrak b_{T-1}(\delta)}}}
    \underset{
        \substack{\text{First-stage}\\\text{feature norm} \\\bigO(\sqrt{\log T})}
        }
    {
\left( 
    \underbrace{C_1 
    \sqrt{
    f(T)}}
\right.
    }+
\underbrace{C_2\sqrt{
    2   
    d_{\bx} f(T) 
    } + \sqrt{d_{\bx}} C_3
    }_{\substack{\text{Correlated noise}\\\text{Concentration term} \\\bigO(\sqrt{d_{\bx}\log T})}} 
+
    \underset{\substack{\text{Correlated noise}\\ \text{Bias term}\\\bigO(\gamma \sqrt{T} )}}
    {\left.
    \underbrace{
    % \|\boldsymbol{\gamma}\|_2
    \gamma C_4 
    \sqrt{T}
    }\right)}
\end{align*}
with probability at least $1-\delta \in (0,1)$. 
Here, $\mathfrak b_{T-1}(\delta)$ is the confidence interval defined by \Cref{thm:confidencebeta}, $f(T)$ is defined in \Cref{cor:ft}, $C_5$ in \Cref{lem:bound_noises}, $\gamma \triangleq\|\boldsymbol{\gamma}\|_2 = \| \mathbb E[\eta_s \bep_s] \|_2$. $C_1$, $C_2$, $C_3$, and $C_4$ are $d_{\bz}$, $d_{\bx}$ and $T$-independent positive constants. %$\llmin*{\bsig_{\bz}}$ is the minimum eigenvalue of IVs' design matrix.
\end{reptheorem}
\begin{proof}
Using Equation~\eqref{eq:cumulativeregret}, and defining $\Delta \bb_{t-1} \triangleq \left(\bb_{t-1}-\bb\right)$, the instantaneous regret at step $t$ is
\begin{align*}
    \overline{r}_{t}
&\triangleq
    \ell_{t}\left(\bb_{t-1}\right)-\ell_{t}\left(\bb\right)
=
    \left(
        y_t - \bb_{t-1}^{\top} \bx_{t}
    \right)^{2}- 
    \left(
    y_t - \bb^{\top} \bx_{t}
    \right)^{2}\\
% &=
%     \left(
%         \bb^\top \bx_t + \eta_t - \bb_{t-1}^{\top} \bx_{t}
%     \right)^{2}- 
%     \eta_t^{2}\\
&=
    \left(
        \left(\bb_{t-1}- \bb\right)^{\top} \bx_t - \eta_t 
    \right)^{2}- 
    \eta_t^{2}
=
    \left(
        \left(\bb_{t-1}- \bb\right)^{\top} \bx_t \right)^{2}
    +
    2 \eta_{t}\left(\bb_{t-1}- \bb\right)^{\top} \bx_t\\
&=
    \left(
        \Delta \bb_{t-1}^{\top} \bx_t \right)^{2}
    +
    2 \eta_{t}\Delta \bb_{t-1}^{\top} \bx_t
\end{align*}
 Since $\bx_{t} = \bt^\top \bz_t + \bep_t$ by \eqref{eq:first_stage}, the second term can be rewritten substituting  as 
\begin{align*}
    2\eta_{t}\Delta \bb_{t-1}^{\top} \bx_{t}
=
    2\eta_{t}\Delta \bb_{t-1}^{\top} \bt^\top \bz_t 
    + 
    2\eta_{t}\Delta \bb_{t-1}^{\top} \bep_t 
\end{align*}

Therefore, the cumulative oracle regret $ \overline{R}_{T}$ by horizon $T$ is
\begin{align}
\overline{R}_{T}
=      \sum_{t=1}^{T} \overline{r}_{t}
=
    \underbrace{
    \sum_{t=1}^{T}\left(\Delta \bb_{t-1}^{\top} \bx_{t}\right)^{2}
    }_{(\bullet 1\bullet)}
    +
    2\underbrace{\sum_{t=1}^{T} \eta_{t}\Delta \bb_{t-1}^{\top} \bt^\top \bz_t}_{(\bullet 2 \bullet )} 
    + 
    2\underbrace{\sum_{t=1}^{T} \eta_{t}\Delta \bb_{t-1}^{\top} \bep_t}_{(\bullet 3 \bullet )} \label{eq:decompose}
\end{align}
The proof proceeds by bounding each of the three terms individually.

\paragraph{Term 1: Second-stage regression error.}\label{s1}
The first term $(\bullet 1\bullet)$ quantifies the error introduced by the second stage regression. This is exactly equal to the identification regret. Thus, by \Cref{thm:reg_otsls_identification},
\begin{align*}
(\bullet 1 \bullet)
&\leq
\underbrace{\mathfrak b_{T-1}(\delta)}_{d_{\bz}\log T} 
    \underbrace{\left(
    \left(L_{\bt}^2 \red{\frac{\mathfrak{r}^4}{L_{\bz}^2}} + \red{\frac{\mathfrak{r}^4}{L_{\bz}^4}} d_{\bx} \sigma_{\bep}^2\right)
f(T)  
+ \red{\frac{\mathfrak{r}^4}{L_{\bz}^4}} C_5
    \right)}_{d_{\bx}\log T}
% \underbrace{\mathfrak b_{T-1}(\delta)}_{d_{\bz}\log T} \underbrace{\left(\left(L_{\bt}^2 \mathfrak{r}^2 + \frac{\mathfrak{r}^2}{L_{\bz}^2} d_{\bx} \sigma_{\bep}^2\right)f(T)  + \frac{\mathfrak{r}^2}{L_{\bz}^2} C_5\right)}_{d_{\bx}\log T}\\
= \bigO(d_{\bx}d_{\bz}\log^2 T).
\end{align*}

\paragraph{Term 2: Coupling of first-stage features and second-stage parameter estimation.}
Now, we bound $(\bullet 2 \bullet )$ using martingale inequalities similar to the ones used for the confidence intervals to derive a uniform high probability bound.

\textit{Step 1:} Following \Cref{thm:uniform-bound}, we define \[
    w_s \triangleq \left(\bb_{s-1}-\bb\right)^{\top} \bt^\top \bz_s 
\quad\text{ and }\quad 
    \mathcal F_{t-1}\triangleq\sigma\left(\bz_1, \bep_1, \eta_1,\ldots, \bz_{t-1}, \bep_{t-1}, \eta_{t-1}, \bz_t   \right).\]
It is immediate to verify that the hypothesis are satisfied, since $w_t$ is $\mathcal F_{t-1}$-measurable as $\bb_{t-1}$ and $\bz_t$ are too. 
Bearing in mind this substitution we have

\begin{align}\label{eq:term2_martingale}
    \left|
    \sum_{t=1}^{T} \eta_{t} {w}_t
    \right| 
&\leq
   \sqrt{2\left(1 
    +\sigma_{\eta}^2\sum_{t=1}^{T}{w}_t^{2}\right)\log \left(\frac{\sqrt{1+\sigma_{\eta}^2 \sum_{t=1}^{T} {w}_t^{2}}}{\delta}\right)}
\end{align}
with probability at least $1-\delta$.

\textit{Step 2:}  Now, the problem reduces to upper bounding $\sum_{t=1}^T w_t^2$.
We proceed like for the first term in the \textit{Step 2} for \textbf{Term 1}:
\begin{align*}
    \sum_{t=1}^T w_t^2
=    
\sum_{t=1}^T\left\langle\bb_{t-1}-\bb, \bt^\top \bz_t\right\rangle^2 
&\underset{(a)}{\leq} 
    \sum_{t=1}^{T} \mathfrak b_{t-1}(\delta) \left\|\bt^\top \bz_t\right\|_{\widehat{\bH}_{t-1}^{-1}}^{2} \\ 
&\underset{(b)}{\leq}
    \mathfrak b_{T-1}(\delta) \sum_{t=1}^{T}  \left\|\bt^\top \bz_t\right\|_{\widehat{\bH}_{t-1}^{-1}}^{2}\\ 
&\underset{(c)}{\leq} 
\mathfrak b_{T-1}(\delta)  \red{
    \frac{L_{\bt}^2 
    \mathfrak{r}^4}{L_{\bz}^2}
    }
    \sum_{t=1}^{T} 
    \frac{1}{\llmin*{\bG_{\bz,t-1}}  }\\
    % \mathfrak b_{T-1}(\delta) 
    %  L_{\bt}^2 
    % \mathfrak{r}^2
    % \sum_{t=1}^{T} 
    % \frac{1}{\llmin*{\bG_{\bz,t-1}}  } \\
    &{\leq} 
\mathfrak b_{T-1}(\delta)  \red{
    \frac{L_{\bt}^2 
    \mathfrak{r}^4}{L_{\bz}^2}
    } f(T)
    =
     \mathcal{O}\left(d_{\bz} \log^2 T \right) \: .
\end{align*}
where $(a)$ comes from Cauchy-Schwarz and \Cref{thm:confidencebeta}, $(b)$ from monotonicity of $\mathfrak  b_{t-1}(\delta)$ and $(c)$ from Equation~\eqref{eq:theta_z}.
%where in the first inequality we also used the fact that $\mathfrak{b}_{t-1}(\delta)$ is monotonically increasing in $t$, to take the radii outside the summation.

\textit{Step 3:} Thus, substituting the bound of \textit{Step 2} into \Cref{eq:term2_martingale}, we obtain
\begin{align*}
   2\sum_{t=1}^{T} \eta_{t}( \bb_{t-1}- \bb)^{\top} \bt^\top \bz_t
&\leq
   \underbrace{
   \sqrt{2\left(1 
    +\red{\sigma_{\eta}^2 }\mathfrak b_{T-1}(\delta)  
    \red{
    \frac{L_{\bt}^2 
    \mathfrak{r}^4}{L_{\bz}^2}
    f(T)}\right)}}_{\bigO(\sqrt{d_{\bz}} \log T)}
%     \\
% &\quad\quad~~
\underbrace{
\sqrt{\log \Bigg( {\delta}^{-1}
    \sqrt{
    1 + \red{\sigma_{\eta}^2 }\mathfrak b_{T-1}(\delta)  
    \red{
    \frac{L_{\bt}^2 
    \mathfrak{r}^4}{L_{\bz}^2}
    f(T)}\Bigg)}}}_{\bigO\left(\sqrt{\log(\sqrt{d_{\bz}} \log T)}\right)}
    \\
&=
    \bigO\left(  
    \sqrt{d_{\bz}} \log T
    \sqrt{\log(\sqrt{d_{\bz}} \log T)}\right)
\end{align*}
with probability at least $1-\delta$.

Thus, the order of $(\bullet 2 \bullet )$ is negligible with respect to term $(\bullet 1 \bullet )$.

\paragraph{Term 3: Coupling of First- and Second-stage Noises.}\label{s3}
We bound the term $(\bullet 3 \bullet)$ containing the self-fulfilling bias, i.e. the correlation between the first- and second-stage noise, by splitting it into two contributions.
\begin{align}
    \sum_{t=1}^T \eta_t ( \underbrace{\bb_{t-1}-\bb }_{\Delta \bb_{t-1}})^\top \bep_t 
&=
    \sum_{t=1}^T \Delta \bb_{t-1}^{\top}\left(\bep_t \eta_t\right) + \sum_{t=1}^T \Delta \bb_{t-1}^{\top} \mathbb{E}_t\left[\bep_t \eta_t\right]-  \sum_{t=1}^T \Delta \bb_{t-1}^{\top} \mathbb{E}_t\left[\bep_t \eta_t\right]\notag \\
& =
    \sum_{t=1}^T \Delta \bb_{t-1}^{\top} \bep_t \eta_t-\sum_{t=1}^T \Delta \bb_{t-1}^{\top} \boldsymbol{\gamma}+\sum_{t=1}^T \Delta \bb_{t-1}^{\top} \boldsymbol{\gamma} \notag\\
& =
    \underbrace{\sum_{t=1}^T \Delta \bb_{t-1}^T\left(\bep_t n_t-\boldsymbol{\gamma}\right)}_{\substack{\text{Martingale Concentration Term}}}
    +
    \underbrace{
    \sum_{t=1}^T \Delta \bb_{t-1}^{\top} \boldsymbol{\gamma}
    }_{\substack{\text{Bias Term}}}
\label{eq:decomp_subg_3}
\end{align}
Now we can use \Cref{lem:selffulfilling} and \Cref{lem:biasOsqrtT} to conclude that term $ (\bullet 3 \bullet ) $ is bounded by the following quantity
\begin{align*}
   &~~~~\sum_{t=1}^T \eta_t ( \underbrace{\bb_{t-1}-\bb }_{\Delta \bb_{t-1}})^\top \bep_t\\  
   &\leq
    \underbrace{
      8e^2\left(\sigma_\eta^2+\sigma_{\bep}^2\right) 
    \sqrt{d_{\bx}} 
    \red{\frac{\mathfrak{r}^2}{L_{\bz}^2}}
    \sqrt{\mathfrak b_{T-1}(\delta)}
    \left(
    \sqrt{2 f(T) \log (2 / \delta)}
    + \sqrt{
    \max\left\{ 
        \frac{1}{\lambda},
        \frac{2}{ \llmin{\mathbf{\boldsymbol{\Sigma}}}}
    \right\}
    }
    \log (2 / \delta)\right)
    }_{\sqrt{d_{\bx}d_{\bz}} \log(T)}
\\
 &\quad   +
    \underbrace{ 
    \|\boldsymbol{\gamma}\|_2 \red{\frac{\mathfrak{r}^2}{L_{\bz}^2}} \sqrt{\mathfrak b_{T}(\delta)}
    \left(
    \frac{T_0}{\sqrt{\lambda}}
    +2
    \frac{\sqrt{2T}}{\sqrt{\llmin*{\bsig}}}
    \right)
    }_{\bigO(\|\boldsymbol{\gamma}\|_2 \sqrt{d_{\bz} T \log T} )}
\end{align*}

\noindent\textbf{Final Step: Assembling the Results.} Thus, by assembling the upper bounds on the three terms, we get
\begin{align*}
\overline{R}_{T}&=\sum_{t=1}^{T} \overline{r}_{t}\\
&\leq
\underbrace{\mathfrak b_{T-1}(\delta)}_{d_{\bz}\log T} \underbrace{\left(\left(\red{
    \frac{L_{\bt}^2 
    \mathfrak{r}^4}{L_{\bz}^2}} + \red{
    \frac{\mathfrak{r}^4}{L_{\bz}^4}} d_{\bx} \sigma_{\bep}^2\right) f(T)  
+ \red{
    \frac{\mathfrak{r}^4}{L_{\bz}^4}} C_5\right)}_{d_{\bx}\log T}\\
&+   \underbrace{
   \sqrt{2\left(1 
    +\red{\sigma_{\eta}^2 }\mathfrak b_{T-1}(\delta)  
    \red{
    \frac{L_{\bt}^2 
    \mathfrak{r}^4}{L_{\bz}^2}
    f(T)}\right)}}_{\bigO(\sqrt{d_{\bz}} \log T)}
    \underbrace{
\sqrt{\log \Bigg( {\delta}^{-1}
    \sqrt{
    1 + \red{\sigma_{\eta}^2 }\mathfrak b_{T-1}(\delta)  
    \red{
    \frac{L_{\bt}^2 
    \mathfrak{r}^4}{L_{\bz}^2}
    f(T)}\Bigg)}}}_{\bigO\left(\sqrt{\log(\sqrt{d_{\bz}} \log T)}\right)}\\
&+
    \underbrace{
      8e^2\left(\sigma_\eta^2+\sigma_{\bep}^2\right) 
    \sqrt{d_{\bx}} 
    \red{\frac{\mathfrak{r}^2}{L_{\bz}^2}}
    \sqrt{\mathfrak b_{T-1}(\delta)}
    \left(
    \sqrt{2 f(T) \log \frac{2}{\delta}
    }
    + \sqrt{
    \max\left\{ 
        \frac{1}{\lambda},
        \frac{2}{ \llmin{\mathbf{\boldsymbol{\Sigma}}}}
    \right\}
    }
    \log\frac{2}{\delta}\right)
    }_{\bigO(\sqrt{d_{\bx}d_{\bz}} \log(T))}\\
    &\quad   +
    \underbrace{ 
    \|\boldsymbol{\gamma}\|_2 \red{\frac{\mathfrak{r}^2}{L_{\bz}^2}} \sqrt{\mathfrak b_{T-1}(\delta)}
    \left(
    \frac{T_0}{\sqrt{\lambda}}
    +2
    \frac{\sqrt{2T}}{\sqrt{\llmin*{\bsig}}}
    \right)
    }_{\bigO(\|\boldsymbol{\gamma}\|_2 \sqrt{d_{\bz} T \log T} )}\\
&\underset{(a)}{\leq}
\mathfrak b_{T-1}(\delta)
\left(\left(\red{
    \frac{L_{\bt}^2 
    \mathfrak{r}^4}{L_{\bz}^2}} + \red{
    \frac{\mathfrak{r}^4}{L_{\bz}^4}} d_{\bx} \sigma_{\bep}^2\right) f(T)  
+ \red{\frac{\mathfrak{r}^4}{L_{\bz}^4}} C_5\right)\\
&+ \sqrt{\mathfrak b_{T-1}(\delta)} \Bigg(C_1 \sqrt{f(T)} +  C_2 \sqrt{2 {d_{\bx}} f(T)}
+ T_0 \sqrt{d_{\bx}} + \|\boldsymbol{\gamma}\|_2 C_4
{\sqrt{T}}\Bigg)\\
&=    \underset{\substack{\text{Identification}\\ \text{Regret} \\\bigO(d_{\bx}d_{\bz} \log^2 T)}}{\underbrace{\bigO(\mathfrak b_{T-1}(\delta) d_{\bx} f(T))}}%\\
    + 
    \underset{\substack{\text{Estimation} \\\bigO(\sqrt{d_{\bz} \log T})}}{\underbrace{\sqrt{\mathfrak b_{T-1}(\delta)}}}
    \underset{
        \substack{\text{First-stage}\\\text{feature norm} \\\bigO(\sqrt{\log T})}
        }
    {
\left( 
    \underbrace{C_1 
    \sqrt{
    f(T)}}
\right.
    }+
\underbrace{C_2\sqrt{
    2   
    d_{\bx} f(T) 
    } + \sqrt{d_{\bx}} C_3
    }_{\substack{\text{Correlated noise}\\\text{Concentration term} \\\bigO(\sqrt{d_{\bx}\log T})}} 
+
    \underset{\substack{\text{Correlated noise}\\ \text{Bias term}\\\bigO(\gamma \sqrt{T} )}}
    {\left.
    \underbrace{
    % \|\boldsymbol{\gamma}\|_2
    \gamma C_4 
    \sqrt{T}
    }\right)}
\end{align*}
% \begin{align*}
%     = \bigO(d_{\bx}d_{\bz}\log^2 T) +  \bigO\left(\sqrt{d_{\bz}} \log T \sqrt{\log(\sqrt{d_{\bz}} \log T)}\right) + \bigO(\sqrt{d_{\bx}d_{\bz}} \log(T)) + \bigO(\|\boldsymbol{\gamma}\|_2 \sqrt{d_{\bz} T \log T} )\\
%     &= \bigO(d_{\bx}d_{\bz}\log^2 T) +  \bigO(\|\boldsymbol{\gamma}\|_2 \sqrt{d_{\bz} T \log T} )\\
% \end{align*}

For simpler presentation of the result, we ignore the $\log\log$ terms and define in inequality (a) the constants
\begin{align*}
    C_1 &\triangleq \sqrt{2\left(1 + \red{\sigma_{\eta}^2 \frac{L_{\bt}^2 \mathfrak{r}^4}{L_{\bz}^2}} \right)}\\
    C_2 &\triangleq 8e^2\left(\sigma_\eta^2+\sigma_{\bep}^2\right) \red{\frac{\mathfrak{r}^2}{L_{\bz}^2}} \sqrt{\log \frac{2}{\delta}}\\
    C_3 &\triangleq 8e^2\left(\sigma_\eta^2+\sigma_{\bep}^2\right) \red{\frac{\mathfrak{r}^2}{L_{\bz}^2}} 
    \sqrt{ \max\left\{ 
        \frac{1}{\lambda},
        \frac{2}{ \llmin{\mathbf{\boldsymbol{\Sigma}}}}
    \right\}} \log \frac{2}{\delta}\\
    C_4 &\triangleq  \red{\frac{\mathfrak{r}^2}{L_{\bz}^2}}
    \left(T_0 {\lambda}^{-1/2}
    +2 \sqrt{2} \llmin*{\bsig}^{-1/2} \right)\: .
\end{align*}
$C_1$, $C_2$, $C_3$, and $C_4$ are dimension $d_{\bx}$, $d_{bz}$, and horizon $T$ independent constants leading to a regret upper bound of order
\[\bigO(d_{\bx}d_{\bz}\log^2 T) +  \bigO(\|\boldsymbol{\gamma}\|_2 \sqrt{d_{\bz} T \log T} ).\]

\end{proof}

\begin{remark}[Just-identified case.] For the just-identified IVs, $d_{\bx} = d_{\bz} = d$. Thus, the bound on oracel regret of \otsls{} reduces to $\bigO(d^2 \log^2 T) +  \bigO(\gamma \sqrt{d T \log T} )$.     
\end{remark}

\begin{remark}[Exeogeneous case.]
    For the exogeneous case, $\gamma = 0$. Thus, we \otsls{} achieves an oracle regret bound $\bigO(d^2 \log^2 T)$. This is comparable in terms of dependence on $d$ and $T$ to the oracle regret of online ridge regression and online VAWR in the exoegeneous setting.
\end{remark}

\begin{remark}[A discussion on $\lambda$] In the design matrix, we use the fact that $\bG_{\bz,t} \triangleq \bZ_t^{\top} \bZ_t = \sum_{s=1}^t \bz_s \bz_{s}^\top+ \bG_{\bz,0}$, such that $\lambda_{\min}(\bG_{\bz,0})\geq \lambda$. We discuss different ways to achieve it and standard rectifications, if we remove this condition.

\textit{1. Adding a regulariser:} We can ensure $\bG_{\bz,0} = \lambda \mathbf{I}_{d_{\bz}}$, and thus, $\lambda_{\min}(\bG_{\bz,0}) = \lambda$. This can be achieved by using a $\ell_2$ regulariser, i.e. $\lambda \|\bb\|_2^2$, in the second-stage, such that $\lambda > 0$.

\textit{2. Data augmentation:} The other way to ensure is to consider a data augmentation approach. Rather than begining with $\bz_0 = \mathbf{0}$, i.e. a vector with all zero entries, we can choose a $\bz_0$ during initialisation, such that $\lambda_{\min}(\bG_{\bz,0}) = \normiii{\bz_0 \bz_0^{\top}}_2 \geq \lambda$, and $\norm{\bz_0}_2^2 = \bz_0^{\top} \bz_0 \leq L_{\bz}^2$.

If we want to remove this condition that $\lambda_{\min}(\bG_{\bz,0})\geq \lambda$, and redefine $\bG_{\bz,t}$ as $\sum_{s=1}^t \bz_s \bz_{s}^\top$, then we can modify the bounds in two ways:

\textit{1. Well-behavedness of IVs:} It is common in econometrics and OLS regression literature to assume well-behavedness of the covariates for simplifying the analysis (Assumption A.2. in Chapter 8.2~\citep{greene2003econometric}; Theorem 13.13.~\citep{wasserman2004all}. In our case, well-behavedness means that for any $t>0$, $\llmin{\sum_{s=1}^t \bz_s \bz_{s}^\top} > 0$, i.e. the first-stage design matrix dictated by the IVs is invertible at every step. In this case, we can proceed with our analysis by substituting $\lambda = \inf_{t\geq 0} \llmin{\sum_{s=1}^t \bz_s \bz_{s}^\top}$.

\textit{2. Added cost for initial design matrices:} Now, if we want to be more flexible and not to consider the well-behavedness of IVs assumption, we have to follow the rectifications as in the analysis of unregularised OLS by~\cite{maillard2016self} or the analysis of unregularised forward (VAWR) algorithm by~\cite{ouhamma2021stochastic}. Specifically, if we assume $t = T_0 > 1$ is the first time step when the design matrix becomes positive definite, the identification regret bound becomes $\widetilde{R}_T = \bigO\left(\log(\delta^{-1} C_{\text{peeling}} (T L_{\bz}^2/\gamma d_{\bz})^{d_{\bz}}) \log(\mathrm{det}(G_{\bz, T})/\mathrm{det}(G_{\bz, T_0}))\right)$ with probability at least $1-\delta$ as $T \gg T_0$. Here, $C_{\text{peeling}}$ is some dimension and $T$-independent constant invoked by the peeling technique in Theorem 5.4. of~\citep{maillard2016self}. Thus, we can asymptotically maintain the $\bigO( d_{\bx} d_{\bz} \log^2 T)$ identification regret bound for \otsls{}. Similar relaxations can be derived for all other bounds as an interesting theoretical exercise.
%+ d_{\bx} T_0 \log^2(T_0/\delta)\right)$.
\end{remark}

\newpage

% \newpage
\subsection{Discussion on different definitions of regret}\label{sec:reg_big}

\red{
Under exogeneity, \cite[Theorem 3.1]{ouhamma2021stochastic} show that the oracle and population regrets differ by $o(\log^2 T)$. \textit{In this section, we show it to be not true under endogeneity}.
Thus, \textit{studying the three regrets shows} that the \textit{hardness of identification of underlying model (the identification regret $\widetilde R_T$), quality of predictions due to unavailability of oracle/true model (oracle regret $\overline{R}_T$), and quality of predictions w.r.t. a fixed dataset/population (population regret $R_T$) can differ significantly under endogeneity}, whereas they are of similar order under exogeneity.
}

Historically, econometrics  has focused on (asymptotic) correct identification of the estimator $\bb$ in presence of endogeneity in order to assign causal meaning to the parameters in a regression model \citep{angrist1995two}.

Instead, the primary concern of the statistics and statistical learning community has arguably been generalisation (see for example \citep{shalev2014understanding}), with the online learning literature initiated by the works by \cite{foster1991prediction} and \cite{littlestone1991line}.

Under exogeneity, the problem of identifying the true parameter $\bb$ is solved at the same time as minimizing the \mse{} by the \ols{} estimator, which is in fact a consistent estimator and is the Minimum Mean Squared Error Estimator (\mmse{}). This is not true if we introduce the more realistic assumption of endogeneity. In fact the estimator is biased and the  \mmsee{} is different from the oracle $\bb$ in this case:    
\begin{align*}
    \bb^{\mmsee} 
&=
    \mathbb{E}\left[\bx \bx^{\top}\right]^{-1} \mathbb{E}[\bx y]
 =
    \bb+\mathbb{E}\left[\bx \bx^{\top}\right]^{-1} \mathbb{E}[\bx \eta],
\end{align*}
% \begin{align*}
%     &\partial_{\beta^{\erm}} \mathbb{E}_{\bx y}\left[\left(y-\bx \bb^{\erm}\right)^2\right]=0 \\
% &\quad\quad\quad\quad\iff 
%     \mathbb{E}[\bx y]=\mathbb{E}\left[\bx \bx^{\top}\right] \bb^{\erm} 
% \end{align*}
where we isolated the bias term  $\mathbb{E}\left[\bx \bx^{\top}\right]^{-1} \mathbb{E}[\bx \eta]$.
The expression simply follows by solving for $\bb'$  in
$\partial_{\beta'} \mathbb{E}_{\bx y}\left[\left(y-\bx \bb'\right)^2\right]=0 
$.
% , and by substituting for the definition of $y$ in  $ 
% \mathbb{E}[\bx y]=\mathbb{E}\left[\bx \bx^{\top}\right] \bb^{*} $ (provided that we can invert the covariance matrix $\mathbb{E}\left[\bx \bx^{\top}\right]$).
Under exogeneity, $\mathbb{E}[\bx \eta] = 0$ and the \mmsee{} conicedes with $\bb$, while, in the more realistic case of endogenous noise $\mathbb{E}[\bx \eta] \neq 0$ and the estimator is biased for $\bb$.

Therefore, the oracle is  not the best estimator in terms of prediction accuracy. 
Instead, we should aim for a consistent estimator  of the \mmse{}, like the Empirical Risk Minimiser (\erm) which leads to the \ols{} estimator. 
In the endogenous setting, the \ols{} estimator is biased for the oracle $\bb$, and this poses the problem of assigning causal meaning to it.

% To avoid these difficulties, econometricians have historically focused on the endogenous setting, not considering out-of-sample prediction. On the contrary, the statistical learning community has focused on the exogenous case where the problem of assigning causal meaning to the estimators is simply solved by \ols{} and they have cared mostly about generalisation. 

% \todoR{all previous paragraph needs a check: eg is the \mmse{} the one on the population or has it a differnt name? \dbcomment{I do not know a different name. the maximum difference is the oracle one is ëxpected risk", the population one is ëmpirical risk" but both lead to same estimator under realisable setting.}}

\textit{We are interested in finding notions  of regret that preserve the causal interpretation of the estimators while allowing a study of their generalisation properties}. This has led us to extend the \emph{instrumental variables} analysis to the online setting and to the inspection of different notions of regret. The regrets that we introduce below, differ from the \emph{population regret} which is usually considered in online learning under the exogeneity assumption. The population regret is defined as
\begin{equation}
    R_T(\bb_T)
\triangleq
    \sum_{t=1}^T(y_t-\bx_t^\top \bb_{t-1})^2- \min_{\bb} \sum_{t=1}^T(y_t-\bx_t^\top \bb )^2
\end{equation}
for a time-dependent estimator $\bb_{t-1}$ and where we indicated with $\bb_T \triangleq \arg\min_{\bb} \sum_{t=1}^T(y_t-\bx_t^\top \bb )^2$. 
Instead, we introduce two  alternative regrets that   measure the performance of an estimator $\bb_{t-1}$ with respect to the oracle $\bb$ (instead of $\bb_T$). Respectively, they are the \emph{oracle regret} $\overline{R}_T(\bb)$ and the \emph{identification regret} $\widetilde{R}_T(\bb)$ below
\begin{align*}
    \overline{R}_T(\bb) 
= 
    \sum_{t=1}^T(y_t-\bx_t^\top \bb_{t-1})^2- \sum_{t=1}^T(y_t-\bx_t^\top \bb )^2,\quad  \quad
\quad
     \widetilde R_T(\bb)
=
    \sum_{t=1}^T(\bx_t^\top \bb_{t-1}-\bx_t^\top \bb )^2.
\end{align*}

 \Cref{thm:lower} shows that an estimator that performs well in terms of \emph{oracle regret} is instead a bad choice for the \emph{population regret}, and the other way around.
\begin{theorem}\label{thm:lower}
For any estimator, the quantity $\Delta \mathfrak R_T \triangleq R_T\left(\bb_T\right)-\overline{R}_T(\bb) $ is lower bound in expectation by
$\mathbb E[\Delta \mathfrak R_T] = \Omega(T)$
\end{theorem}
\textit{Proof Sketch.} Solving for $\bb_T$ leads to the OLS estimator for data up to time $T$:
$\bb_{T}=\bb +\left(\sum_{t=1}^T\bx_{t} \bx_{t}^{\top}\right)^{-1} \sum_{t=1}^T \bx_{t}^{\top }\eta_t$ provided that we can invert the design matrix in the previous expression.
Using this expression we ca rewrite $\Delta \mathfrak R_T$ using $\Delta \bb_T \triangleq \bb-\bb_T$ and $\bG_{\bx, T} \triangleq \sum_{t=1}^T\bx_{t} \bx_{t}^{\top}$ as follows
\begin{align*}
    \Delta \mathfrak R_T
&=
    \sum_{t=1}^T\left(\bx_t^{\top}\Delta \bb_T\right)^2+2 \sum_{t=1}^T\eta_t\bx_t^{\top}\Delta \bb_T
= 
    3 \|\Delta \bb_T\|^2_{\bG_{\bx, T}}.
\end{align*}
Thanks to \Cref{prop:concentration_eigenvalue}, we know that the minimum eigenvalue of $\bG_{\bx, T}$ is $\Omega(T)$ which implies $\|\Delta \bb_T\|^2_{\bG_{\bx, T}}\geq \|\Delta \bb_T\|^2_2 \llmin{\bG_{\bx, T}} \gtrsim \|\Delta \bb_T\|^2_2 T $. Furthermore, we can bound away from zero in expectation $\|\Delta \bb_T\|_2^2$ by using Cramér–Rao bound on each component of the estimator $\bb_T$, for which we have $\mathbb{E}\left[(\bb_{T,i}-\bb_i)^2\right] \geq \frac{\left[1+b^{\prime}(\bb_i)\right]^2}{I(\bb_i)}+b(\bb_i)^2$ where $b(\bb_i)= \mathbb E[\bb_{T,i}]-\bb_{i}$ is the bias of the estimator and $I(\bb_i)$ is the Fisher Information evaluated at $\bb_i$. We know that $\bb_T$ is a biased estimator of $\bb$ in the endogenous setting, therefore the bias is strictly positive for at least one component, and this concludes the proof.

\newpage
\section{Regret Analysis for IV Linear Bandits: \ofuliv{}}\label{app:ofuliv}

\begin{reptheorem}{thm:reg_ofuliv}[Regret of \ofuliv{}]
Under the same assumptions as that of Theorem~\ref{thm:reg_otsls_oracle}, for horizon $T > 1$, \Cref{alg:ofuliv} incurs a regret
\begin{equation*}
    R_T 
= 
    \sum_{t=1}^T r_t 
% = 
%     \mathcal{O}\left( d\sqrt{T\log\left( \frac{T}{\delta}\right)\left(\frac{C_3+1}{\lambda} + \frac{\log(T)+1}{\llmin*{\bsig}}\right)}\right),
\leq 
     2 \sqrt{T} 
    \sqrt{\mathfrak b_{T-1}(\delta)}
    \sqrt{\left(\left(\red{
    \frac{L_{\bt}^2 
    \mathfrak{r}^4}{L_{\bz}^2}} + \red{
    \frac{\mathfrak{r}^4}{L_{\bz}^4}} d_{\bx} \sigma_{\bep}^2\right) f(T)  
+ \red{\frac{\mathfrak{r}^4}{L_{\bz}^4}} C_5\right)}\: ,
\end{equation*}
with probability at least $1-\delta$. Here, $C_5$ is $\bigO(d_{\bx})$ constant as defined in \Cref{eq:c4}.
\end{reptheorem}
\begin{proof} 

The instantaneous regret $ r_{t} $ reads
\begin{align*}
    \left\langle\bb, \bx_{*}\right\rangle-\left\langle\bb, \bx_{t}\right\rangle 
&\leq
    \left\langle\widetilde{\bb}_{t-1}, \bx_{t}\right\rangle-\left\langle\bb, \bx_{t}\right\rangle 
    \tag{since $(\bx_{t}, \widetilde{\bb}_{t-1})$ is optimistic inside $\mathcal X_t \times \mathcal B_t$}\\
% &=
%     \left\langle\widetilde{\bb}_{t-1}-\bb, \bx_{t}\right\rangle \\
&=
    \left\langle\bb_{t-1}-\bb, \bx_{t}\right\rangle+\left\langle\widetilde{\bb}_{t-1}-\bb_{t-1}, \bx_{t}\right\rangle 
    \tag{summing and subtracting $\bb_{t-1}$}
    \\
& \leq
    \left\|\bb_{t-1}-\bb\right\|_{ \widehat\bt_{t-1}^\top \bG_{\bz,t-1}  \widehat\bt_{t-1}}
    \left\|\bx_{t}\right\|_{\left({ \widehat\bt_{t-1}^\top \bG_{\bz,t-1}  \widehat\bt_{t-1}} \right)^{-1}} \notag\\
    &\quad+
    \left\|\widetilde{\bb}_{t-1}-\bb_{t-1}\right\|_{\widehat\bt_{t-1}^\top \bG_{\bz,t-1}  \widehat\bt_{t-1}}
    \left\|\bx_{t}\right\|_{\left({ \widehat\bt_{t-1}^\top \bG_{\bz,t-1}  \widehat\bt_{t-1}} \right)^{-1}} \tag{Cauchy-Schwarz } \\
% & \leq
%     \left\|\bb_{t-1}-\bb\right\|_{\bar{V}_{t}}\left\|\bx_{t}\right\|_{\bar{V}_{t}^{-1}}^{-1}+\left\|\widetilde{\bb}_{t-1}-\bb_{t-1}\right\|_{\bar{V}_{t}}\left\|\bx_{t}\right\|_{\bar{V}_{t}^{-1}} \\
& \leq 
    2 \sqrt{\mathfrak b_{t-1}(\delta)}
    \left\|\bx_{t}\right\|_{\widehat\bH_{t-1}^{-1}}\tag{\Cref{thm:confidencebeta}}
\end{align*}
The last inequality uses the concentration of $\bb_t$ around the true value $\bb$, and the fact that we choose $\widetilde{\bb}_{t-1}$ inside $\cB_{t-1}$. In both cases, the two norms are bounded by the radius of the ellipsoid, i.e. $\sqrt{\mathfrak b_{t-1}(\delta)}$.
Since we already know in this case how to concentrate the sum of the features norms $\sum_{t=1}^{T}\left\|\bx_{t}\right\|_{\widehat\bH_{t-1}^{-1}}^{2}  $ from \Cref{lem:bound_featsx}, we bound the cumulative regret using Cauchy-Schwarz inequality in the first inequality below, and we substitute the bound on the instantaneous regret that we just obtained: 
\begin{align}
    R_{T} 
&{\leq} 
    \sqrt{T \sum_{t=1}^{T} r_{t}^{2}} \notag
\leq
    2 \sqrt{T \sum_{t=1}^{T} \mathfrak b_{t-1}(\delta)
    \left\|\bx_{t}\right\|_{\widehat\bH_{t-1}^{-1}}^2}\notag
\leq
    2 \sqrt{T} \sqrt{\mathfrak b_{T-1}(\delta)}
    \sqrt{
    \sum_{t=1}^T 
    \left\|\bx_{t}\right\|_{\widehat\bH_{t-1}^{-1}}^2
    }\label{eq:monotone_b}
\end{align}
where in the second inequality, we used the fact that the radius $\mathfrak{b}_{t-1}(\delta)$ is monotonically increasing in $t$.
Now, we can use \Cref{lem:bound_featsx} to bound the sum of feature norms, and putting all together, we obtain the following bound:
\begin{align*}
    R_{T} 
&\leq
    2 \sqrt{T} 
    \underbrace{\sqrt{\mathfrak b_{T-1}(\delta)}}_{\mathcal{O}(\sqrt{d_{\bz} \log T})}
    % \left(
    \underbrace{
    {\left(\left(\red{
    \frac{L_{\bt}^2 
    \mathfrak{r}^4}{L_{\bz}^2}} + \red{
    \frac{\mathfrak{r}^4}{L_{\bz}^4}} d_{\bx} \sigma_{\bep}^2\right) f(T)  
+ \red{\frac{\mathfrak{r}^4}{L_{\bz}^4}} C_5\right)}^{1/2}
    }_{\mathcal{O}(\sqrt{d_{\bx} \log T})}  
= 
    \mathcal{O}\left( \sqrt{d_{\bx} d_{\bz} T} \log T 
\right).
\end{align*}

\end{proof}
% \dbcomment{Checked till here.}

\newpage
\section{Experimental analysis}\label{App:exp}

% In the adversarial set-up, the additional information provided to the algorithm allows logarithmic regret bounds instead of the classic $\mathcal O(\sqrt{T})$. 
% In a stochastic setting, this difference is not evident, as one can see from our experiments. The two algorithms have almost identical behaviour. 

\textbf{Experimental setup.} We deploy the online regression and bandit experiments in Python3 on a Intel(R) Core(TM) i7-8665U CPU @ 1.90GHz.
The code and the real data for the experiments can be found in the \texttt{Experiments} folder within the supplementary material. 
% It is divided into three files: \texttt{2SLSLinearBandit.ipynb}, \texttt{2SLSOnlineRealData.ipynb}, and \texttt{2SLSOnlineRegression.ipynb}.
These files contain the necessary code for: 
\begin{enumerate}[label=\roman*.]
    \item online regression and linear bandit experiments with varying covariate dimensions and endogeneity levels;
    \item the simulations for price-sales dynamics, both in a regression setting and in the pricing case described in \Cref{example1};
    \item online estimation for real data of gasoline consumption.
\end{enumerate}
We refer to the code for a detailed understanding of the implementation and to reproduce our results.

\noindent
\textbf{Notation:} 
We denote the normal distribution with mean $\mu$ and standard deviation $\sigma$ as $\mathcal N(\mu,\sigma)$, with $\mathcal N_n$ we indicate its multivariate extension to $n$ dimensions. Therefore, the normal distribution of $n$ dimensions with mean $\boldsymbol{\mu}$ and standard deviation $\boldsymbol{\Sigma}$ is $\mathcal N_n(\boldsymbol{\mu},\boldsymbol{\Sigma})$. In a similar way, $\mathcal U_n$ is the uniform distribution over $[0,1]^n$. The vector with components all equal to one is denoted by $\vec{1}$.

% \newpage
\subsection{Regression with endogeneity on synthetic data}\label{app:or}
% \textbf{Data generation \& competing algorithms.} 
We conduct an empirical comparison of the performance of \otsls{}, and that of two other algorithms: Online Ridge Regression (\ridge) and the Vovk-Azoury-Warmuth Forecaster (\vaw) \citep{orabona2019modern}. The difference between the two algorithms lies in the design matrices' inverse. The Vovk-Azoury-Warmuth Forecaster, in the inverse of the design matrix, uses the information at time $t$ by adding the covariate available at that time. This addition guarantees better regret bounds on adversarial losses, while on the stochastic setup, the two algorithms are expected to perform similarly, as we confirm experimentally in the following.
Specifically, we compare the following online algorithms and their respective estimators: 
\begin{align}
\bb_{t}^{\ridge} &=  (\bX_{t}^\top\bX_{t} + \lambda_{\ridge} \I_d)^{-1} \bX_{t}^\top \by_{t} , \tag{Online Ridge}\\
\bb_{t}^{\vaw} &=  (\bX_{t+1}^\top\bX_{t+1} + \lambda_{\vaw} \I_d)^{-1} \bX_{t}^\top \by_{t} ,\tag{\vaw{}}\\
 \bb_{t}^{\otsls{}}
&=
    (\widehat{\bX}_{t}^\top  \widehat{\bX}_{t})^{-1} \widehat{\bX}_{t}^\top \by_{t}=
    (\widehat{\bt}_{t}^{\top} \bZ_{t}^\top  \bZ_{t} \widehat{\bt}_{t})^{-1} \widehat{\bt}^{\top}_{t}\bZ_{t}^\top \by_{t} . 
\tag{\otsls} %\\
\end{align}

% We choose the regularisation parameters to be all equal to $10^{-3}$. 
% We consider the true parameter of the model to be generated as $\bb  {\sim} \mathcal N_{50}( \vec {10}, \I_{50})$ and $\bt_{i,j}  {\sim} \mathcal N_1(0,1)$  for each of its entry. At each time step we sample the vectors $ \bz_t  {\sim} \mathcal N_{50}(\vec{0},\I_{50}), $ and $\bep_t  {\sim} \mathcal N_{50}(\vec{0}, \I_{50})$.
% The endogenous noise in the second stage is $\eta_t =  ( \widetilde\eta_t + \sum_{i=1}^{12}\bep_{t,i})/\sqrt{13}$, where $\widetilde\eta_t  {\sim} \mathcal N_1(0,1)$ is a random variable independent from all the others.
% In Figure~\ref{fig:1a_app}, we depict the mean square error between the true parameter of our model and the online estimates. 

\noindent\textbf{Experimental setup.}
Here, we adopt the following experimental setup and generative model.
% \begin{align*}
%     \bx_{t} 
% &
% =
%      \bt\, \bz_{t} + \bep_t    
% = 
%     \bt\, \bz_{t} + 
%     (\boldsymbol{e}_t +\widetilde \bep_{t})  \tag{First Stage}
% \\
%     y_t
% &= 
%     \bb^\top \bx_{t} + \eta_t
% = 
%     \bb^\top \bx_{t} + (\rho \cdot e_{t,1}+\widetilde\eta_{t})
%     \tag{Second Stage}
% \end{align*} 
\begin{align*}
    \bx_{t} 
&
=
     \bt\, \bz_{t} + \bep_t    
    \tag{First Stage}
\\
    y_t
&= 
    \bb^\top \bx_{t} + \eta_t
= 
    \bb^\top \bx_{t} + (\rho \cdot \epsilon_{t,1}+\widetilde\eta_{t}),
    \tag{Second Stage}
\end{align*}

where $\epsilon_{t,1}$ indicates the first component of the vector $\bep_t$, and $\eta_{t} = \rho \cdot \epsilon_{t,1}+\widetilde\eta_{t}$.
We choose $d_{\bx}=\{2,5,8\}$ and $d_{\bz}=\{4,10,16\}$, respectively. In our experiments, we choose arbitrarily $\bb$ as a normalised vector with equal negative entries. Therefore, the values in the components are uniquely determined by the dimension $d_{\bx}$.
We choose $\bt=\I_{d_{\bz},d_{\bx}}$ which has ones on  the entries $i=j$ and zeros for $i\neq j$.
Then, we sample at each time $t$, the vectors $\bz_{t} \sim \mathcal N_{d_{\bz}}(\vec{0}, \I_{d_{\bz}})$, the vector noise   $\bep_{t} \sim \mathcal N_{d_{\bx}}(\vec{0}, \I_{d_{\bx}})$, and the scalar noise $\widetilde\eta_{t} \sim \mathcal{N}_1(0,1)$. We run the algorithms with the same regularisation parameters, i.e. $\lambda=10^{-1}$.\footnote{One can choose to tune $\lambda$ by line search or other hyperparameter tuning algorithms. We chose to set $\lambda=10^{-1}$ and use it in all the experiments.}
We repeat our experiments 20 times. 
For each algorithm, we report the mean and standard deviation of the cumulative regret over the 20 runs (shaded areas correspond to one standard deviation).
In \Cref{f:rcr}, we present the learning curves for the identification regret $\widetilde R_T$. In \Cref{f:rmse}, we plot the  $MSE=\norm{\bb_t-\bb}_2$, which shows the estimation error of the two. In each of the figures, we further compare for different values of increasing endogeneity $\rho$.

\begin{figure}[ht]
\centering
\begin{subfigure}[b]{.3\columnwidth}
    \includegraphics[width=\textwidth]{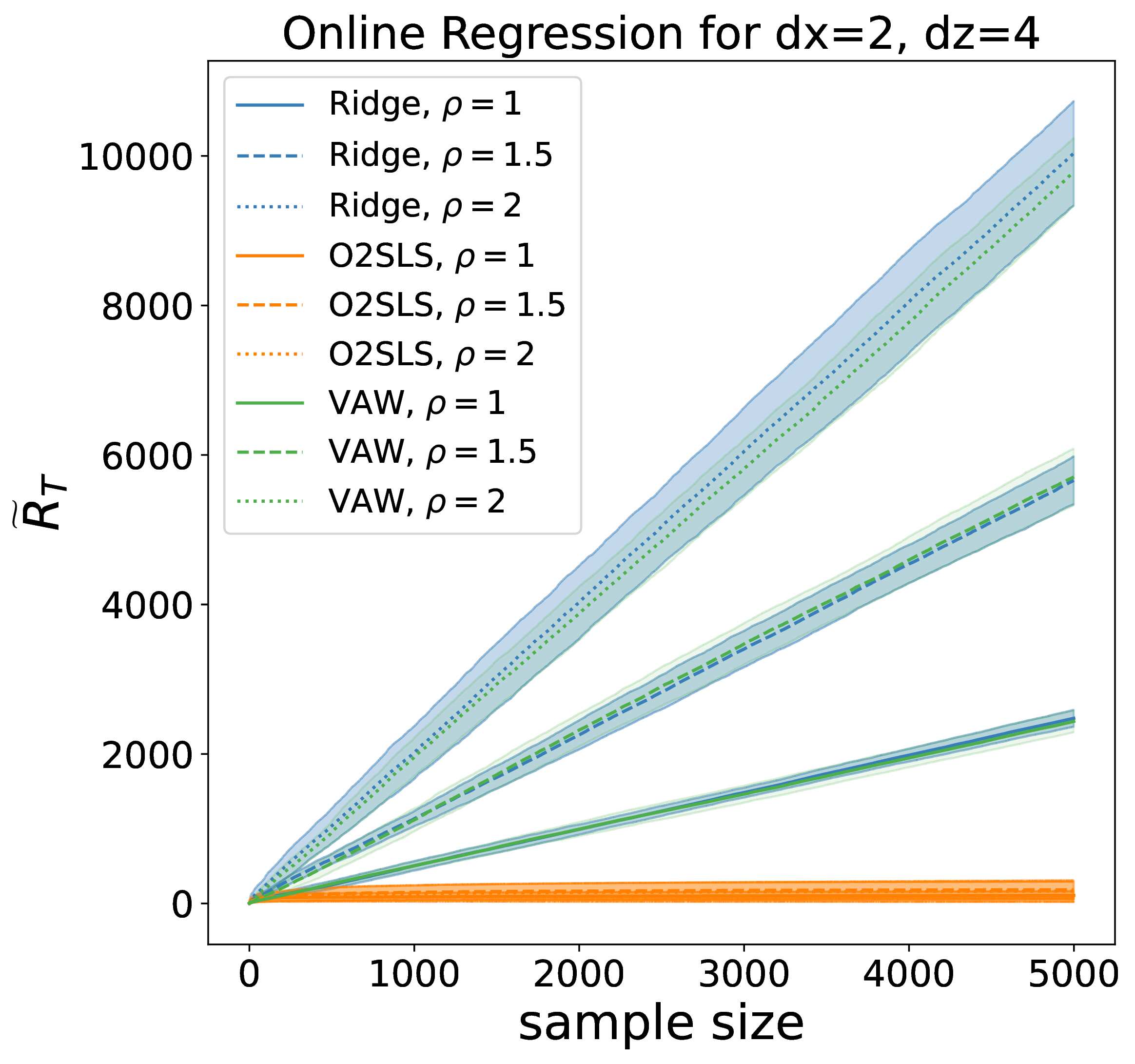}
    \caption{Cum. reg., $d_{\bx}=2, d_{\bz}=4$}
  \label{fig:3a}
\end{subfigure}
\hfill
\begin{subfigure}[b]{.3\columnwidth}
    \includegraphics[width=\columnwidth]{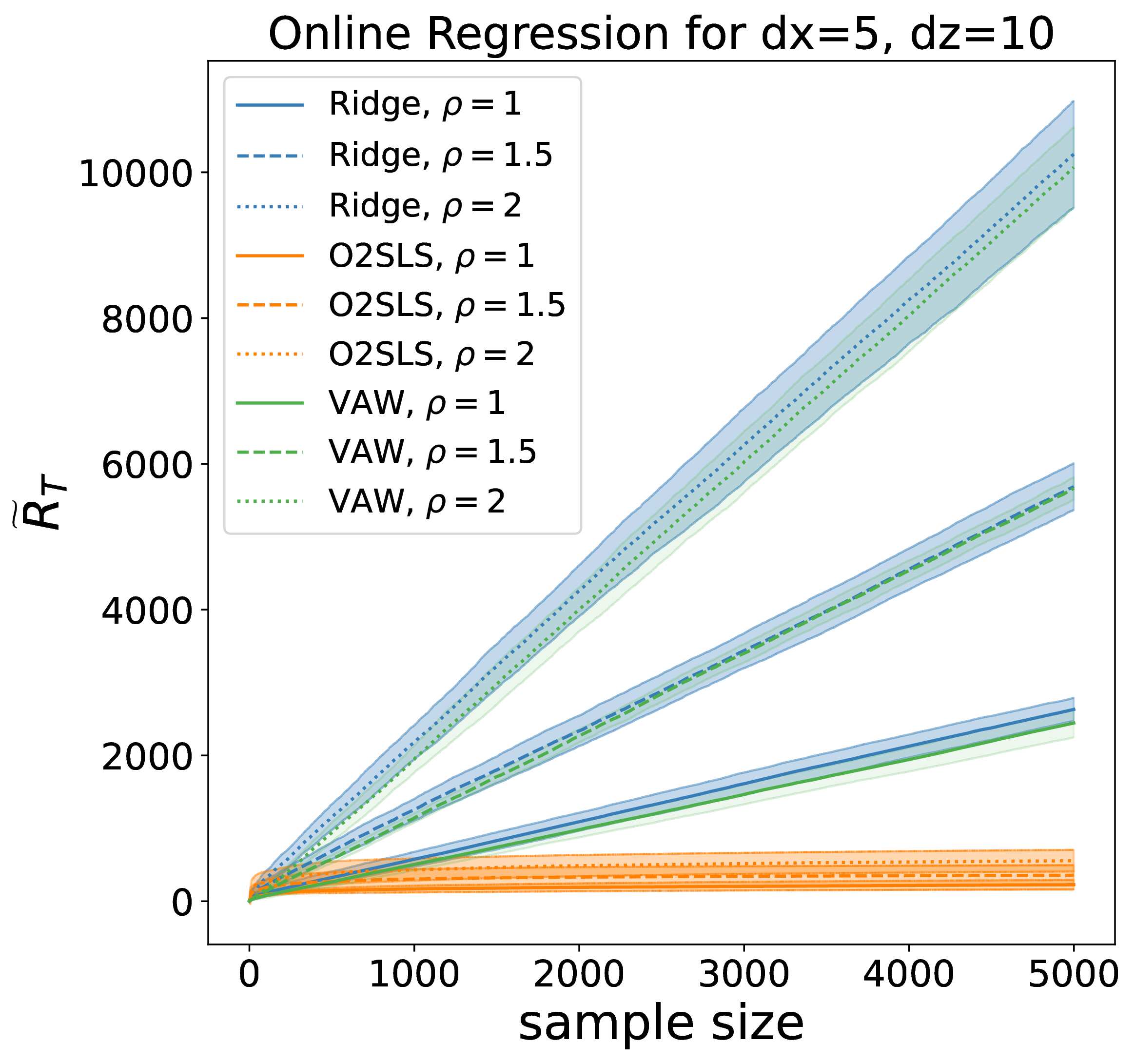}
    \caption{Cum. reg., $d_{\bx}=5, d_{\bz}=10$}
  \label{fig:3b}
\end{subfigure}%
\hfill
\begin{subfigure}[b]{.3\columnwidth}
    \includegraphics[width=\textwidth]{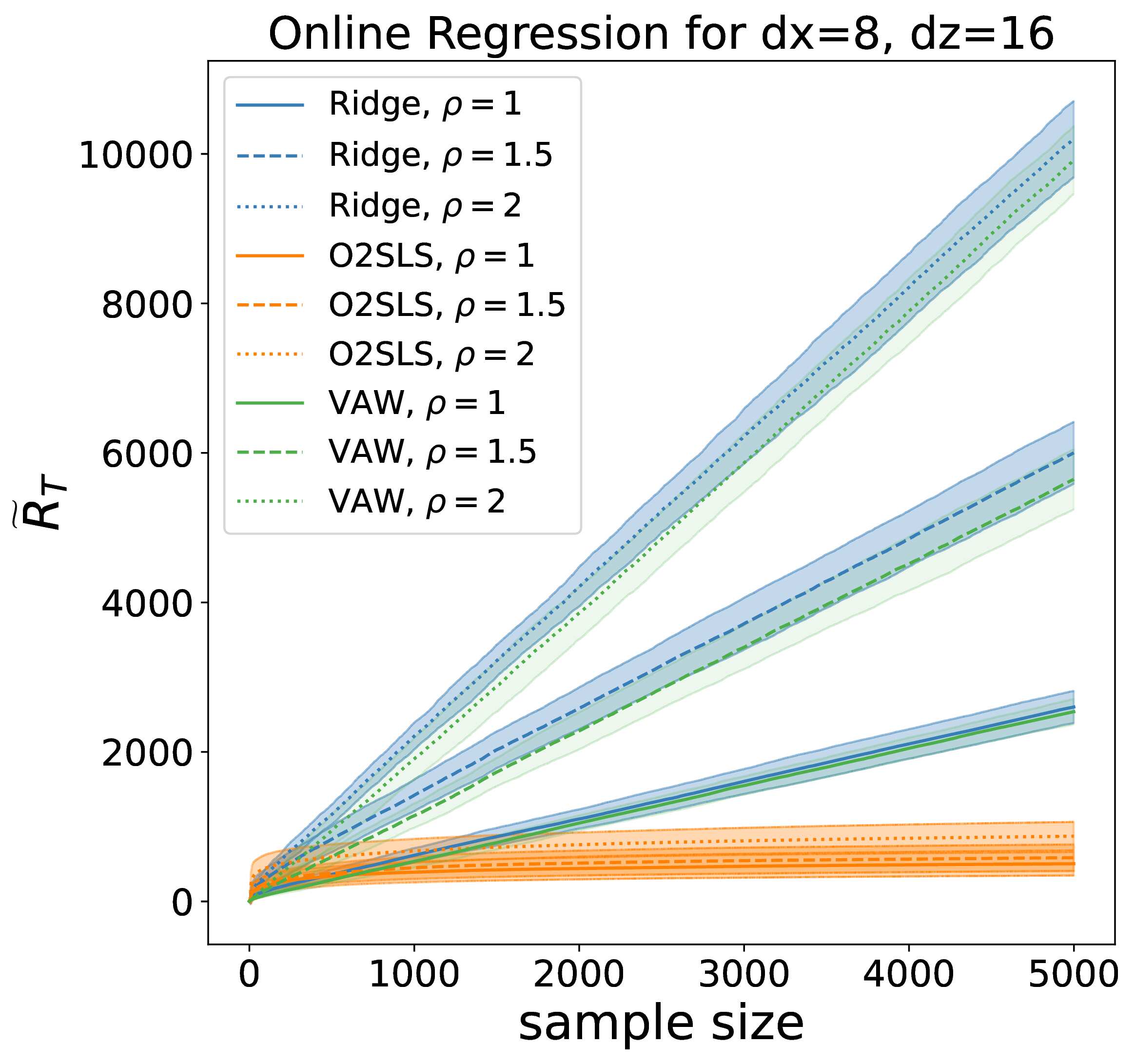}
    \caption{Cum. reg., $d_{\bx}=8, d_{\bz}=16$}
  \label{fig:3c}
\end{subfigure}%
\caption{Cumulative identification regret for an online regression setting of \otsls{}, \ridge{} and \vaw{} for different endogeneity levels and covariates' dimension.}\label{f:rcr}
\end{figure}

\begin{figure}[ht]
\centering
\begin{subfigure}[b]{.3\columnwidth}
    \includegraphics[width=\textwidth]{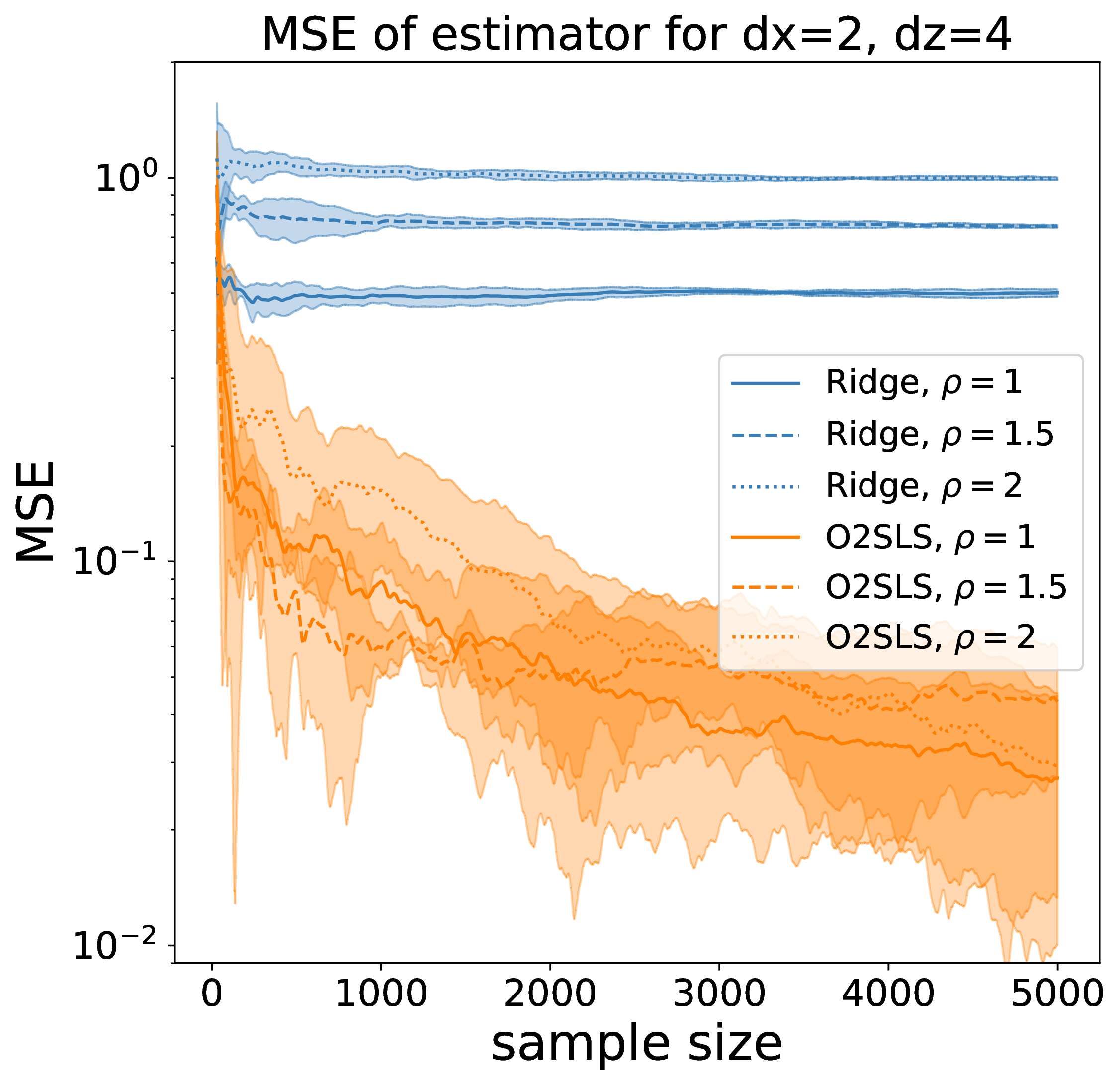}
    \caption{MSE, $d_{\bx}=2, d_{\bz}=4$}
  \label{fig:4a}
\end{subfigure}
\hfill
\begin{subfigure}[b]{.3\columnwidth}
    \includegraphics[width=\columnwidth]{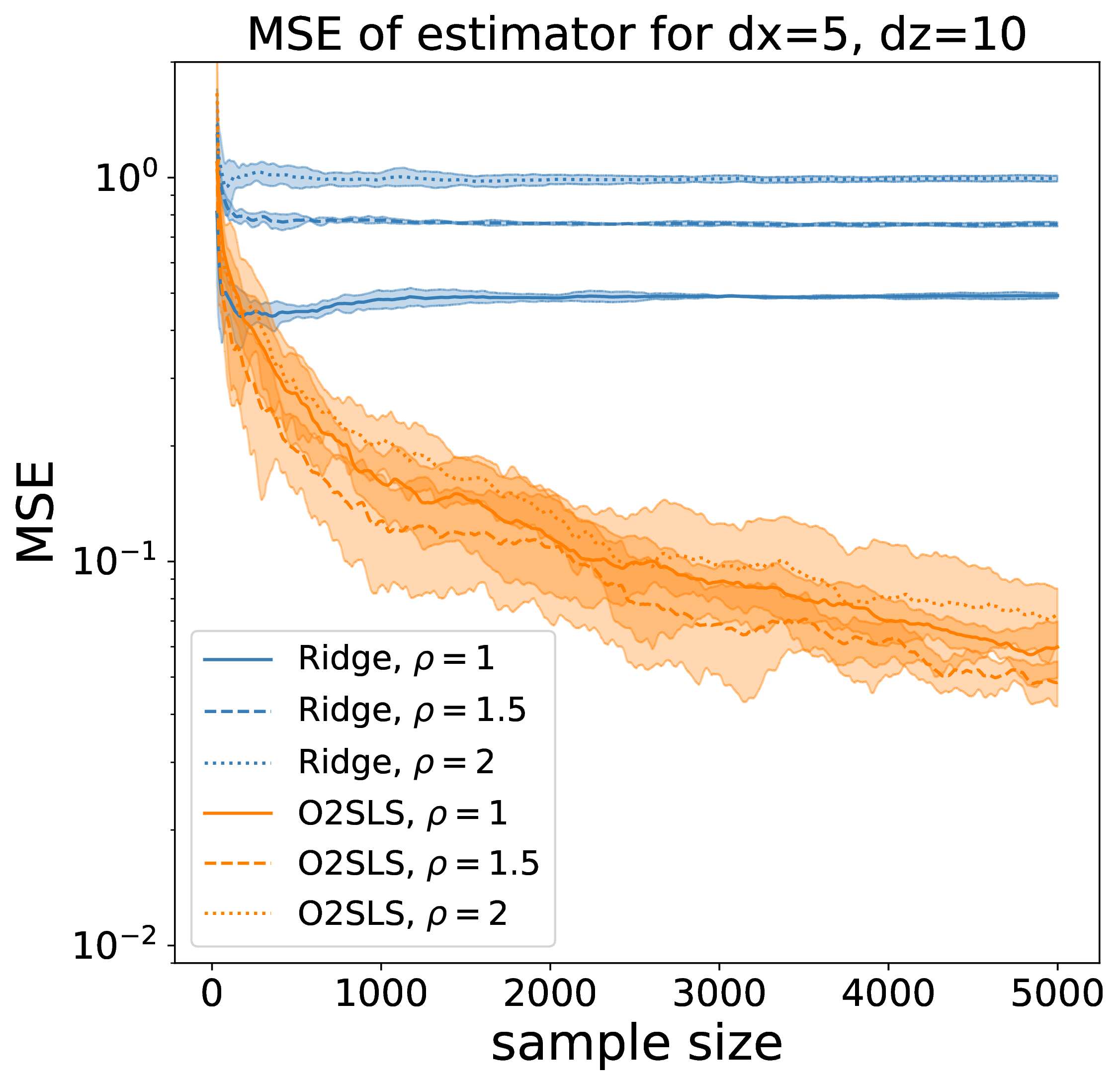}
    \caption{MSE, $d_{\bx}=5, d_{\bz}=10$}
  \label{fig:4b}
\end{subfigure}%
\hfill
\begin{subfigure}[b]{.3\columnwidth}
    \includegraphics[width=\textwidth]{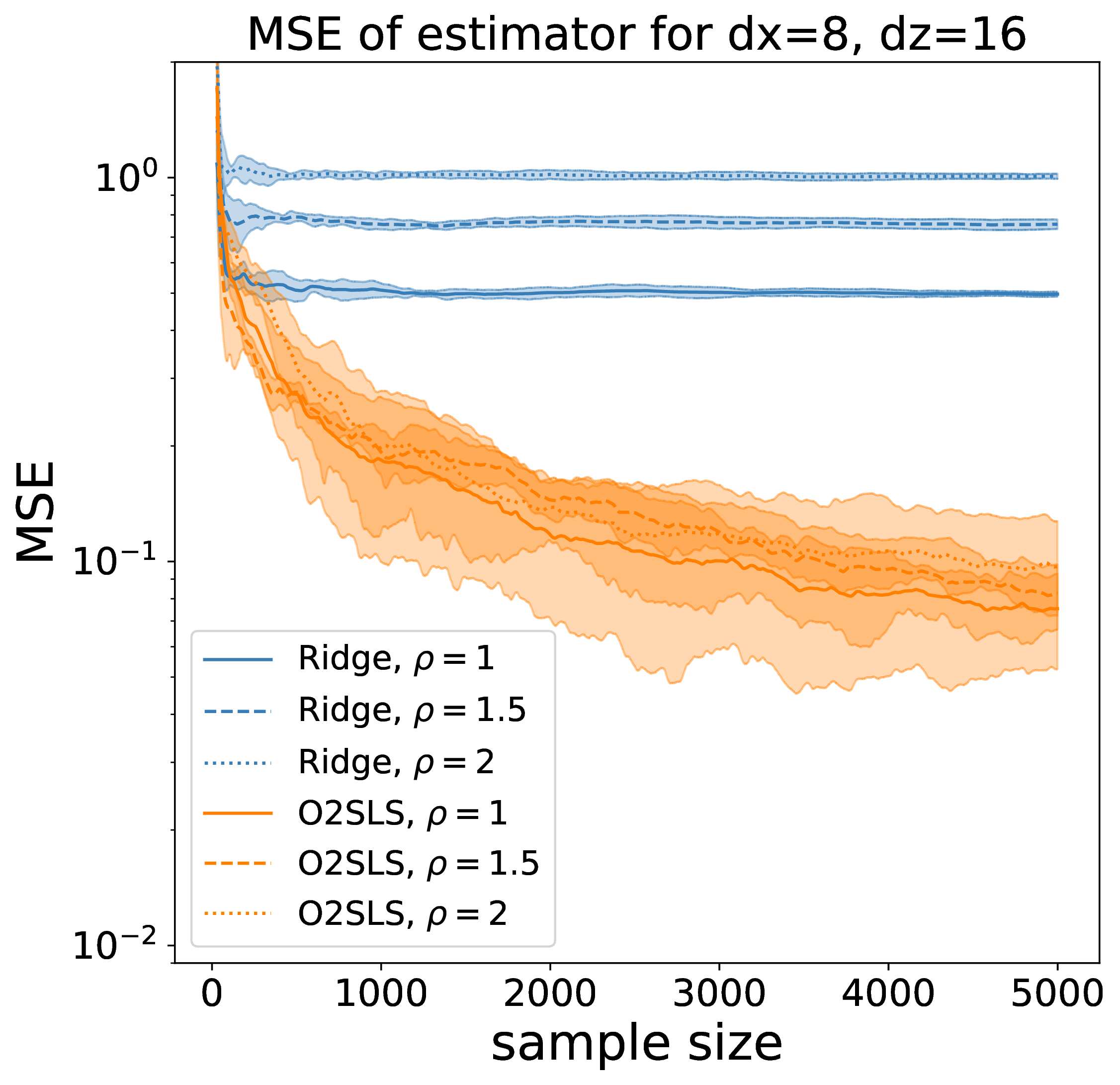}
    \caption{MSE, $d_{\bx}=8, d_{\bz}=16$}
  \label{fig:4c}
\end{subfigure}%
\caption{MSE $=\norm{\bb_t-\bb}_2$ of \oful{} and \ofuliv{} for different endogeneity levels and covariates' dimension.}\label{f:rmse}
\end{figure}

\textbf{Results.} By analysing Figure~\ref{f:rcr} and \Cref{f:rmse}, we observe that when endogeneity is present, both \ridge{} and \vaw{} exhibit similar performances, which are noticeably inferior to that of \otsls{}. Figure~\ref{f:rmse} also shows that the performance gap between \ridge{} and \otsls{} is approximately one order of magnitude in terms of the mean square error obtained by the respective estimators with respect to the true value of the parameter $\bb$. \otsls's gain in performance is reflected in the identification regret.  \textit{\otsls's performance gets increasingly better while compared to \ridge{} and \vaw{} for increasing values of endogeneity ($\rho$)} as shown in Figure~\ref{f:rcr}.

\subsection{Linear bandits with endogeneity on synthetic data}\label{app:experiments_lbe}

Finally, we compare the performance of \ofuliv{} and \oful{}~\citep{abbasi2011improved} for the LBE setting. \oful{} builds a confidence ellipsoid centered at $\bb^{\ridge}_{t}$ to concentrate around $\bb$, while \ofuliv{} uses $\bb^{\otsls{}}_{t}$ to build an accurate estimate of the true $\bb$.
\red{
Our experiments demonstrate the superiority of our instrumental variables approach, which can be used to boost the performances in a linear bandit problem plagued by endogeneity when valid instruments are available. Our approach is superior to the standard \oful{} or even a clever one-stage OFUL reduction.
We beat the previous baselines in terms of regret performances across a wide range of covariate dimensions and endogeneity levels. 
Furthermore, we also experimentally confirm that using the \otsls{} estimators leads to a linear bandit where performances do not depend on an upper bound to the hidden parameters $\bb$ and $\bt$ that we would like to estimate like for \oful{} and one-stage OFUL, where, knowledge of such upper bounds is something unavoidable. This last fact is an additional interesting and unexpected property of our algorithm.

}

\noindent\textbf{Experimental setup.} We adopt the following experimental setup and data generation model.
\begin{align*}
    \bx_{t,a} 
&= 
    \bt \bz_{t,a} +
    \bep_t  \tag{First Stage}
\\
    y_t
&= 
    \bb^\top \bx_{t,A_t} + \eta_t 
=
    \bb^\top \bx_{t,A_t} + (\rho \cdot \epsilon_{t,1}+\widetilde\eta_t) 
    \tag{Second Stage}
\end{align*} 
Here, $\epsilon_{t,1}$ indicates the first component of the vector $\bep_t$, and $\eta_{t} = \rho \cdot \epsilon_{t,1}+\widetilde\eta_{t}$.
We choose $d_{\bx}=\{2,5,8\}$ and $d_{\bz}=\{4,10,16\}$ respectively. In our experiments, we choose arbitrarily $\bb$ as a normalised vector with equal negative entries. 
Therefore, the values in the components are uniquely determined by the dimension $d_{\bx}$.
We choose $\bt=\I_{d_{\bz},d_{\bx}}$ which has ones on  the entries $i=j$ and zeros for $i\neq j$.
Then, we sample at each time $t$ and for every arm $a$ the vectors $\bz_{t,a} \sim \mathcal N_{d_{\bz}}(\vec{0}, \I_{d_{\bz}})$, for each time $t$ the vector noise   $\bep_{t} \sim \mathcal N_{d_{\bx}}(\vec{0}, \I_{d_{\bx}})$, and the scalar noise $\eta_{t} = \widetilde\eta_{t} + \rho \cdot \epsilon_{t,1}$ where $\widetilde\eta_{t} \sim \mathcal{N}_1(0,1)$. We run the algorithms with the same regularisation parameters equal to $\lambda=10^{-1}$.
We repeat our experiments 20 times. 
We average the results, and for each algorithm, we report the mean and standard deviation of the cumulative regret (shaded areas correspond to one standard deviation).
In \Cref{f:bcr} we present the learning curves for the regret $R_T$, while in \Cref{f:bmse} we plot the  $MSE=\norm{\bb_t-\bb}_2$ which shows the estimation error of the two. In addition, we compare them for increasing values of endogeneity, i.e. $\rho$.

\begin{figure}[H]
\centering
\begin{subfigure}[b]{.3\columnwidth}
    \includegraphics[width=\textwidth]{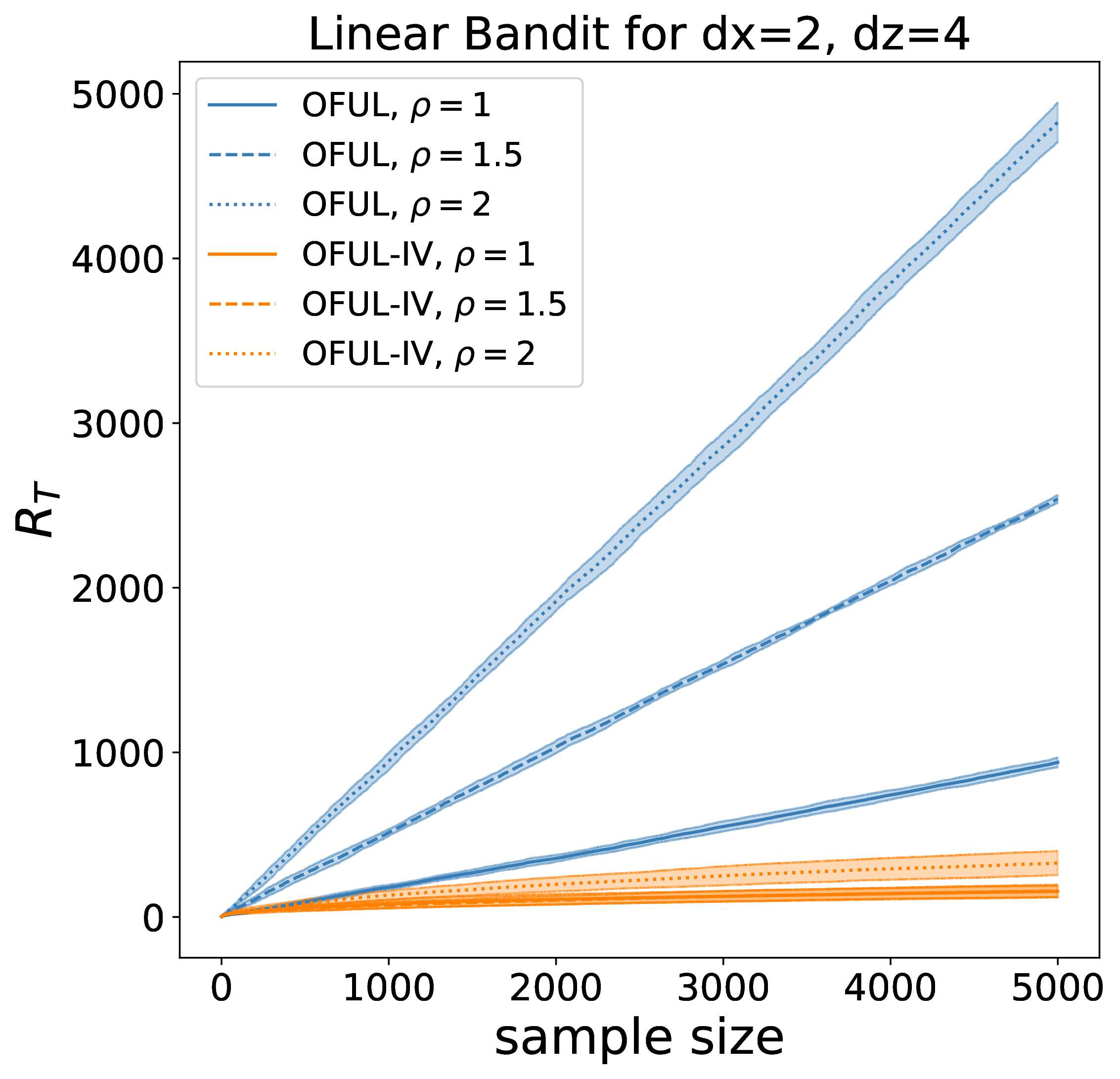}
    \caption{Cum. reg., $d_{\bx}=2, d_{\bz}=4$}
  \label{fig:1a}
\end{subfigure}
\hfill
\begin{subfigure}[b]{.3\columnwidth}
    \includegraphics[width=\columnwidth]{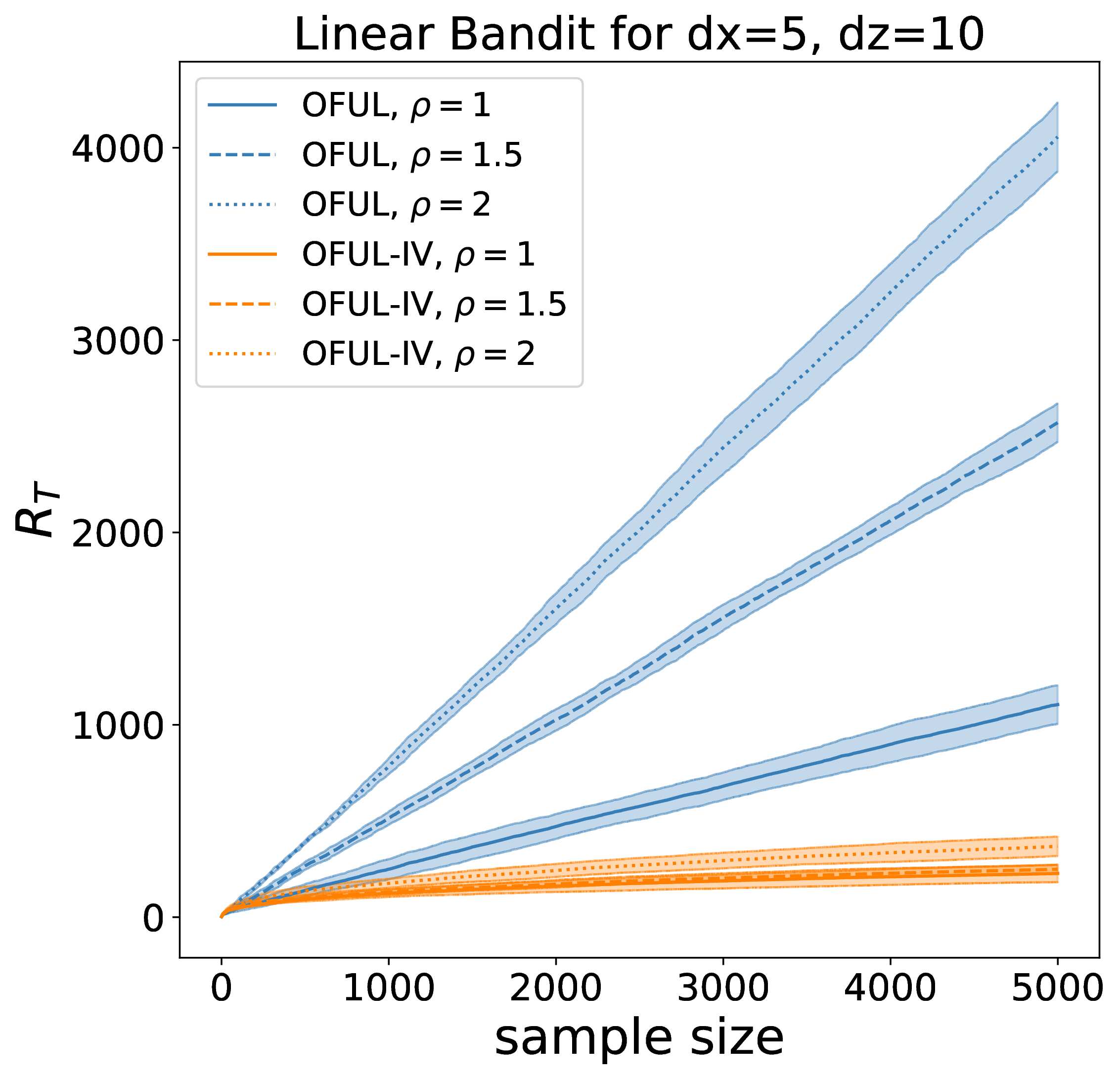}
    \caption{Cum. reg., $d_{\bx}=5, d_{\bz}=10$}
  \label{fig:1b}
\end{subfigure}%
\hfill
\begin{subfigure}[b]{.3\columnwidth}
    \includegraphics[width=\textwidth]{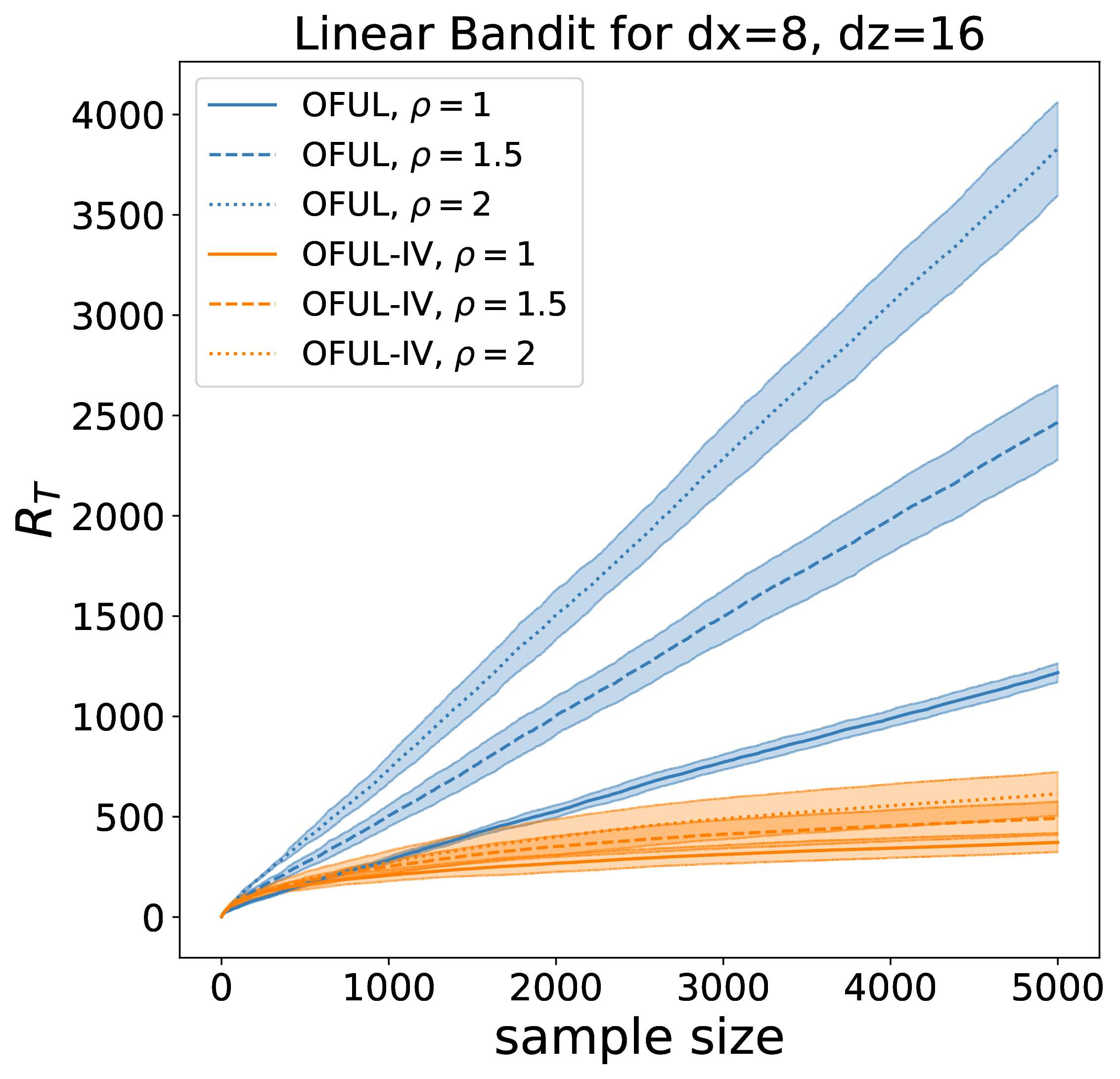}
    \caption{Cum. reg., $d_{\bx}=8, d_{\bz}=16$}
  \label{fig:1c}
\end{subfigure}%
\caption{Cumulative regret of \oful{} and \ofuliv{} for different endogeneity levels and covariates' dimension.}\label{f:bcr}
\end{figure}

\begin{figure}[H]
\centering
\begin{subfigure}[b]{.3\columnwidth}
    \includegraphics[width=\textwidth]{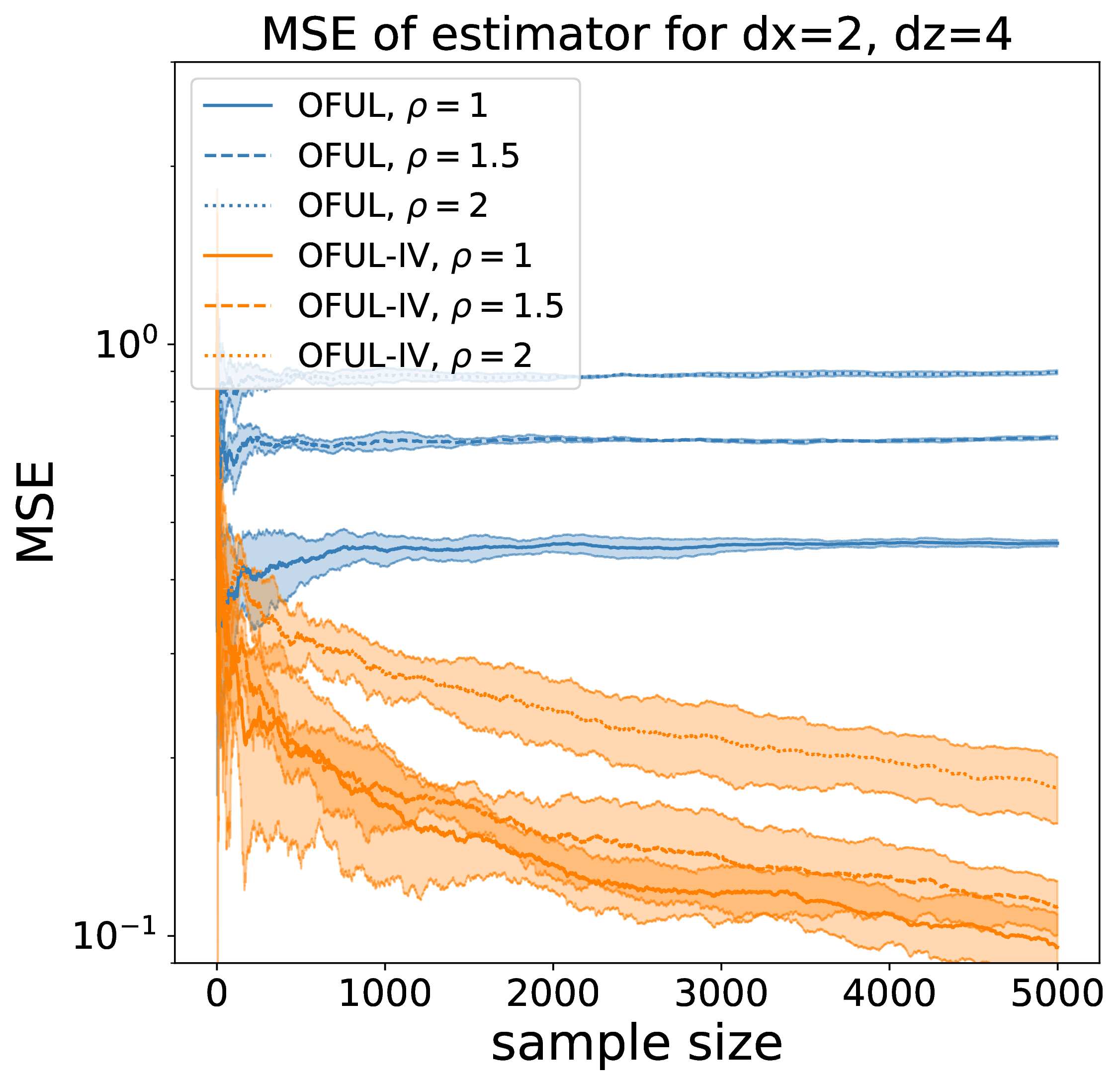}
    \caption{MSE, $d_{\bx}=2, d_{\bz}=4$}
  \label{fig:2a}
\end{subfigure}
\hfill
\begin{subfigure}[b]{.3\columnwidth}
    \includegraphics[width=\columnwidth]{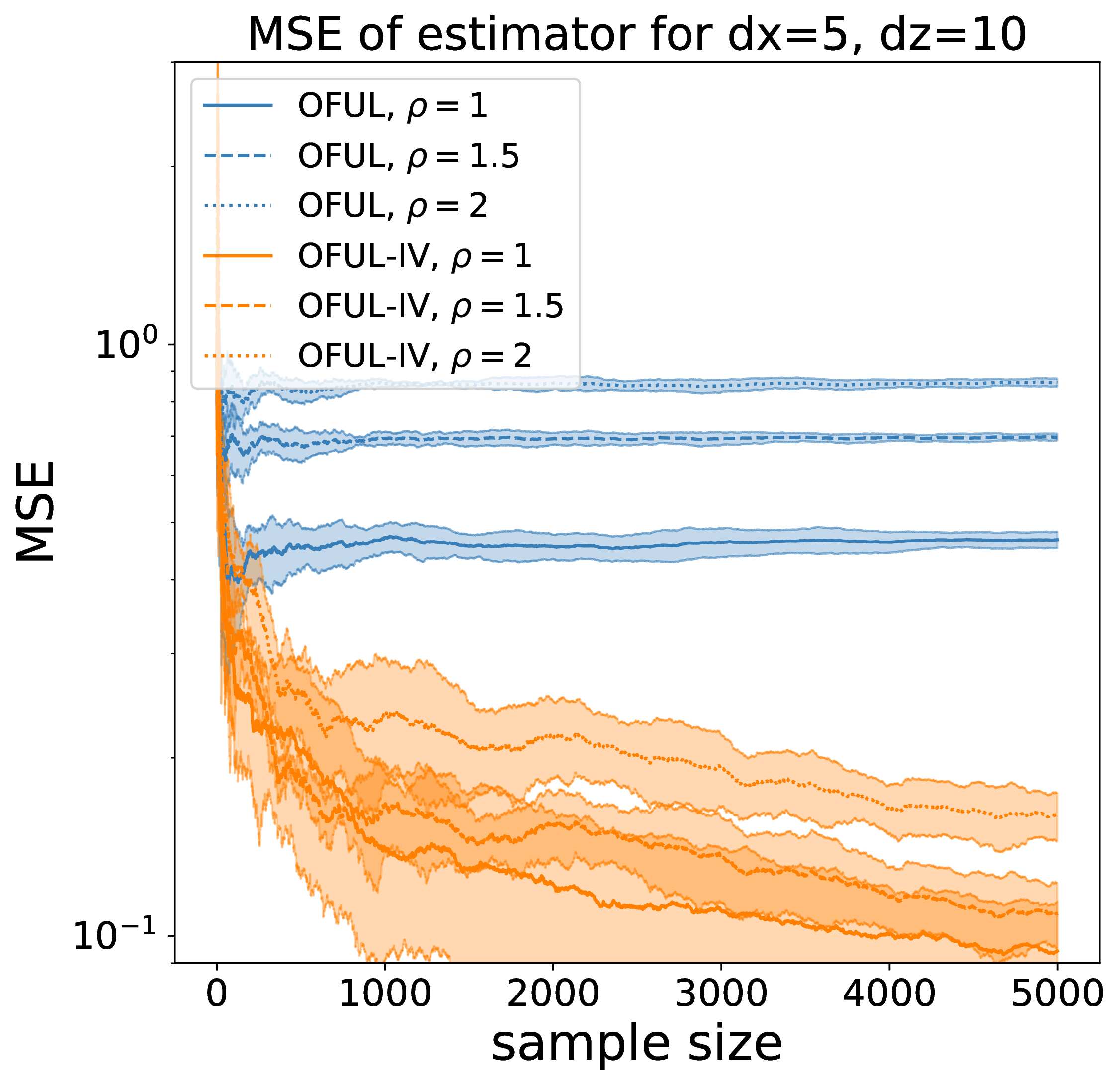}
    \caption{MSE, $d_{\bx}=5, d_{\bz}=10$}
  \label{fig:2b}
\end{subfigure}%
\hfill
\begin{subfigure}[b]{.3\columnwidth}
    \includegraphics[width=\textwidth]{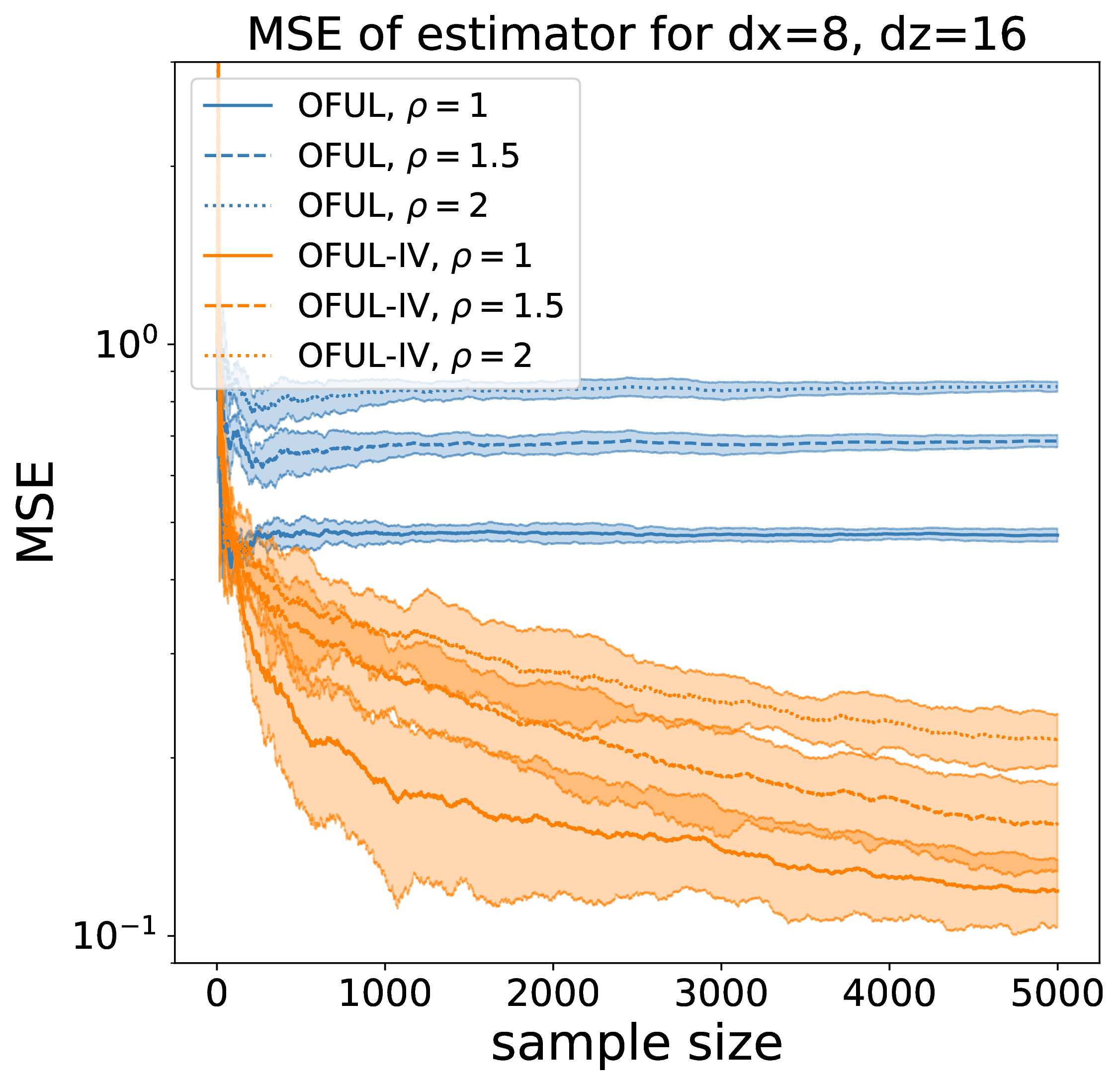}
    \caption{MSE, $d_{\bx}=8, d_{\bz}=16$}
  \label{fig:2c}
\end{subfigure}%
\caption{
MSE $=\norm{\bb_t-\bb}_2$ of \oful{} and \ofuliv{} for different endogeneity levels and covariates' dimension.
}\label{f:bmse}
\end{figure}

\paragraph{Results.} 
The experiments show that for a wide range of dimensions of the covariates and for different values of endogeneity, the MSE of the estimator used by \ofuliv{} is almost 1-order smaller than those of \oful{} over a $T=5 000 $ rounds \Cref{f:bmse}. \textit{Thanks to a better estimate and better confidence intervals, \ofuliv{} has much lower cumulative regret than \oful{} for linear bandits with endogeneity} as illustrated in \Cref{f:bcr}.

\subsubsection{Unavoidable dependence on the norms of hidden parameters for \oful{} and one-stage \oful{} } 

\red{

In this section, we compare \oful{}, one-stage \oful{}, and \ofuliv{} for different endogeneity levels $\rho$ and values of the parameter $S$ defined as  $S\triangleq \norm{\bb}_2=\norm{\bt \bb}_2$, where the equalities follow from the choice  $\bb= -S \cdot \Vec{1}_{d_{\bx}}/\norm{\Vec{1}_{d_{\bx}}}_2 $ and $\bt=\I_{d_{\bz},d_{\bx}}$ in the setting of Appendix~\ref{app:experiments_lbe}. 
We fix the dimension of the covariates to $d_{\bx}=2$ and $d_{\bz}=4$, but similar results hold in general for the over-identified case  $d_{\bz}\geq d_{\bx}$.

 One-stage \oful{} reduction is defined in Section~\ref{sec:ofuliv}, by $y_t =(\bt\bb)^{\top} \bz_t+ \bb^\top \bep_t +\eta_t$. The average performance of one-stage \oful{} is dependent on $\bb$, which directly goes into the noise variance $\sigma_{new}^2 = 2(\|\bb\|_2^2 \sigma_{\bep}^2 + \sigma_{\eta}^2)$, and $\|\bt \bb\|_2$ goes into the index computation of \oful{} through the ellipsoid radius~\citep[Thoerem 2]{abbasi2011improved}.
  \textit{Favourable to \oful{} and one-stage \oful, we use knowledge of $S$ for them, which might be unavailable in reality}. In that case, we might use a larger but misspecified upper-bound as a proxy. \textit{We also do not use and need knowledge of $S$ for \ofuliv.} 
  
  We show how this dependency can cause results of \oful{} and one-stage \oful{} to deteriorate arbitrarily even when we know the tightest possible upper bound on $\|\bb\|_2$ and $\|\bb\bt\|_2$. In principle, the two norms could be different, but the choices of Appendix D of $\bb$ and $\bt$ help us to simplify the ablation study. Thus, we keep the two norms are the same and control them by just one parameter $S$. 
  
\paragraph{Results.}   From our numerical investigation, we observe that the average regret for OFUL-IV is insensitive to $S$.
  On the contrary, $S$ immediately affects one-stage \oful{}, as expected. Also,\textit{ in all the settings with different endogeneity levels $\rho$ and norms $S$, OFUL-IV incurs lower regrets than the others}.}

\begin{figure}[t!]
  \centering
  \begin{subfigure}[b]{0.25\textwidth}
    \includegraphics[width=\textwidth]{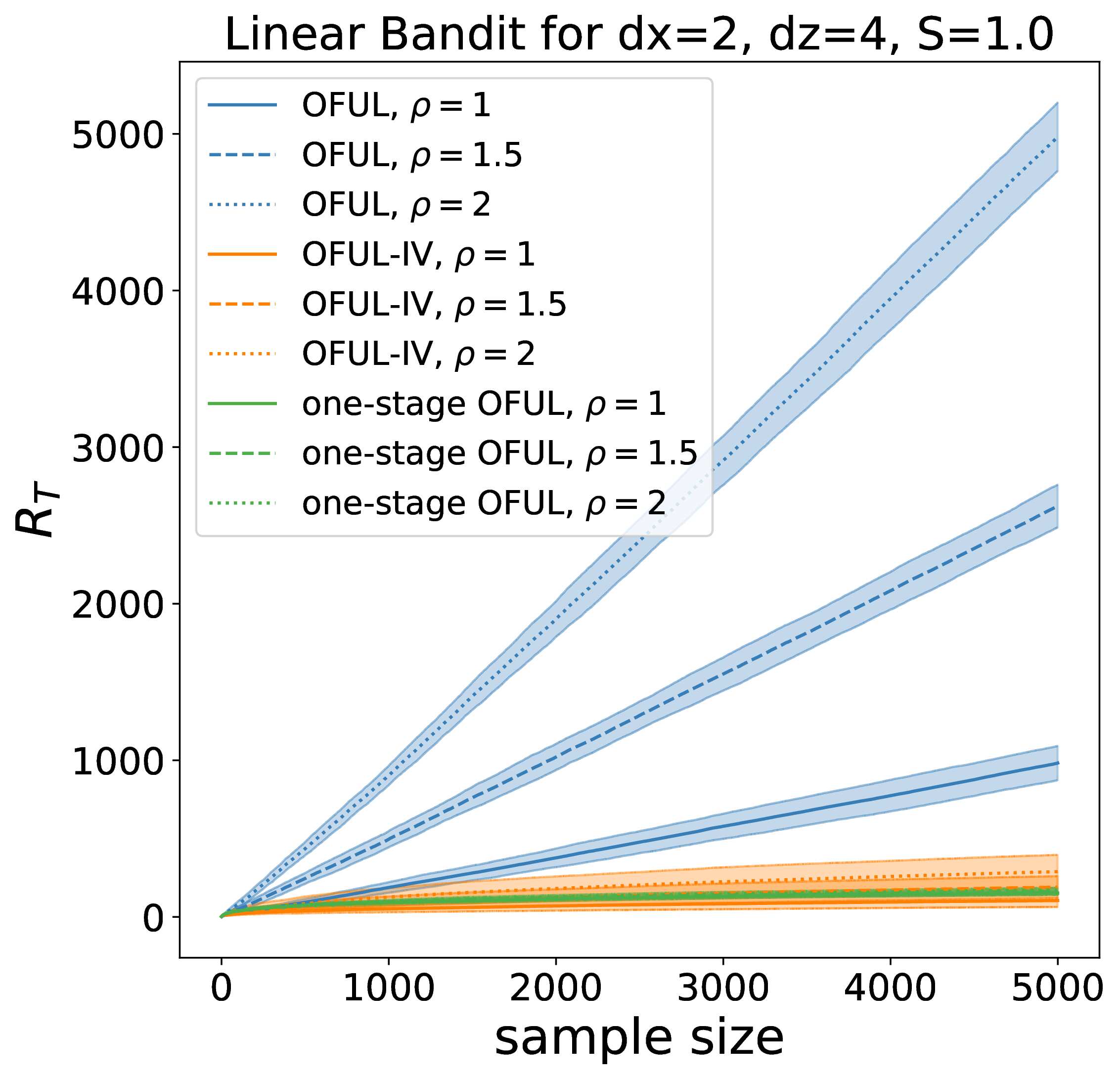}
    \caption{Cum. reg.: $S=1.0$}\label{fig:subfig1}
  \end{subfigure}
  \hfill
  \begin{subfigure}[b]{0.25\textwidth}
    \includegraphics[width=\textwidth]{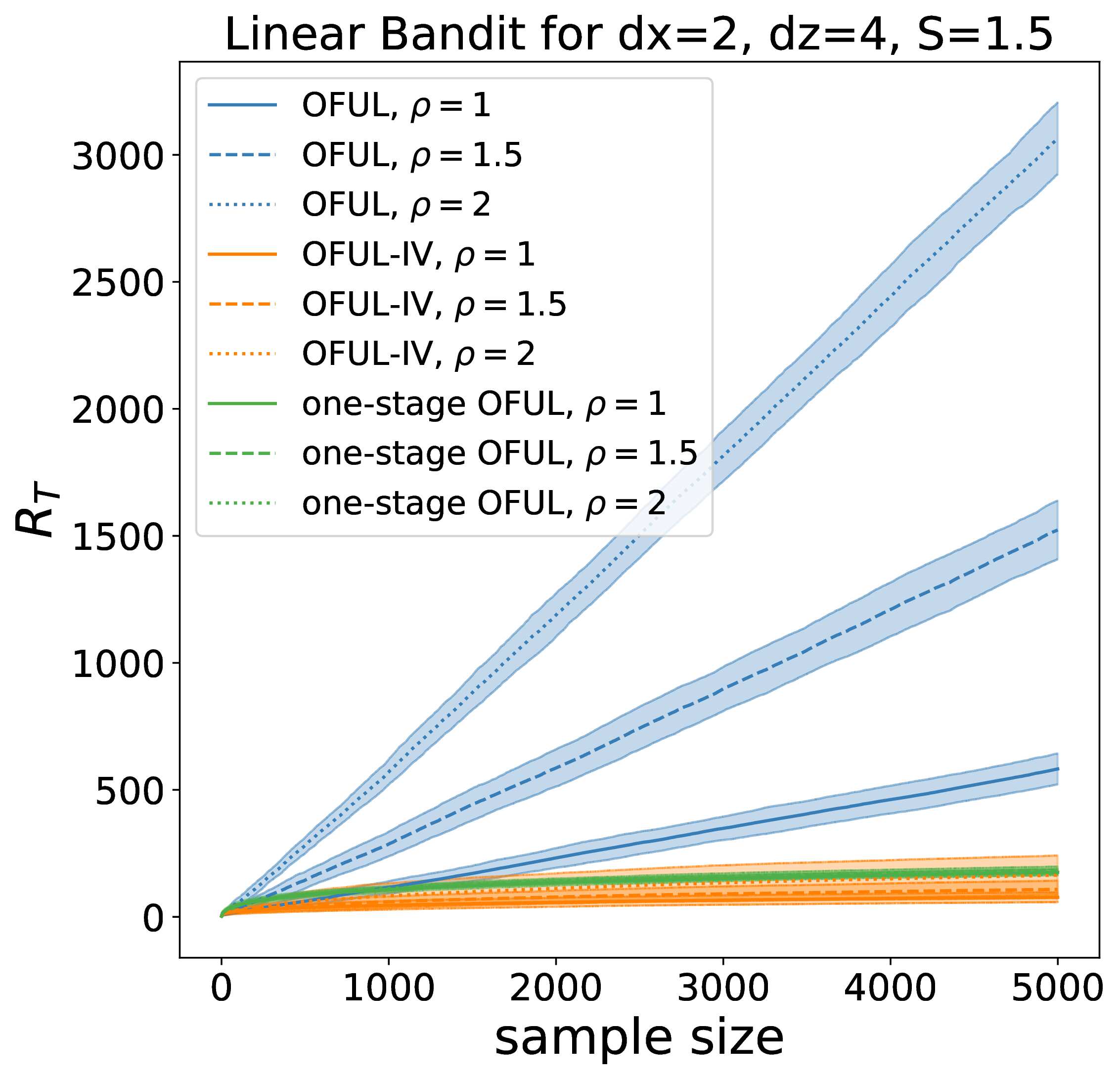}
    \caption{Cum. reg.: $S=1.5$}
    \label{fig:subfig2}
  \end{subfigure}
  \hfill
  \begin{subfigure}[b]{0.25\textwidth}
    \includegraphics[width=\textwidth]{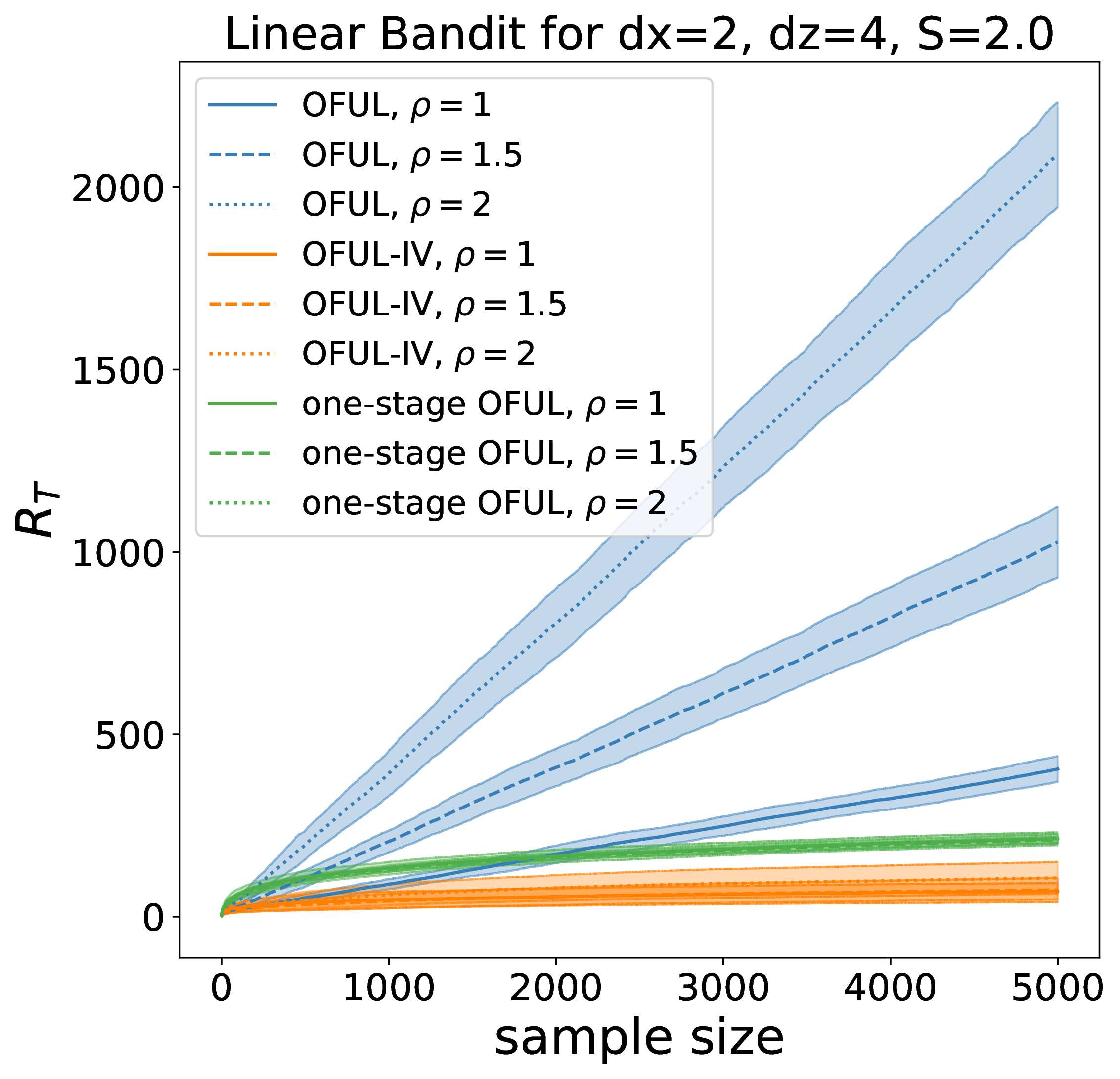}
    \caption{Cum. reg.: $S=2.0$}
    \label{fig:subfig3}
  \end{subfigure}
  
  \vspace*{0.5cm}
  
  \begin{subfigure}[b]{0.25\textwidth}
    \includegraphics[width=\textwidth]{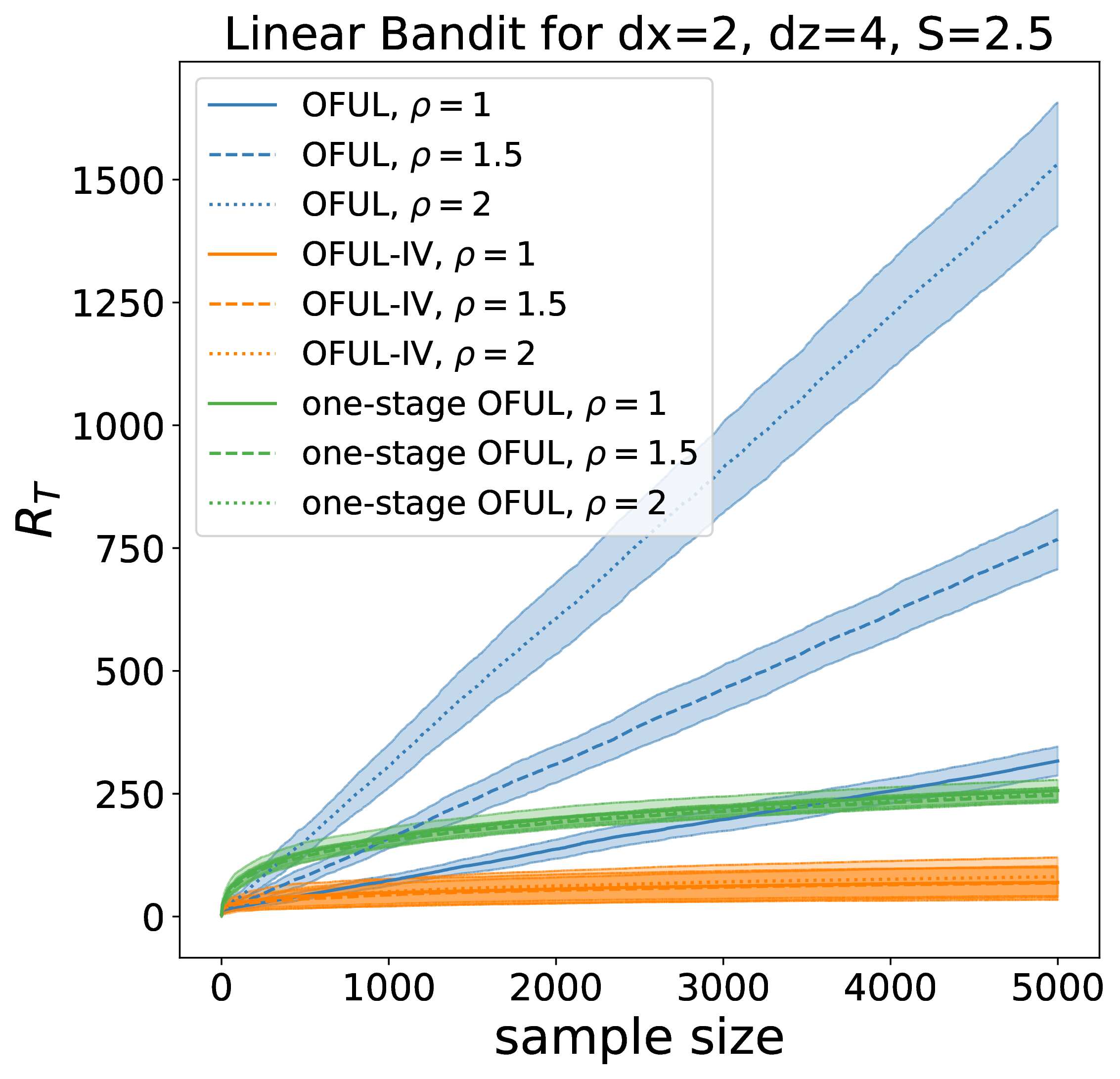}
    \caption{Cum. reg.: $S=2.5$}
    \label{fig:subfig4}
  \end{subfigure}
  \hfill
  \begin{subfigure}[b]{0.25\textwidth}
    \includegraphics[width=\textwidth]{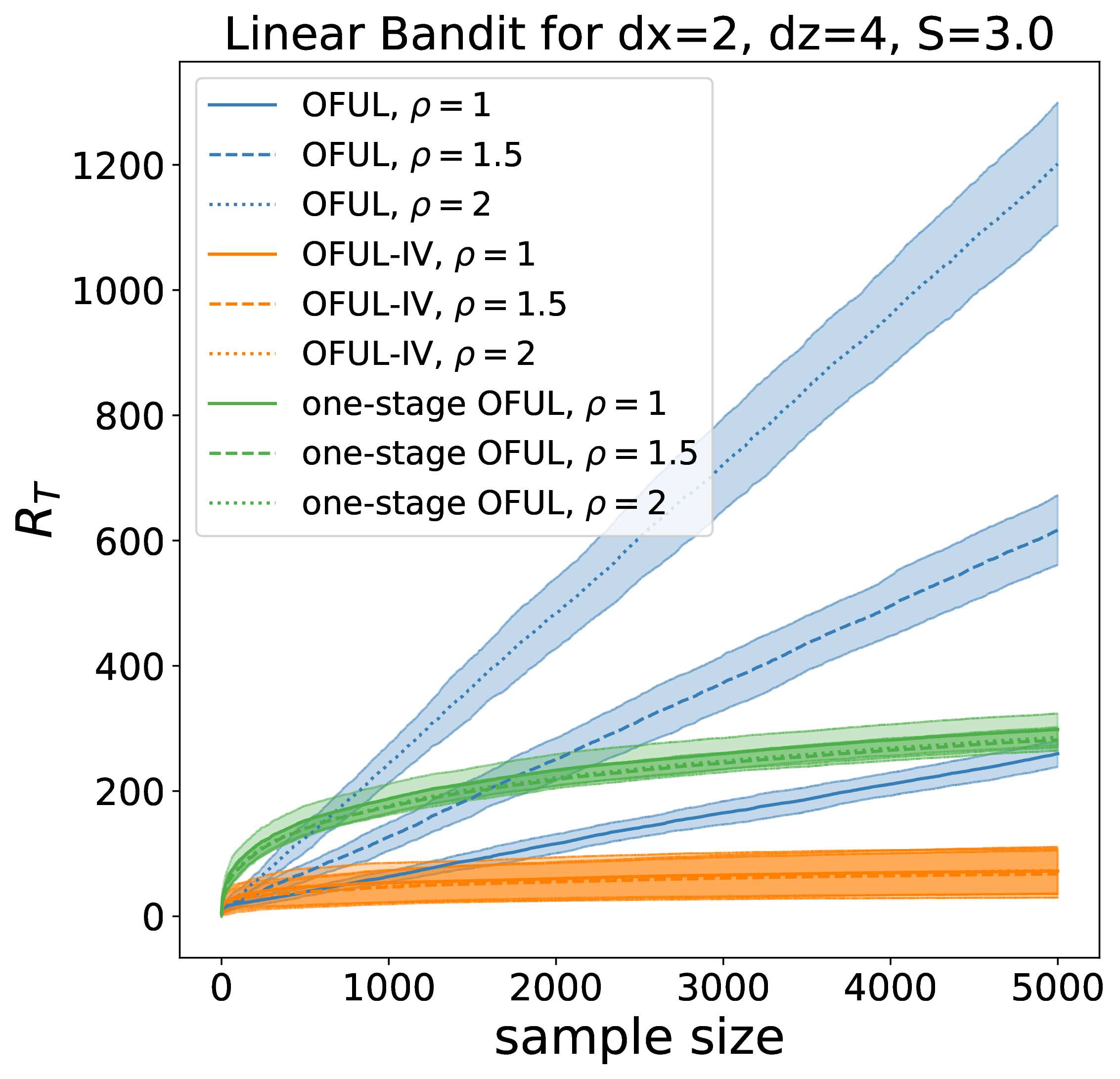}
    \caption{Cum. reg.: $S=3.0$}
    \label{fig:subfig5}
  \end{subfigure}
  \hfill
  \begin{subfigure}[b]{0.25\textwidth}
    \includegraphics[width=\textwidth]{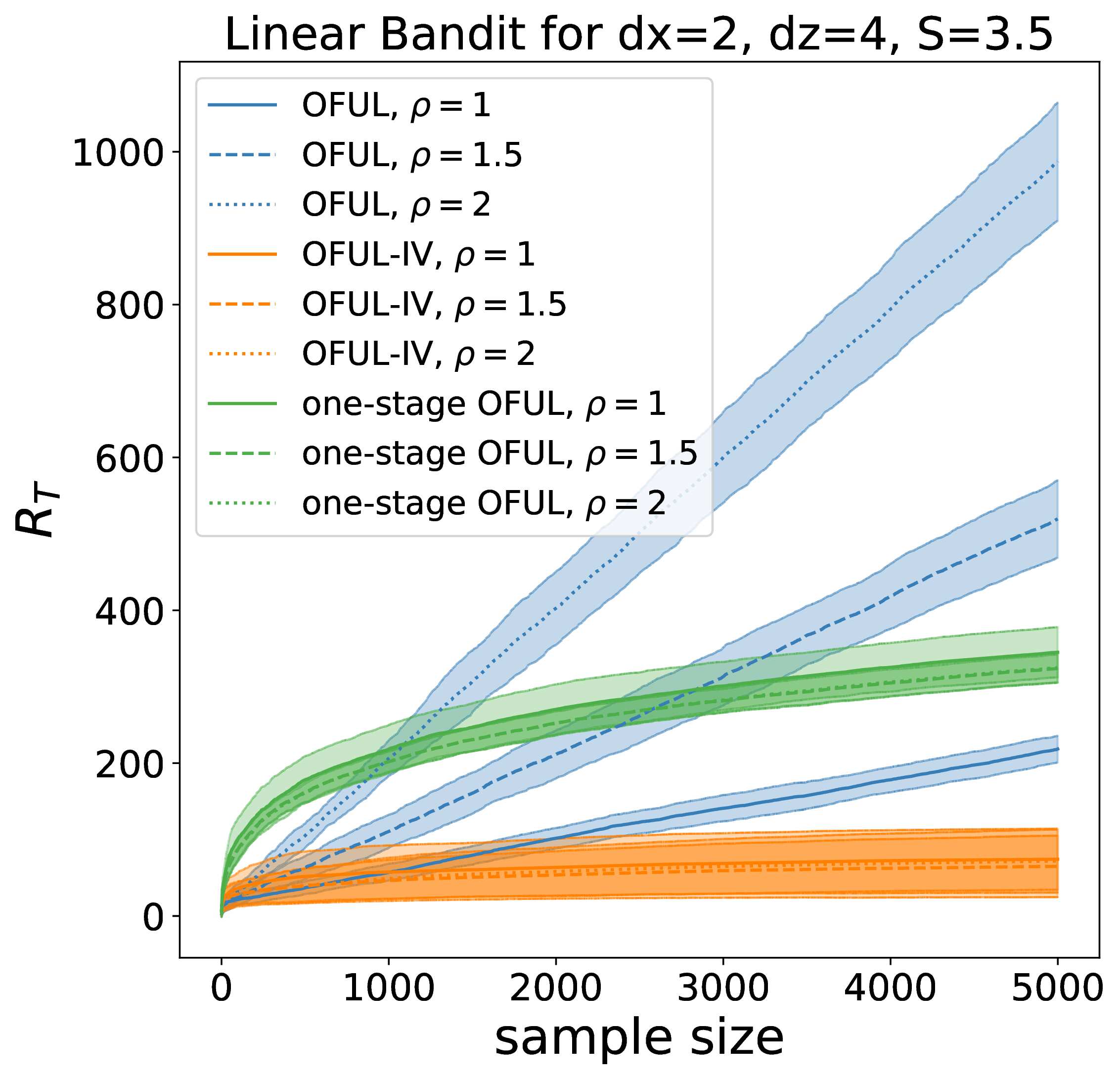}
    \caption{Cum. reg.: $S=3.5$}
    \label{fig:subfig6}
  \end{subfigure}
  
  \vspace*{0.5cm}
  
  \begin{subfigure}[b]{0.25\textwidth}
    \includegraphics[width=\textwidth]{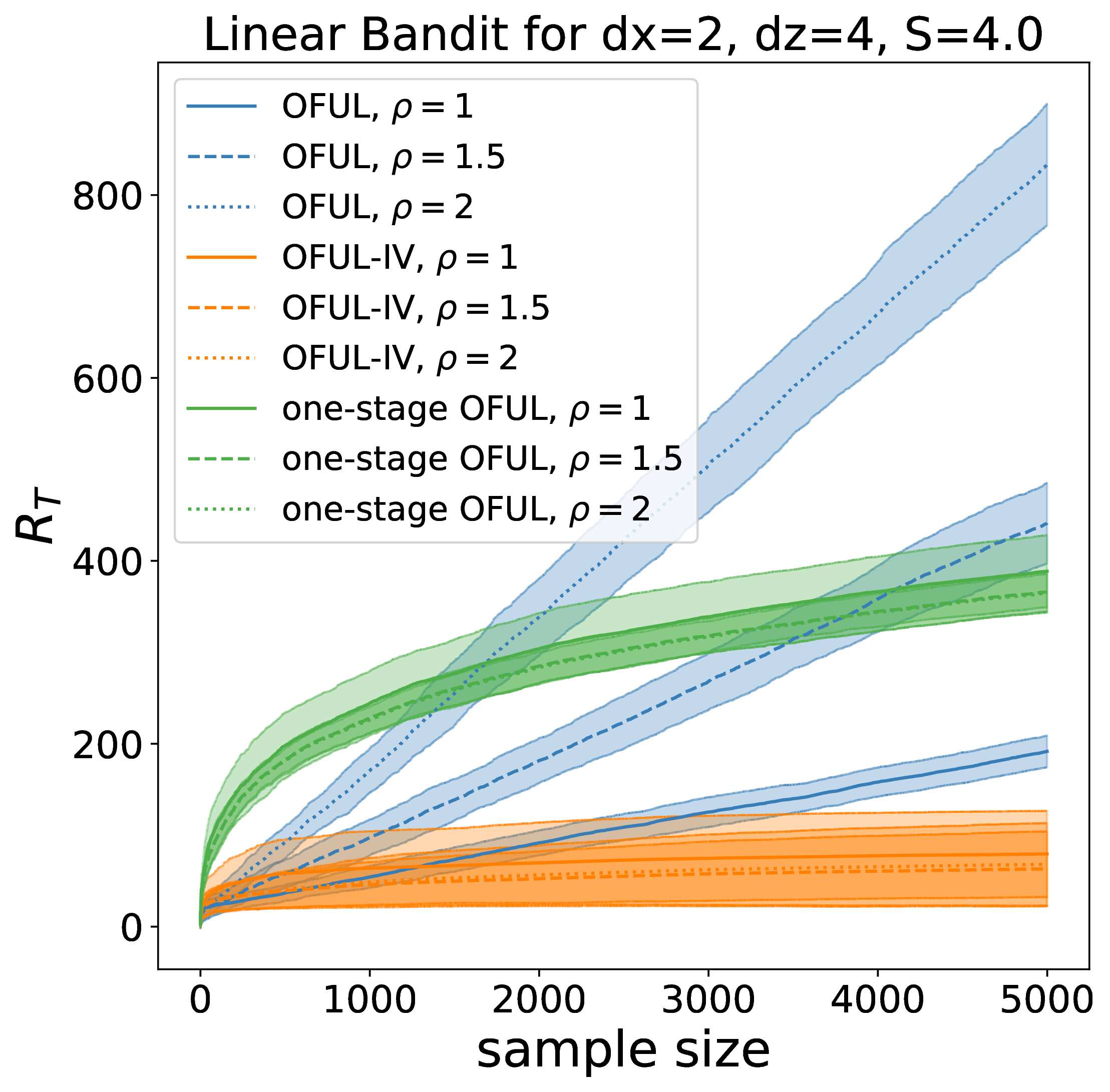}
    \caption{Cum. reg.: $S=4.0$}
    \label{fig:subfig7}
  \end{subfigure}
  \hfill
  \begin{subfigure}[b]{0.25\textwidth}
    \includegraphics[width=\textwidth]{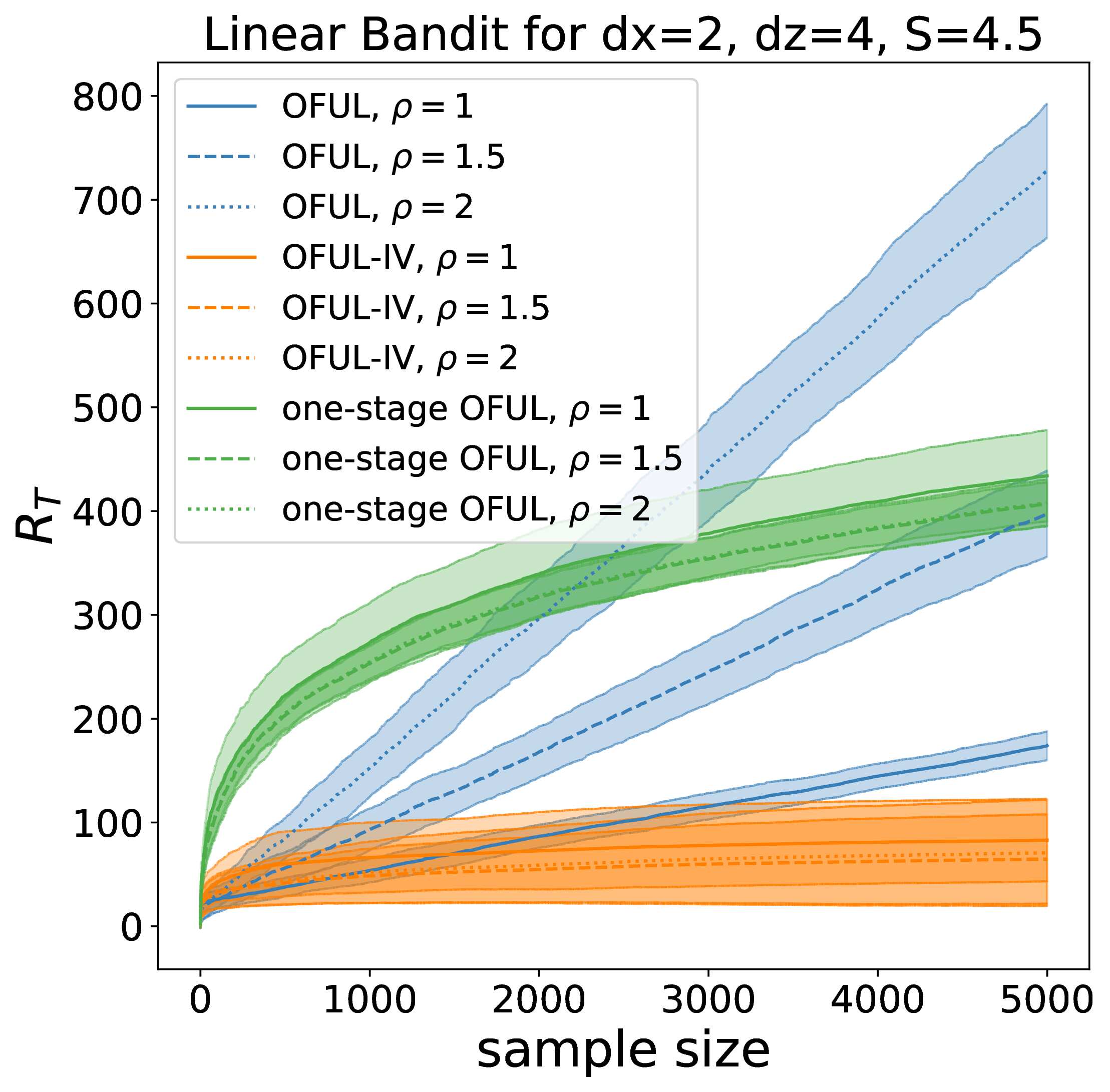}
    \caption{Cum. reg.: $S=4.5$}
    \label{fig:subfig8}
  \end{subfigure}
  \hfill
  \begin{subfigure}[b]{0.25\textwidth}
    \includegraphics[width=\textwidth]{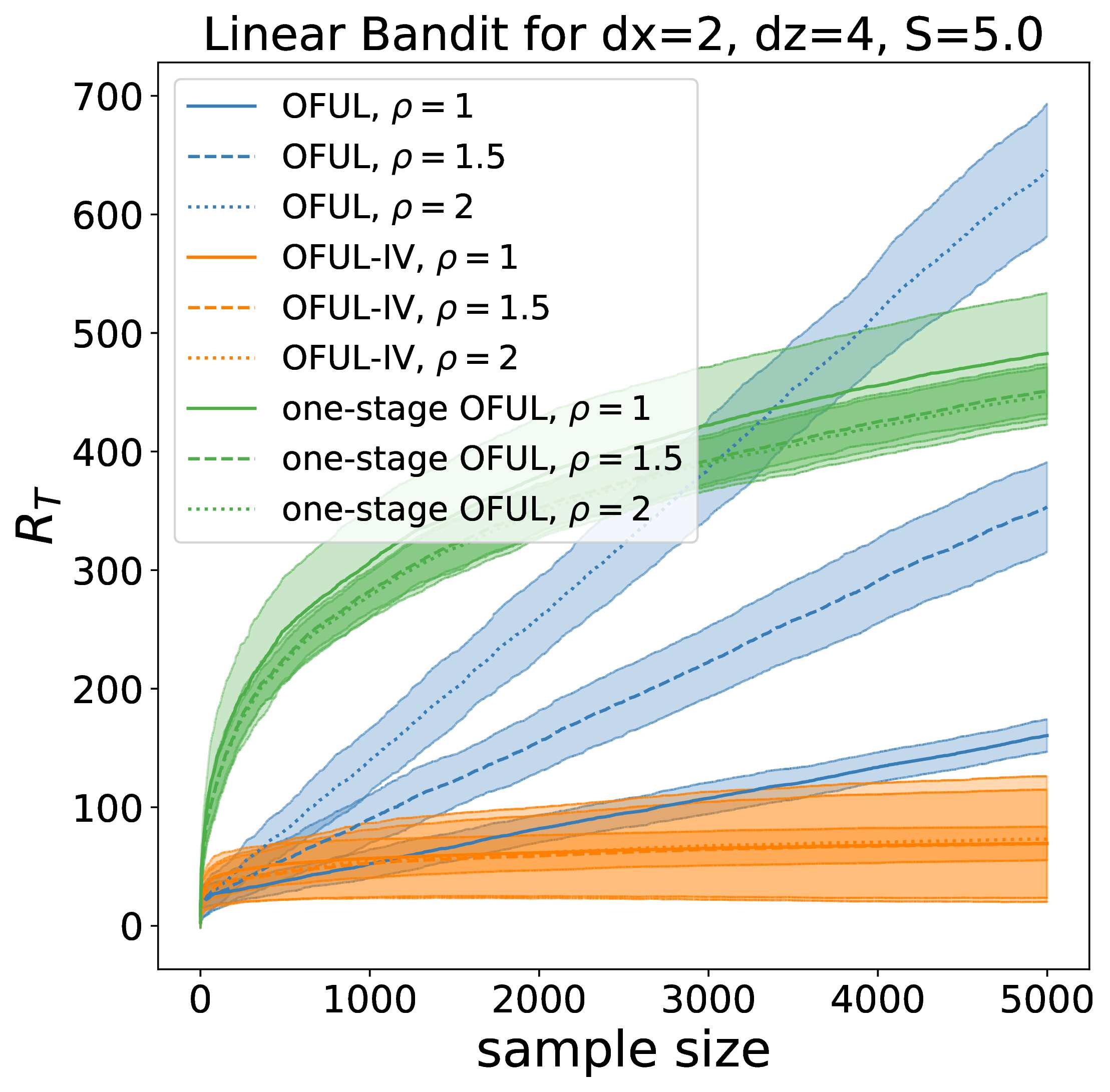}
    \caption{Cum. reg.: $S=5.0$}
    \label{fig:subfig9}
  \end{subfigure}
  \caption{\red{Comparing the evolution of cumulative regrets incurred by \oful{} (blue), \oful{} applied on one-stage reduction of LBE (green), and \ofuliv{} (orange) across different values of the upper bounds of $\bb$, i.e. $S$. For each $S$, we plot the mean and standard deviation of regrets of each of the algorithms across three levels of endogeneity $\rho=1.0, 1.5$, and $2.0$. While performance of \oful{} applied on one-stage reduction of LBE (green) deteriorates with increasing value of $S$, performance of \ofuliv{} is independent of $S$ and incurs the least cumulative regret for all the settings.}}
\end{figure}

% \begin{figure}[t!]
%   \centering
%   \begin{subfigure}[b]{0.25\textwidth}
%     \includegraphics[width=\textwidth]{FigR/bCR_dx2_dz4_S2.0.jpg}
%     \caption{Cumulative regret: $S=2.0$}
%     \label{fig:subfig3}
%   \end{subfigure}
%   \hfill
%   \begin{subfigure}[b]{0.25\textwidth}
%     \includegraphics[width=\textwidth]{FigR/bCR_dx2_dz4_S3.5.pdf}
%     \caption{Cumulative regret: $S=3.5$}
%     \label{fig:subfig6}
%   \end{subfigure}
%   \hfill
%   \begin{subfigure}[b]{0.25\textwidth}
%     \includegraphics[width=\textwidth]{FigR/bCR_dx2_dz4_S5.0.pdf}
%     \caption{Cumulative regret: $S=5.0$}
%     \label{fig:subfig9}
%   \end{subfigure}
%   \caption{\red{Comparing the evolution of cumulative regrets incurred by \oful{} (blue), \oful{} applied on one-stage reduction of LBE (green), and \ofuliv{} (orange) across different values of the upper bounds of $\bb$, i.e. $S$. For each $S$, we plot the mean and standard deviation of regrets of each of the algorithms across three levels of endogeneity $\rho=1.0, 1.5$, and $2.0$. While performance of \oful{} applied on one-stage reduction of LBE (green) deteriorates with increasing value of $S$, performance of \ofuliv{} is independent of $S$ and incurs the least cumulative regret for all the settings.}}
% \end{figure}

\newpage
\subsection{Price-Sales Dynamics (PSD)}

 In \Cref{s:intro}, we introduced in \Cref{example1} a compelling scenario where a market analyst aims to understand the impact of item prices ($\text{Price}_t$) on item sales ($\text{Sales}_t$), using a continuous stream of daily data. However, in this context, assuming exogeneity of price in the model is not viable due to the presence of unobserved confounding variables ($\text{Event}_t$), which remains unaccounted for in the regression model, resulting in an omitted variable. These confounders can influence both the noise term ($\eta_t$) and the price ($\text{Price}_t$) itself. 
In \Cref{sec:ofuliv} we revisit \Cref{example1}, and we formulate it as a decision process by introducing a bandit setting where the agent has to decide on a price ($\text{Price}_{t,a}$) among a feasible set of prices in order to increase the sales ($\text{Sales}_t$). The agent also has access to a set of suppliers (contexts) such that each price corresponds to a compatible material cost ($\text{MaterialCost}_{t,a}$). 

\paragraph{Data generation with PSD.} To investigate these two scenarios, we generate semi-synthetic data according to the generating processes in
% way that is similar to the \textbf{LBE}. We give the details of this semi-synthetic generated data 
in \Cref{a:psdgen}. We take $d_{\bx} = d_{\bz} = 1$, and then we sample for each time $t$ the endogenous $\text{Event}_t\sim \text{Bernoulli}(0.1)$,  and the exogenous noises $\epsilon_{F,t}\sim\mathcal{N}_1(0,0.01)$, 
$\eta_{S,t}\sim\mathcal{N}_1(0,0.1)$. The Random variable $\text{MaterialCost}_t\sim\mathcal{U}_1$ is sampled for each time $t$ for the online regression setting and also for each arm $a$ in the pricing setting.
We fix the hidden constants of the problem to $\theta =1$ and $\beta = -1$.

\begin{table}[ht]
\footnotesize{
    \centering
    \begin{tabular}{l|l|l}
    & \textbf{Online regression with PSD} &  \textbf{Pricing with PSD} \\ \hline
    1st Stage 
    & 
    $\underbrace{\text{Price}_t}_{\text{cvt } {x}_t} = 
    \theta  \underbrace{\text{MaterialCost}_t}_{\text{IV } z_t} + \underbrace{\rho_{F}  \text{Event}_t +\epsilon_{F,t}}_{\text{exogenous noise }\epsilon_t}$ 
    & 
    $\underbrace{\text{Price}_{t,a}}_{\text{arms/cvts } {x}_{t,a}} 
= 
    \theta \underbrace{\text{MaterialCost}_{t,a}}_{\text{arms/IVs } z_{t,a}} + \underbrace{\rho_{F}  \text{Event}_t +\epsilon_{F,t}}_{\text{exogenous noise } \epsilon_t} $
    \\
    2nd Stage 
    & 
    $\underbrace{\text{Sales}_t}_{\text{outcome } y_t} 
    = 
    \beta  \underbrace{\text{Price}_t}_{\text{cvt } {x}_t} + \underbrace{\rho_{S}  \text{Event}_t + \eta_{S,t}}_{\text{endogenous noise }\eta_t}$
    & 
    $\underbrace{\text{Sales}_{t}}_{\text{outcome } y_t} 
= 
    \beta  \underbrace{\text{Price}_{t,A_t}}_{\text{chosen cvt } {x}_{t,A_t}} + \underbrace{\rho_{S}  \text{Event}_t + \eta_{S,t}}_{\text{endogenous noise }\eta_t}$
    \end{tabular}
    \caption{Experimental settings for PSD. We consider two settings. \textbf{Online regression with PSD} aims at learning the dynamic online, and performances are measured by the identification regret. \textbf{Pricing with PSD} is formulated as a linear bandit problem; we devise sampling strategies using a new oracle and confidence ellipsoid, and we study its regret. The term ``cvt'' stands for ``covariate'', and ``cvts'' stands for ``covariates''.}\label{a:psdgen}
    }
\end{table}

\paragraph{Results for online regression with PSD.}

To control the degree of endogeneity, we manipulate the impact of the unobserved variable $\text{Event}_t$ on both $\text{Price}_t$ and $\text{Sales}_t$ by adjusting the values of two parameters: $\rho_F$ and $\rho_S$. 
From the perspective of endogeneity, it is crucial to assess the accuracy of parameter estimation. We investigate this by presenting regret learning curves for varying levels of increasing endogeneity in \Cref{fa:lcps}. In \Cref{fa:heatmap}, we show the final regret across a wide range of parameter values at the time horizon.

\begin{figure}[ht]
\centering
\begin{subfigure}{0.45\textwidth}
  \centering
  \includegraphics[width=0.6\columnwidth]{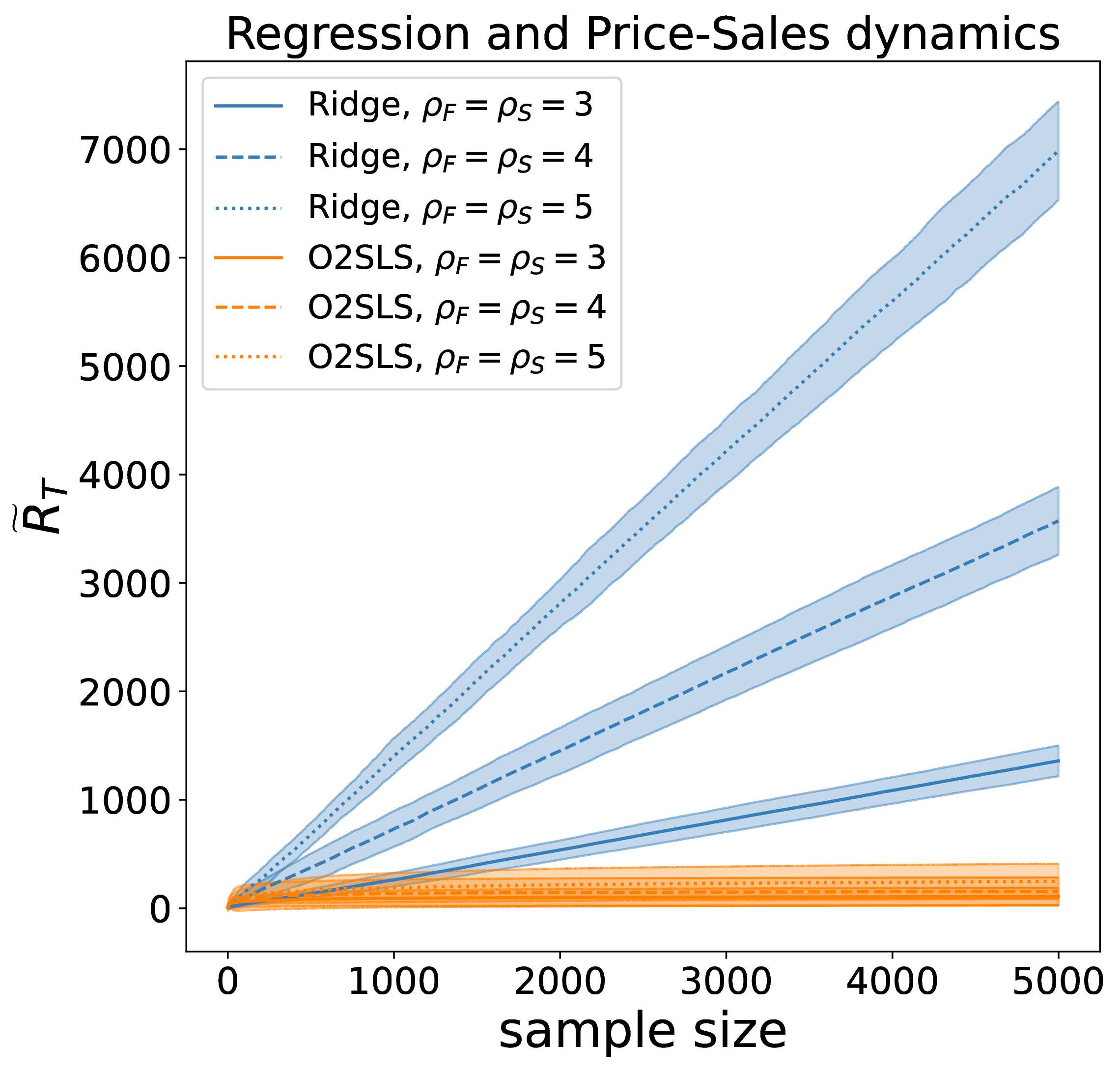}
  \caption{Learning curves for identification regret.}
  \label{fa:lcps}
\end{subfigure}%
\begin{subfigure}{0.55\textwidth}
  \centering
  \includegraphics[width=0.9\columnwidth]{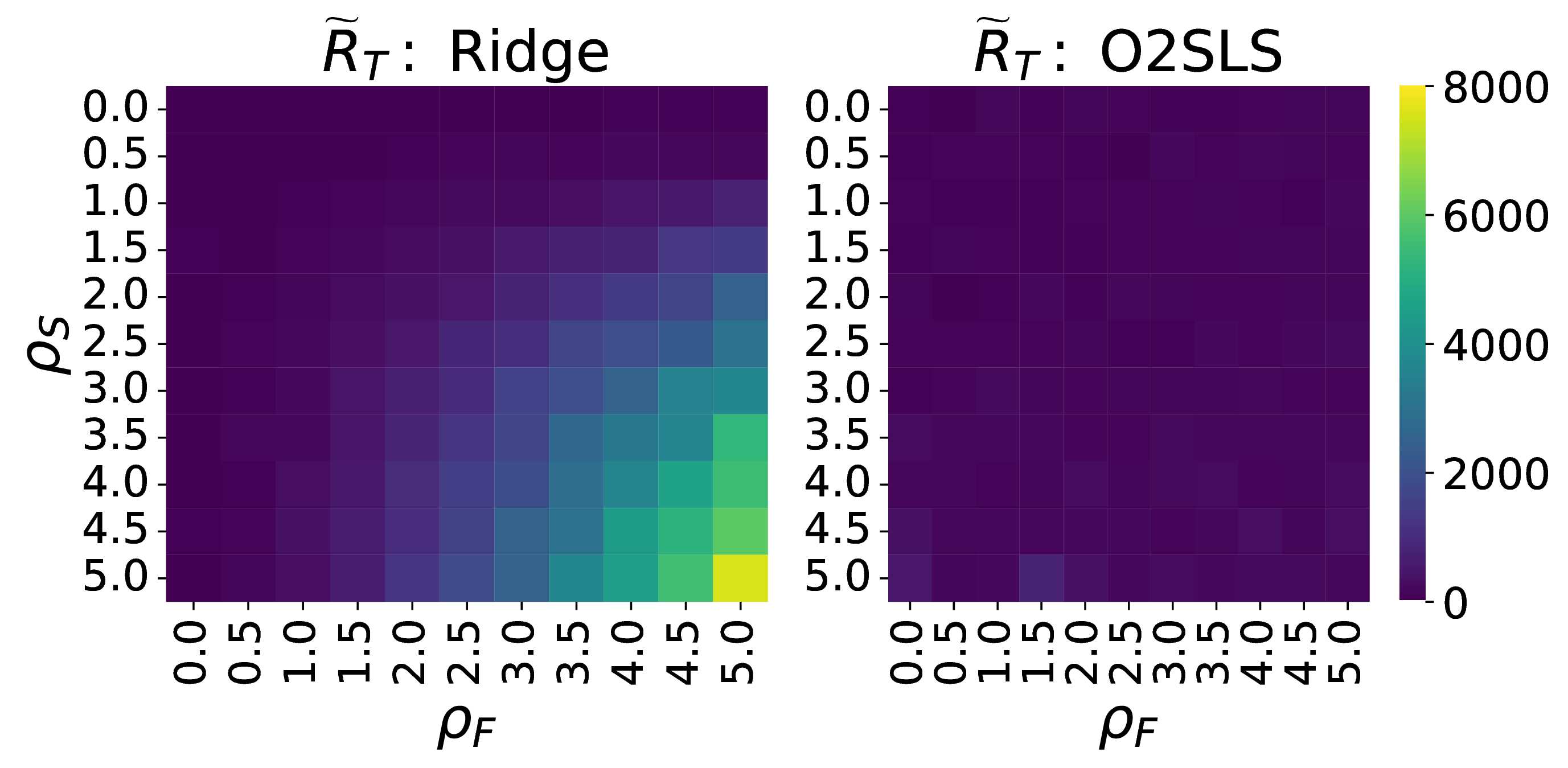}
  \caption{Final step $\widetilde R_T$, different endogeneity ranges.}
  \label{fa:heatmap}
\end{subfigure}
\caption{
% \Cref{fa:heatmap} Identification regret after $T = 5000$ steps of Online Ridge (left) and \otsls{}, for different combination of $\rho_F$ and $\rho_S$ in $[0,200]$.
In \Cref{fa:lcps} are shown the learning curves for the identification regret $\widetilde R_T$ of Online Ridge and \otsls{} over $T=5000$ steps, and for $\rho_S=\rho_F= 3,4,5$. 
\Cref{fa:heatmap} is the identification regret after $T = 5000$ steps of Online Ridge (left) and \otsls{} (right), for different combination of $\rho_F$ and $\rho_S$.
With the increasing endogeneity, \otsls{} attains lower regret than \ridge{} in all experiments.
}
\end{figure}

In \Cref{fa:lcps}, we examine the learning curves for the identification regret while varying the level of endogeneity. We visually represent the learning curves, capturing the mean values and their corresponding standard deviations. These curves are generated by averaging the results over 20 independent runs for different parameter combinations. Specifically, we consider increasing parameters $\rho_S=\rho_F = 3,4,5$ (higher $(\rho_S,\rho_F)$ indicates higher endogeneity) to show how performances deteriorate for increasing endogeneity levels. 
\otsls{} consistently outperforms Online Ridge (\ridge{}) across the entire parameter space considered.
The performance-gain increases with increasing values of the two parameters, i.e. with increasing levels of endogeneity.

In Figure \ref{fa:heatmap} we show the final regret across a wide range of parameter values at the time horizon $T=5000$. On the left, we present the identification regret after $T=5000$ steps of Online Ridge, while the right side is the regret for \otsls{}. We consider all combinations of $\rho_F$ and $\rho_S$ within the range of values $\rho_S,\rho_F \in \{0, 0.5,1,1.5,2, 2.5, 3,3.5,4,4.5,5 \}$. This range encompasses a significant portion of the parameter space we wish to explore and, thus, a wide range of endogeneity.
Figure~\ref{fa:heatmap} validates the claim that \otsls{} outperforms Online Ridge in achieving lower regret across a wide range of parameter values.

Upon examining the results, a clear pattern emerges. \otsls{} consistently identifies the true parameter correctly across all levels of endogeneity, maintaining accurate estimates. This demonstrates the method's ability to handle endogeneity and provide reliable parameter estimation.
Conversely, Online Ridge displays a notable decline in performance as the level of endogeneity increases. The estimated values deviate increasingly from the true parameter, leading to incorrect estimation. This emphasizes the vulnerability of Online Ridge to the detrimental effects of endogeneity, resulting in biased parameter estimates.
These findings highlight the superior performance of \otsls{} in handling endogeneity and preserving accurate estimation, in contrast to the deteriorating performance of Online Ridge with increasing endogeneity levels. The robustness of \otsls{} positions it as a tool for tackling endogeneity-related challenges in parameter estimation.

% \paragraph{Pricing with price-sales dynamics}\label{app:ppsd}

\paragraph{Results for pricing with PSD.}

Similarly to the regression setting, we want to investigate the performances of our approach in the case of pricing with PSD.
This can be phrased as a linear bandit problem with endogeneity.
As before, to control the degree of endogeneity, we manipulate the impact of the unobserved variable $\text{Event}_t$ on both the set of prices  $\text{Price}_{t,a}$ and $\text{Sales}_{t}$ by adjusting the values of two parameters: $\rho_F$ and $\rho_S$. 
In this online decision-making scenario, it is crucial to assess parameter estimation accuracy like for online regression with PSD.  This is done in \Cref{fa:bb} and \Cref{fa:bmse11}.
In \Cref{fa:lc11}, we show cumulative regret performances of the two possible approaches: \oful{}, which does not consider the endogeneity of the setting and \ofuliv{}, which does.

In \Cref{fa:bb} and \Cref{fa:bmse11}, we show that for increasing levels of endogeneity, both \oful{} and \ofuliv{} perform worse in estimating correctly the true parameter $\beta=-1$. In particular, when the endogeneity is more severe \oful{} estimates it to have a positive sign mistakenly. This is reflected in very poor performances of \oful{} against \ofuliv{} in terms of cumulative regret in \Cref{fa:lc11}. 

\begin{figure}[ht]
\begin{subfigure}[b]{.3\columnwidth}
    \includegraphics[width=\textwidth]{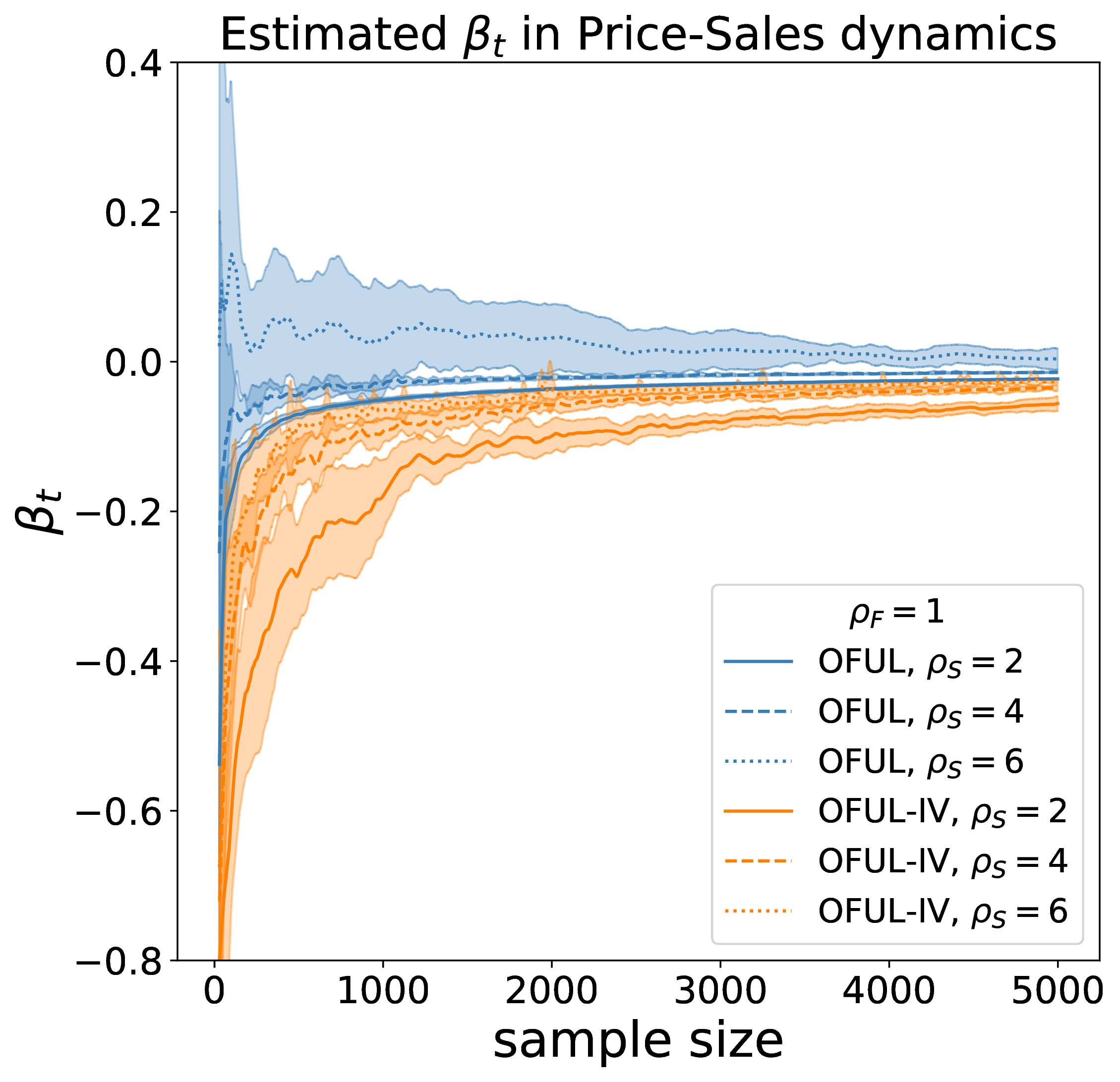}
    \caption{Estimated $\beta_t$.}
  \label{fa:bb}
\end{subfigure}%
\hfill
\begin{subfigure}[b]{.3\columnwidth}
    \includegraphics[width=\columnwidth]{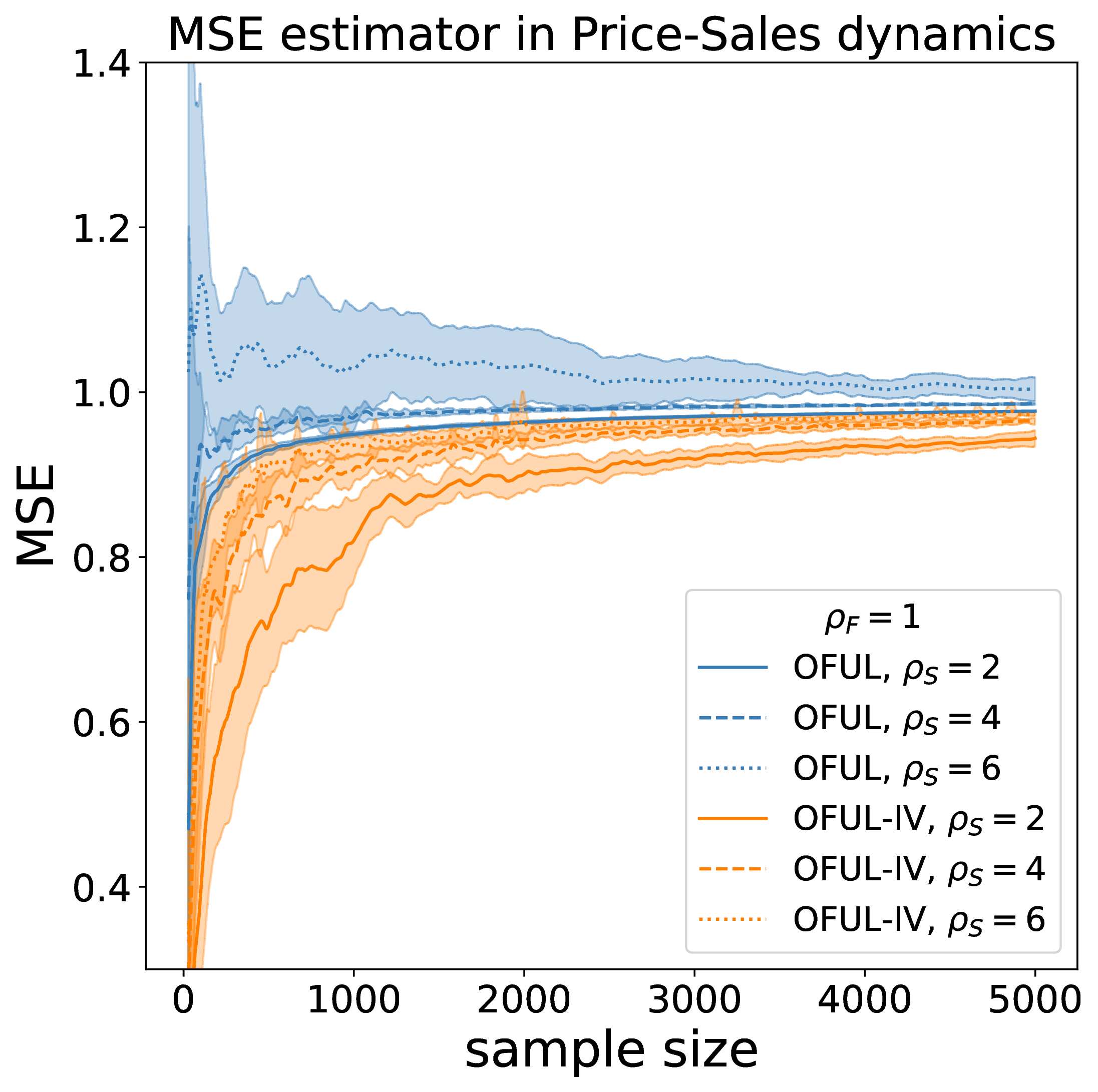}
    \caption{MSE of the estimate $\beta_t$.}
  \label{fa:bmse11}
\end{subfigure}%
\hfill
\begin{subfigure}[b]{.3\columnwidth}
    \includegraphics[width=\textwidth]{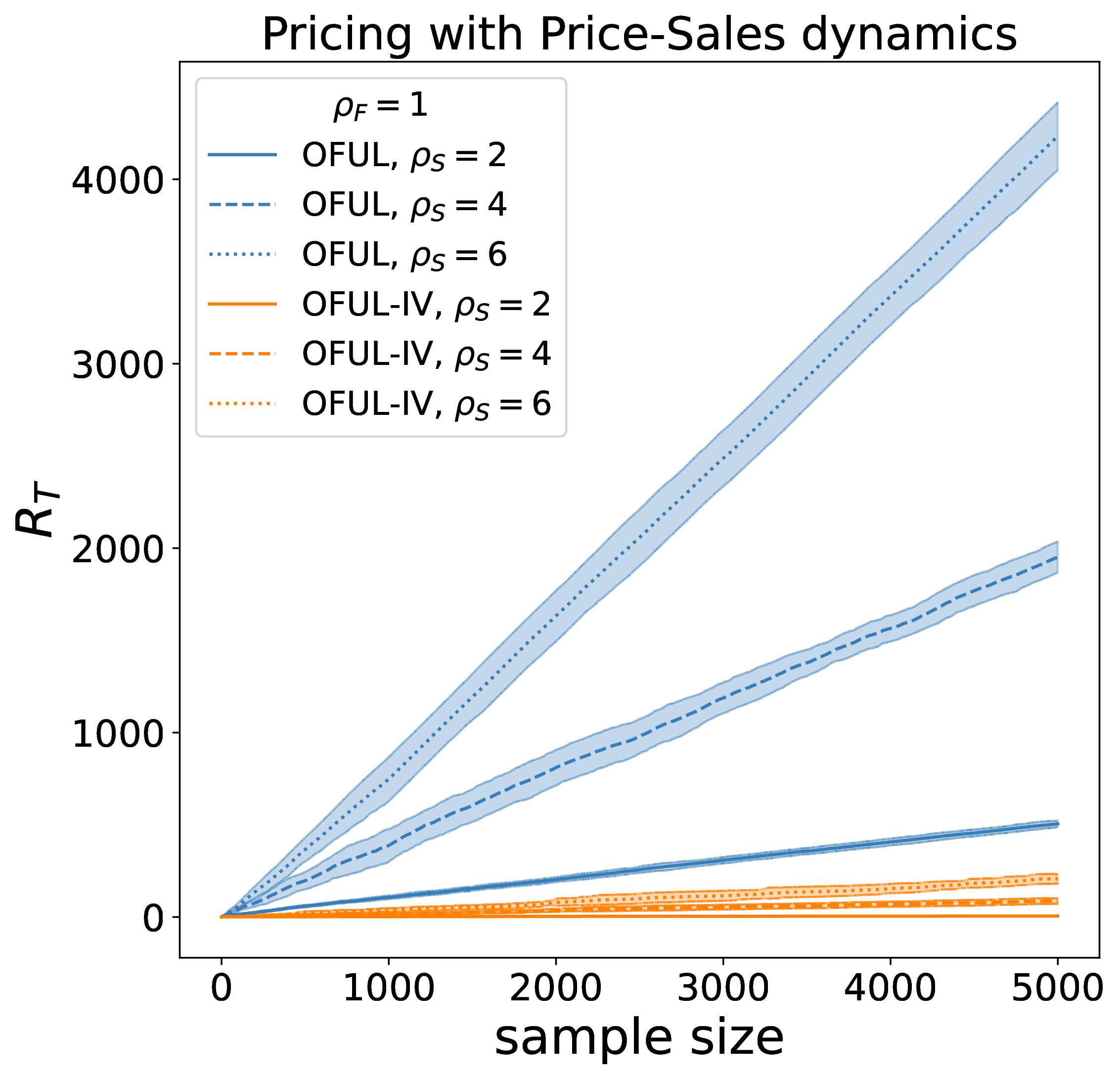}
    \caption{Regret learning curves. }
  \label{fa:lc11}
\end{subfigure}
\caption{Comparison of \oful{} and \ofuliv{} in a pricing with price-sales dynamics setting for a fixed value of $\rho_F=1$ and for increasing endogeneity levels controlled by $\rho_S=2,4,6$. 
Estimated value over a span of $T = 5000$ steps, where shaded regions indicate one standard deviation. 
% We compare for different endogeneity levels  
\Cref{fa:bb} is the estimated value of $\beta_t$ by Online Ridge and \otsls{}, while \Cref{fa:bmse11} is the mean squared error with respect to the true parameter $\beta$.
\Cref{fa:lcps} are learning curves for the regret.
}\label{f:b-psd-cr}
\end{figure}

\newpage
\subsection{An experiment on real data: Online estimation of gasoline consumption}

\paragraph{Motivation in real-life applications.} A real-life example, where applying linear bandits with endogeneity would be important and relevant is the Just-In-Time Adaptive Interventions (JITAI). 

\begin{example}[Just-In-Time Adaptive Interventions (JITAI)~\cite{tewari2017ads}]
JITAI framework is evoked to build mobile health platforms where a number of decisions have to be taken (e.g. ``whether or not to send an activity encouraging message") depending ``on a variety of behavior change domains including alcohol abuse, depression, obesity, and substance abuse", and the data regarding a patient's health condition (``such as GPS location, calendar busyness, and heartrate") is collected through mobile or IoT devices. \citep{tewari2017ads} states that ``Contextual bandit algorithms can be used for personalizing JITAIs" (page 2 of \cite{tewari2017ads}). 

In our contextual bandit formulation, a decision corresponds to an arm, the data regarding a patient form the covariates $X$, and a proximal outcome of a decision (e.g. whether the person gets into an activity or not) leads to the outcome $y$. In this setting, at step $t$, choosing an intervention/decision corresponds to choosing an instrument $Z$. Our goal is to choose a set of instruments at each step $t$ so that the health condition of the patient improves the most. This is an online problem by nature.

Performing a stochastic analysis of an online learning algorithm before deployment is important as it tells us expected gains and limitations, which correspond to the health of actual people. \cite{tewari2017ads} echo our motivation as ``Algorithms designed for the worst-case adversarial framework can perform sub-optimally when data is actually generated stochastically". They also state that one of the ``assumptions that make repeated appearance in the theoretical analyses of contextual bandit algorithms" is ``independence" (or exogeneity as we call it), and ``any candidate online learning algorithm needs to be tested for reasonable departures from these ideal assumptions in simulations before being deployed in a real study with users". 

We exactly aim to do that by studying how online regression behaves under endogeneity and what are the departures from well-known behaviours under exogeneity. \cite{tewari2017ads} further state that ``New algorithms that are more robust to failure of assumptions need to be designed and associated guarantees provided." This discussion makes it evident why studying endogeneity is important and is not a theoretical exercise.
\end{example}

\paragraph{Experiments on LBE with real-world data.} Since we do not have access to datasets like Just-In-Time interventions, which are hardly available in public due to their sensitive nature, we demonstrate the practical experiments on simulated datasets. This is not our construct but a known phenomenon in the existing literature on instrumental variable based bandits~\citep{kallus2018instrument,stirn2018thompson}, or tackling action-independent endogeneity~\citep{krishnamurthy2018semiparametric}. But for regression, we use the publicly available dataset from the ``Economic Report of the President 2000 and Census Bureau and Department of Energy" on motor gasoline consumption in the United States between 1970 and 1999.

\paragraph{Data description.} The dataset used in our analysis spans the period from 1970 to 1999, and pertains to the consumption of motor gasoline in the United States. It encompasses various factors, including the price index, disposable income, and price indices of used cars, new cars, and public transport.
The original source of the data is the \emph{"Economic Report of the President 2000 and Census Bureau and Department of Energy"}. However, for a more comprehensive understanding and analysis of this dataset, we also refer to the work of \cite{heij2004econometric}. In their study, they provide an in-depth examination and econometric analysis of the data, offering valuable insights into the relationships and dynamics among the different variables. We summarise the variables appearing in our model in \Cref{tableapp}. 

% \begin{table}[ht]
% \centering
% \begin{tabular}{ccc}
% \hline
% Observation  &     Symbol                  & Type\\ \hline\hline
% year of observation & $\bullet_t$                    & round\\ \hline
% log real gasoline consumption & GC            & outcome $y_t$ \\ \hline
% log real gasoline price index & PG            & endogenous covariate  $x_{t,1}$\\ \hline
% log real disposable income & RI               & exogenous covariate  $x_{t,2}$ / IV $z_{t,1}$\\ \hline
% log real price index of new cars & RPN         & IV  $z_{t,2}$         \\ \hline
% log real price index of public transport & RPT & IV $z_{t,3}$ \\ \hline
% log real price index of used cars & RPU     & IV $z_{t,4}$ \\ 
% \hline
% \end{tabular}
% \caption{Observations appearing in the two stages equations with the symbols used for them and the type of variable they represent in our regression setting.}
% \label{tableapp}
% \end{table}
\begin{table}[ht!]
\centering
\setlength{\abovecaptionskip}{10pt} % Adjust the value to increase/decrease the spacing
\begin{tabular}{c|c|c}
Observation  &     Symbol                  & Type\\ \hline
year of observation & $\bullet_t$                    & round\\ 
log real gasoline consumption & GC            & outcome $y_t$ \\ 
log real gasoline price index & PG            & endogenous covariate  $x_{t,1}$\\ 
log real disposable income & RI               & exogenous covariate  $x_{t,2}$ / IV $z_{t,1}$\\ 
log real price index of new cars & RPN         & IV  $z_{t,2}$         \\ 
log real price index of public transport & RPT & IV $z_{t,3}$ \\ 
log real price index of used cars & RPU     & IV $z_{t,4}$ \\ 
\end{tabular}
\caption{Observations appearing in the two-stage equations with the symbols used for them and the type of variable they represent in our regression setting.}
\label{tableapp}
\end{table}
% \todo{add here also results coming from the regression, estimated coefficients and $R^2$}

In our study, we adopt an online estimation approach to address the issue of endogeneity in the gasoline price model. Our methodology is motivated by \citep{heij2004econometric}, where instrumental variables are carefully selected based on their relevance and tests for endogeneity. The focus of our analysis is to estimate the price elasticity $\left(\beta_2\right)$ in an online fashion, taking into account the dynamic nature of the data.
It is important to note that the estimation of the price elasticity is carried out using the online Two-Stage Least Squares (\otsls{}) algorithm. This online approach allows us to continually update and refine our estimates as new data becomes available. We leverage instrumental variables, namely RI, RPT, RPN, and RPU, which have been chosen based on their ability to address endogeneity concerns and improve the identification of the price elasticity parameter.
Our approach enables us to capture the dynamic relationship between gasoline consumption and price, taking into account any potential endogeneity issues that may arise.

\begin{align*}
&
    \text{PG}_t 
= 
    c_1+ (\theta_1,\;\theta_2,\;\theta_3,\;\theta_4)\cdot \underbrace{(\text{RI}_t,\; \text{RPT}_t,\; \text{RPN}_t,\; \text{RPU}_t)}_{\text{IVs } \bz_t } +\epsilon_t\tag{First Stage}
\\
&
    \underbrace{\text{GC}_t}_{\text{Outcome } y_t} 
= 
    c_2+ \beta_2 \underbrace{\text{PG}_t}_{\text{Endogenous } x_{t,1}}+\beta_3 \underbrace{\text{RI}_t}_{\text{Exogenous }x_{t,2}}+\eta_t \tag{Second Stage}
\end{align*}

\paragraph{Evaluation metrics.} We use two evaluation metrics to evaluate the goodness of \otsls{} and Online Ridge regressors.
First, we plot the learning curve of \otsls{} and \ridge{} for parameters $\beta_1$ and $\beta_2$, while comparing them with their offline counterparts.
Second, we plot the evolution of $R^2$ of both the estimators, i.e. the proportion of the variation in the dependent variable that is predictable from the independent variables and is mathematically computed in the following way. The sum of squares of residuals, also called the residual sum of squares is
$
\text{SS}_{\mathrm{res}}=\sum_i\left(y_i-\langle \beta, x_i\rangle\right)^2$ and the total sum of squares (proportional to the variance of the data) is
$
\text{SS}_{\text {tot }}=\sum_i\left(y_i-\bar{y}\right)^2
$.
The formal definition of the coefficient of determination is
$
R^2=1-\frac{\text{SS}_{\mathrm{res}}}{\text{SS}_{\mathrm{tot}}}
$.  

\begin{figure*}[ht!]
\centering
    \includegraphics[scale=0.7]{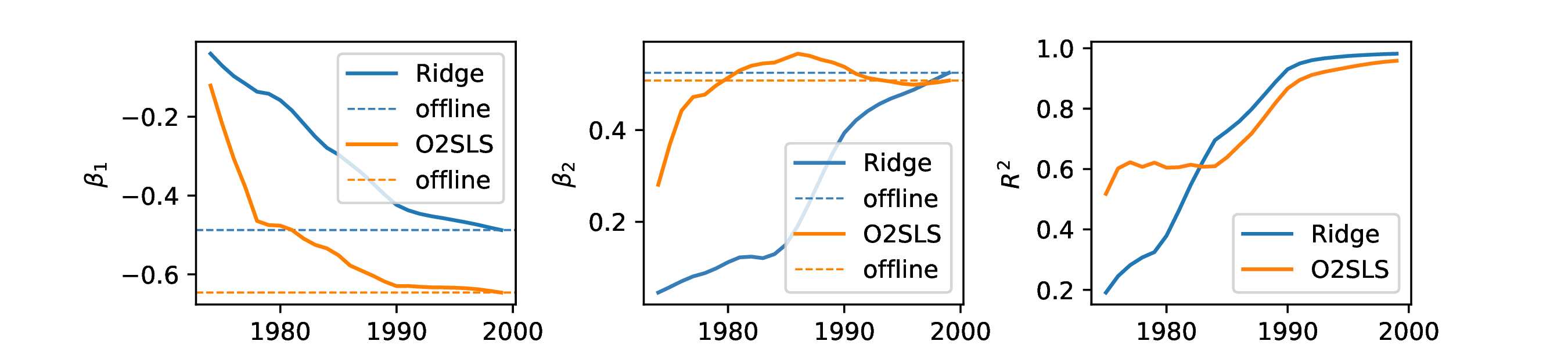}
\caption{Comparison between the estimation of the price elasticity $\beta_1$ (left) and the  coefficient $\beta_2$ of the exogenous variables $\text{RI}_{t}$ (centre) obtained using \otsls{} for online linear estimation and the offline estimate. The dashed line corresponds to the offline estimate obtained by the two Ridge and 2SLS. On the right, we plot the coefficient of determination $R^2$, for the two methods.}\label{gasoline}
\end{figure*}

\paragraph{Results.} Although we do not have a ground truth for the real dataset and cannot average our runs over multiple simulations, several important observations can still be made from the results shown in Figure \ref{gasoline}. Firstly, we observe that the coefficient $\beta_1$ corresponding to the endogenous variable exhibits significant differences between the estimates obtained from Online Ridge and \otsls{}. 
This disparity is expected since \otsls{} is specifically designed to address endogeneity, whereas Online Ridge does not possess this capability. Conversely, we notice that the difference between the estimates for the exogenous variable coefficient $\beta_2$ is less pronounced. 
In both cases, \otsls{} converges to its final value in fewer rounds (years in this dataset) and then remains relatively constant. This behaviour confirms that the dataset is affected by endogeneity, and \otsls{} effectively mitigates this issue.
Secondly, the coefficient of determination approaches values closer to one at earlier stages, indicating that the estimation of the $\beta$ coefficients is more reliable and robust for \otsls{} compared to \ridge{}. This suggests that \otsls{} provides more trustworthy estimates and demonstrates superior performance in handling endogeneity-related challenges.

\vfill

\end{document}